%% file: draft4.tex
\documentclass[twoside,12pt]{article}

\newif\ifjmlr

\ifjmlr
	\usepackage[abbrvbib]{jmlr2e}
	\input{header-jmlr}
\else
	\input{header}
    \usepackage{authblk}

\newcommand{\M}{\rmM}
\newcommand{\N}{\bbN}
\newcommand{\Reg}{\mathrm{Reg}}

\begin{document}
\ifjmlr
\author{\name Anna Shalova \email a.shalova@tue.nl\\ 
	\addr 	
	\AND Andr\'e Schlichting \email a.schlichting@uni-muenster.de\\
	\addr 
	\AND Mark Peletier \email m.a.peletier@tue.nl\\
	\addr }
	\editor{}
\else
\author{Anna Shalova$^{1}$\thanks{\href{mailto:a.shalova@tue.nl}{a.shalova@tue.nl}} \quad Andr\'e Schlichting$^{2}$ \quad Mark Peletier$^{1}$ \vspace{0.3em} \\
{\normalsize $^1$Department of Mathematics and Computer Science,\\
Eindhoven University of Technology} \\
{\normalsize $^2$Institute for Analysis and Numerics, University of Münster}
}

\fi
\title{ Singular-limit analysis of gradient descent with noise injection
}
\date{\empty}
\maketitle

\def\ourkeywords{Noise injection, stochastic optimization, stochastic gradient descent, zero-loss set, overparametrization, regularization, dropout.}

\begin{abstract}
We study the limiting dynamics of a large class of noisy gradient descent systems in the overparameterized regime. In this regime the \emph{zero-loss set} of global minimizers of the loss is large, and when initialized in a neighbourhood of this zero-loss set a noisy gradient descent algorithm slowly evolves along this set. In some cases this slow evolution has been related to better generalisation properties. We characterize this evolution for the broad class of noisy gradient descent systems in the limit of small step size. 

Our results show that the structure of the noise affects not just the form of the limiting process, but also the time scale at which the evolution takes place. We apply the theory to Dropout, label noise and classical SGD (minibatching) noise, and show that these evolve on different two time scales. Classical SGD even yields a trivial evolution on both time scales, implying that additional noise is required for regularization.

The results are inspired by the training of neural networks, but the theorems apply to noisy gradient descent of any loss that has a non-trivial zero-loss set. 

\ifjmlr
\else
\par\medskip\fi
\end{abstract}

\ifjmlr
\fi

\tableofcontents

\section{Introduction}
\label{s:intro}

\subsection{Noise injection}
\label{ss:noise-injection}

Modern machine learning and especially the training of neural networks rely on  gradient descent and its many variants.
Many of those variants introduce  \emph{noise} (randomness) into the algorithm, for instance through minibatching, Dropout, or random corruption of the labels. This noise may be motivated by practical considerations, as in the case of minibatching, but in many cases it is observed that the noise also improves the quality of the resulting parameter point. In particular, the noise often leads to parameter points that generalize better. It would be of great practical value to understand this \emph{implicit bias} of noisy algorithms, and this currently is an active area of research. 

In a seminal paper, Li, Wang, and Arora~\cite{LiWangArora22} focused on a specific class of noisy gradient-descent algorithms, and showed how the particular form of the noise leads to improved generalisation. They focused on overparameterized systems, in which the \emph{zero-loss set} $\Gamma:= \{w: L(w)=0\}$ is a high-dimensional manifold. 
Their key observation is that gradient descent  behaves differently with and without noise: in a neighbourhood of $\Gamma$,  deterministic gradient descent converges to $\Gamma$ and then stops, while noisy gradient descent may continue to evolve after reaching $\Gamma$. Figure~\ref{fig:ex1-intro} below shows an example of this. Li, Wang, and Arora characterized this continuing evolution in a small-step-size limit, and showed its relation to generalisation. 

The aim of this paper is to generalize the observations of~\cite{LiWangArora22} to a much wider class of systems, and characterize the behaviour of these systems when they evolve in the neighbourhood of the zero-loss set $\Gamma$. This generalisation was inspired by the case of Dropout, but the resulting setup covers many more types of noise. This generality also allows us to identify a hierarchy in scaling of different types of noise, such as minibatch noise, label noise, Dropout, and others. As it turns out, the case studied in~\cite{LiWangArora22} is not the first but the second level in this hierarchy, with for instance Dropout occupying the first and more dominant level.

While this work is inspired by the training of neural networks, the main characterizations do not use any neural-network structure, and therefore apply to noisy gradient descent in other application areas as well.

\bigskip
We now describe the main results of this paper, and we start by fixing some notation. \emph{Gradient descent} for a loss function $L: \bbR^m \to [0, \infty)$ with parameter dimension $m\in \bbN$ and learning rate $\alpha>0$ is the iterative algorithm
\[
w_{k+1} = w_k - \alpha \grad_w L(w_k), \qquad\text{for } k\geq0 \,.
\]

\medskip
In this paper we consider a class of random perturbations of gradient descent, that we call for short \emph{noisy gradient descent}:
\begin{subequations}
\label{eq:NGD-intro}
\begin{alignat}2
\label{eq:intro1}
w_{k+1} &= w_k - \alpha \grad_w \hat{L}(w_k, \eta_k),\qquad && \text{for }  k\geq0\, .
\end{alignat}
Here $\hat L:\R^m\times \R^d\to\R$ is an extension of $L$, each $\eta_k$ is a random vector in $\R^d$ satisfying
\begin{alignat}2
	\Expectation \eta_{k, i} &=0 \quad\text{and}\quad \Var \eta_{k,i} = \sigma^2, \qquad &&\text{for } i = 1,\dots , d \,,
	\label{eq:intro1b}
\end{alignat}
and $\eta_{k,i}$ is independent from $\eta_{\ell,j}$ for $(k,i)\not=(\ell,j)$.
The connection between $\hat L$ and $L$ is enforced by the following \emph{consistency requirement} at $\eta=0$:
\begin{equation}
\hat L(w,0) = L(w) \qquad \text{for all }w\in \R^m \, .
\label{eq:consistency}
\end{equation}
\end{subequations}
Apart from~\eqref{eq:consistency} we only require a certain regularity of the function $(w,\eta) \mapsto \hat L(w,\eta)$ (for the full setup, see Section~\ref{sec:setting}).

As mentioned above, many common forms of noise injection can be written in the form~\eqref{eq:NGD-intro}. In the training of neural networks, the  most common form of noise arises from evaluating the loss on random minibatches; we discuss this in Section~\ref{sec:examples:mini}. Dropout~\cite{HintonSrivastavaKrizhevskySutskeverSalakhutdinov12,srivastava2014dropout} is another example of~\eqref{eq:NGD-intro}, both in the more common `Bernoulli' form and in the `Gaussian' form. Other forms are `label noise'~\cite{BlancGuptaValiantValiant20,DamianMaLee2021,LiWangArora22} and `stochastic Langevin gradient descent'~\cite{RaginskyRakhlinTelgarsky17,mei2018mean1}. We discuss all these in Section~\ref{sec:examples}. 

\begin{figure}[ht]
	\newcommand\scale{0.7}
	\sbox0{\includegraphics[scale=\scale]
    {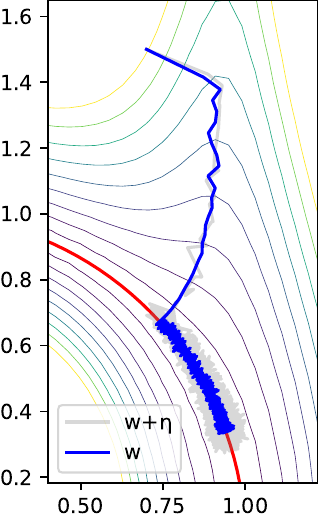}}
	\centering
	\subcaptionbox{Level curves of the function $L$}
		{\adjustbox{valign=T,set depth=\ht0}{\includegraphics[scale=\scale]
        {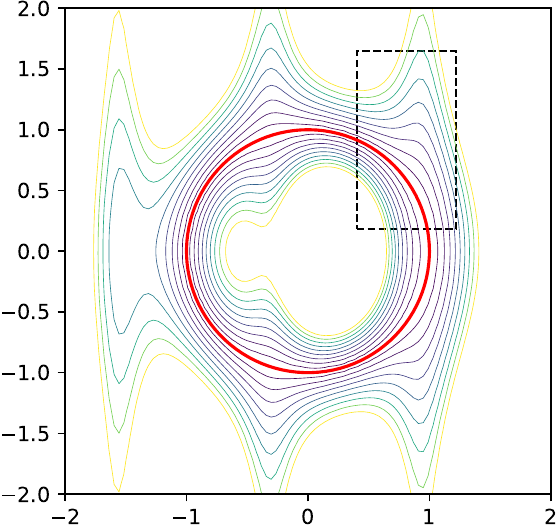}}}
	\qquad
	\subcaptionbox{Gradient descent}	
		{\adjustbox{valign=T,set depth=\ht0}{\includegraphics[scale=\scale]
        {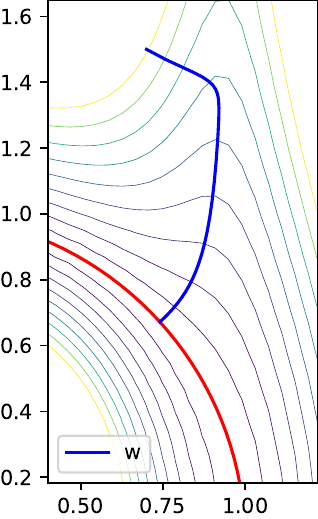}}}
	\qquad
	\subcaptionbox{Noisy gradient descent%
		\label{fig:ex1-intro:sub:NGD-nondeg}}{\box0}
	\caption{Noisy gradient descent may continue to move after reaching the zero-loss set $\Gamma$. The left-hand panel shows the level curves of a function $L:\R^2\to[0,\infty)$, with the zero-level set $\Gamma$ marked in red. The middle panel shows a gradient-descent evolution, starting at the top, and converging to $\Gamma$. The right-hand panel shows an evolution of the noisy gradient descent~\eqref{eq:NGD-intro} with $\hat L(w,\eta) := L(w+\eta)$. See Appendix~\ref{app:details-num-sim} for more details. }
	\label{fig:ex1-intro}
\end{figure}

\subsection{The non-degenerate case}

As mentioned above, the focus of this paper lies on the behaviour of the noisy gradient descent~\eqref{eq:NGD-intro} once it has entered a neighbourhood of the zero-loss set~$\Gamma$. As illustrated by Figure~\ref{fig:ex1-intro:sub:NGD-nondeg}, the algorithm typically continues to evolve after reaching $\Gamma$, and in this paper we aim to characterize this continued evolution.

We motivate the main results by some heuristic calculations. 
Let a sequence $(w_k,\eta_k)_{k=0}^\infty$ be given by~\eqref{eq:NGD-intro}, and assume for simplicity that dynamics of the parameters $w_k$ is slow compared to  the noise variables $\eta_k$. 
It then becomes reasonable to average over the fast variables $\eta_k$:
\begin{align*}
w_{k+1} &\approx w_k - \alpha \bbE_\eta \bigl[ \grad_w \hat{L}(w_k, \eta_k) \bigr].
\end{align*}
If the variance $\sigma^2$ of the noise is small, then by  Taylor expansion we find
\begin{align}
w_{k+1} -w_k
&\approx - \alpha \,\bbE_\eta \bigl[ \grad_w\hat{L}(w_k,0)\bigr] 
  -\alpha \,\bbE_\eta \bigl[\partial_\eta\grad_{w}\hat{L}(w_k,0)[\eta_k]\bigr] \notag \\
  &\qquad{}
  - \frac\alpha 2\,\bbE_\eta\bigl[ \partial^2_{\eta}\grad_{w}\hat{L}(w_k, 0)[\eta_k, \eta_k]\bigr] \,.
\label{eq:intro-approx-nondeg}
\end{align}
(See Section~\ref{sec:notation} for our use of $\partial$.)
Recall that $\hat{L}(w, 0) = L(w)$, so that the first term in~\eqref{eq:intro-approx-nondeg} is the gradient of the loss without noise.
From the property~\eqref{eq:intro1b} we have $\bbE[\eta_k] = 0$ and $\bbE_\eta\bigl[ \eta_k \oti \eta_k \bigr] = \sigma^2I$, leading to the resulting dynamics
\begin{align}
	\label{eq:intro-approx-nondeg2}
w_{k+1} - w_k \approx -\alpha \grad_w L(w_k) - \frac\alpha2 \sigma^2\grad_{w}\Delta_{\eta}\hat{L}(w_k, 0) + o(\sigma^2) 
\qquad \text{as }\sigma^2\to0 \,.
\end{align}

From the equation~\eqref{eq:intro-approx-nondeg2} we can recognize two phases in the evolution of $w_k$. In the first, faster phase, $w_k$ has not yet reached $\Gamma$, and the first term $-\alpha \nabla L$ on the right-hand side dominates. This leads to an evolution at time scale $1/\alpha$. 

When $w_k$ reaches $\Gamma$, however, the  term $-\alpha \nabla L$ becomes small, since $\nabla L=0$ on $\Gamma$, and the second term on the right-hand side of~\eqref{eq:intro-approx-nondeg2} becomes important. This term leads to a slower evolution, at time scale $1/\alpha\sigma^2$, which is driven by the combination of the two terms on the right-hand side in~\eqref{eq:intro-approx-nondeg2}. 

\bigskip

These observartions are made rigorous in the following non-rigorous formulation of our first main result, Theorem~\ref{th:main}. This theorem states a convergence result after accelerating the evolution by a factor $1/\alpha\sigma^2$, thus capturing the second, slower phase. 
\begin{formaltheorem}[Non-degenerate case]
	\label{fth:main1}
Assume $\alpha_n\to0$ and $\sigma_n\to0$. Set 
\[
W_n(\alpha_n \sigma_n^2 k) := w_k\,.
\]
Then the sequence $W_n$ converges to a limit $W = (W(t))_{t>0}$, that satisfies the constrained gradient flow
\begin{equation}
	\label{eq:fthA:GF}
\partial_t W(t) = - P_\Gamma \grad_w \Reg(W(t)), \qquad \text{with $W(t)\in \Gamma$ for $t>0$}.
\end{equation}
Here $P_\Gamma$ is the orthogonal projection onto the  tangent plane of $\Gamma$ at $W(t)$, and 
\[
\Reg(w) := 
\frac12\Delta_\eta \hat L(w,0)= \frac12\sum_{i=1}^d \frac{\partial^2}{\partial \eta_i^2} \hat L(w,\eta) \Big|_{\eta=0}
\qquad \text{for }w\in \Gamma.
\]
\end{formaltheorem}

\begin{figure}[h]
	{\captionsetup{justification=centering,subrefformat=parens}
	\newcommand\scale{0.65}
	\sbox0{\includegraphics[scale=\scale]{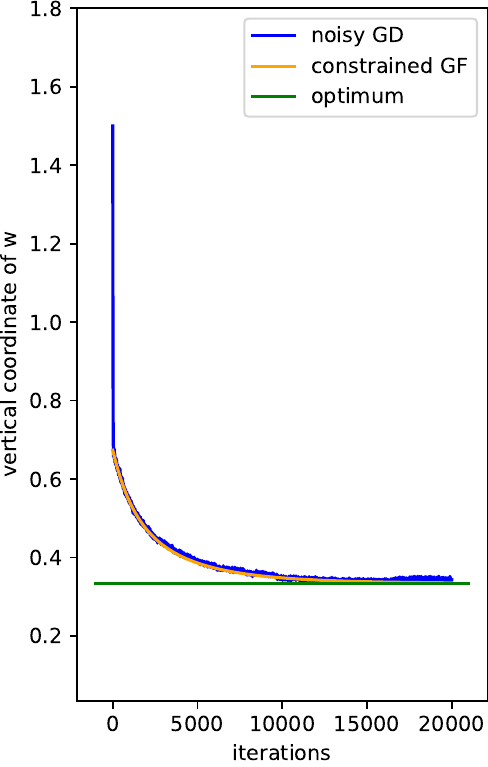}}
	\centering
	\subcaptionbox{Noisy gradient descent as in Fig.~\ref{fig:ex1-intro:sub:NGD-nondeg}%
		\label{fig:ex2-intro:in-contour}}
		{\adjustbox{valign=T, set depth=\ht0}
			{\includegraphics[scale=\scale]{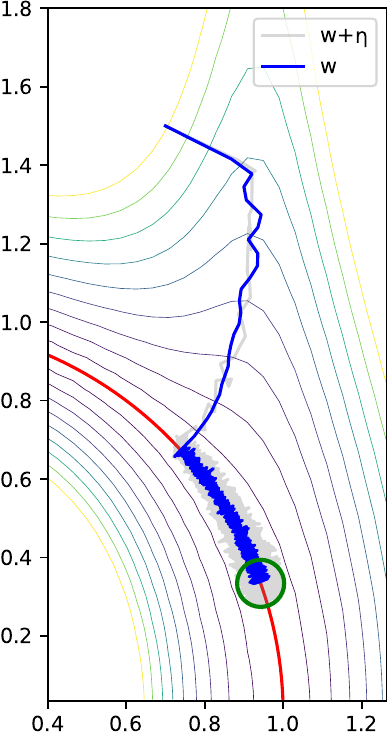}}}
	\qquad
	\subcaptionbox{Vertical coordinate of noisy and  constrained gradient descent%
		\label{fig:ex2-intro:short-time}}
		{\adjustbox{valign=T, set depth=\ht0}
			{\includegraphics[scale=\scale]{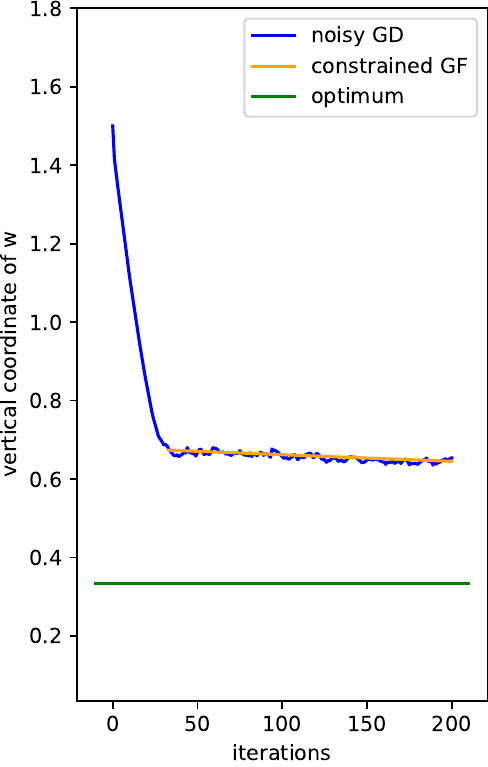}}}
	\qquad
	\subcaptionbox{Same as \subref{fig:ex2-intro:short-time}, but over longer time%
		\label{fig:ex2-intro:long-time}}
		{\adjustbox{valign=T, set depth=\ht0}
			{\includegraphics[scale=\scale]{Figures/ngd-cgd-axis-3.pdf}}}
	}
	\caption{This figure shows the evolution of Fig.~\ref{fig:ex1-intro:sub:NGD-nondeg} in more detail, and compares it to the solution of the constrained gradient flow. Panel~\subref{fig:ex2-intro:in-contour} is the same as Fig.~\ref{fig:ex1-intro:sub:NGD-nondeg}; panels~\subref{fig:ex2-intro:short-time} and~\subref{fig:ex2-intro:long-time} show the vertical coordinate over different periods of time (iterations). Panels~\subref{fig:ex2-intro:short-time} and~\subref{fig:ex2-intro:long-time} clearly illustrate the difference in time scales between the fast evolution towards $\Gamma$ and the much slower evolution along $\Gamma$.  The green circle and lines mark the minimizer of $\Delta_\eta \hat L(w,0)$, which in this case equals $ \Delta_w L(w)$. 
	The orange curves are the solution of the constrained gradient flow~\eqref{eq:fthA:GF}, started from the first point at which $w_k$ came close to $\Gamma$. More details are given in Appendix~\ref{app:details-num-sim}.}
	\label{fig:ex2-intro}
\end{figure}

\begin{remark}[Projection onto the tangent plane]
	The appearance of the projector $P_\Gamma$ in~\eqref{eq:fthA:GF} can be understood in two complementary ways. To start with, if $W$ remains in $\Gamma$, then the right-hand side in~\eqref{eq:fthA:GF} has to be a tangent vector; since $\grad_w \Delta_\eta \hat L(w,0)$ has no reason to be tangent at $w$, the projection is necessary. 

	The formal calculation~\eqref{eq:intro-approx-nondeg2} also shows how this projection is generated in the evolution: if the increment  $-(\alpha\sigma^2/2)\grad_w \Delta_\eta \hat L(w_k,0)$ pushes $w_{k+1}$ away from  $\Gamma$, then in the next iteration the first term $-\alpha \grad_w L(w_{k+1})$ will not be zero, but point `back to $\Gamma$'. Since the prefactors of the first and second term on the right-hand side of~\eqref{eq:intro-approx-nondeg2} differ by a factor $\sigma^2\ll1$, the first term is asymptotically much stronger than the second, and this generates in the limit the projection.
\end{remark}

\begin{example}[The example in Figures~\ref{fig:ex1-intro} and~\ref{fig:ex2-intro}]
	In the examples of Figures~\ref{fig:ex1-intro}-\ref{fig:ex2-intro} we have chosen $\hat L(w,\eta) := L(w+\eta)$, and therefore the function $\Reg$ in Theorem~\ref{fth:main1} is
	\[
	\Reg(w) := \frac12 \Delta_\eta L(w+\eta )\Big|_{\eta=0} = \Delta_w L(w).	
	\]
	The gradient flow~\eqref{eq:fthA:GF} of $\Reg$ along $\Gamma$ therefore evolves towards points on $\Gamma$ where $\Delta_w L$ is smaller, i.e.\ where the `valley' around $\Gamma$ is wider. In Figure~\ref{fig:ex2-intro} one can recognize the widening of the valley in the spreading of the level curves of $L$, and indeed the evolution converges to the minimizer of $\Delta_w L$ on $\Gamma$, indicated by the green circle and line.
\end{example}

\begin{example}[Bernoulli DropConnect]
\emph{Dropout} is a particular form of noise injection that consists of `dropping out' parameters or neurons randomly at each iteration. In Section~\ref{sec:examples} we discuss various types of Dropout; here we give one example, the case of \emph{DropConnect} with Bernoulli random variables. 

Given a loss function $L$ on $\R^d$, we define $\hat L: \R^d\times\R^d\to\R$ by
\[
	\hat L(w,\eta) := L(w{\odot} (1+\eta))\, ,
\]
where $\odot$ is coordinate-wise multiplication. We choose a dropout probability $p\in [0,1]$ and  we let each filter variable $\eta_i$ have the distribution
\begin{equation*}
\eta_{i} = \begin{cases}
        -1 & \text{ w.p. } p\,,\\
        \dfrac{p}{1-p} & \text{ w.p. }  1-p\,.\\
\end{cases}
\end{equation*}
With this choice, $\eta_i=-1$ corresponds to `dropping' or `killing' the parameter $w_i$, and the other value $\eta_i = p/(1-p)$ corresponds to  a rescaling such that in expectation we have $\eta=0$. This choice satisfies~\eqref{eq:intro1b} with $\sigma^2  =p/(1-p)$, and the limit $\sigma\to0$ is the one in which a vanishingly small number of  parameters are dropped.

For this case Theorem~\ref{fth:main1} applies (see Proposition~\ref{cor:BernoulliDropConnect}) and gives the formula for the corresponding function $\Reg$ as 
\begin{equation*}
	\Reg(w) := w\cdot  \nabla L(w) +\sum_{j=1}^m \bra[\big]{L(w\odot \invunit_j) - L(w)}.
\end{equation*}
Here $\invunit_j$ is the inverted unit vector in $\R^m$,
\[
\invunit_j := (1,\dots, 1, \underset {j^{\mathrm{th}}\text{ position}}{0,} 1, \dots, 1 )	 = 1 - \mathrm e_j.
\]
This function can be interpreted as a non-local approximation of the second derivative of $L$ (see Remark~\ref{rem:Structure-BDO-Reg}). The non-locality is a consequence of the large modification of individual coordinates $w_i$ by multiplication by $1+\eta_i$, which can be zero (albeit with small probability~$p$). We discuss DropConnect in more detail in Section~\ref{ss:DropConnect}.
\end{example}
\begin{remark}
    Note that \emph{DropConnect} noise is different from the classical \emph{Dropout training} of neural networks. In contrast to zeroing single parameters Dropout zeroes output of the neurons.  Theorem \ref{fth:main1} is also applicable Dropout noise, we study it for the cases of neural networks and overparametrized linear models in Section \ref{ss:examples:ClassicalDropout}.
\end{remark}

\subsection{The degenerate case}

In at least three important examples of noise injection, namely minibatching, label noise, and stochastic gradient Langevin descent, the procedure above applies, but the limiting constrained gradient flow is trivial: $\partial_t W = 0$. This is because for those examples it happens that  $ \Delta_\eta \hat L(w,\eta)$ is constant in $w$ and $\eta$, and for this reason we call these types of noise \emph{degenerate}.

To determine the behaviour for these degenerate forms of noise we return to the formal calculation~\eqref{eq:intro-approx-nondeg}.  Instead of taking the expectation we write, assuming $\grad_\eta^2 \hat L\equiv0$ for simplicity,
\begin{align}
	w_{k+1} -w_k &\approx -\alpha \nabla_w\hat{L}(w_k,0) 
		-\alpha \partial_\eta\grad_w\hat{L}(w_k,0)[\eta_k] 
		- \frac\alpha2 \partial^2_{\eta}\grad_{w}\hat{L}(w_k, \eta^*)[\eta_k, \eta_k] \notag \\
	&= - \alpha \grad_w L(w_k) - \alpha \partial_{\eta}\grad_w\hat{L}(w_k,0)[\eta_k]\,.
	\label{eq:intro:approx:deg1}
\end{align}
Take for instance the case that the $\eta_{k,i}$ are i.i.d.\ centered normal random variables with variance~$\sigma^2$. Then the second term on the right-hand side is again a centered normal random variable, with variance that scales as $\alpha^2\sigma^2$. 
As a result, we expect that the noise has a non-trivial contribution when the accumulated quadratic variation is of order one, namely after $k \sim 1/\alpha^2\sigma^2$ steps. 

\begin{figure}[htp]
	\centering
	{\captionsetup{justification=centering,subrefformat=parens}
	\newcommand\scale{0.7}
	\sbox0{\includegraphics[scale=\scale]{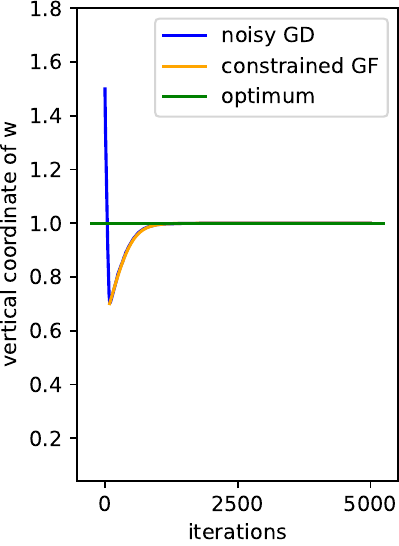}}
	\begin{varwidth}{\textwidth}
	\subcaptionbox{Non-degenerate noisy GD, $\hat L(w,\eta) = L(w) + \frac12 a(w)\eta^2$%
		\label{fig:nondeg-deg:nondeg:in-contour}}
        {\adjustbox{valign=T, set depth=\ht0}
			{\includegraphics[scale=\scale]{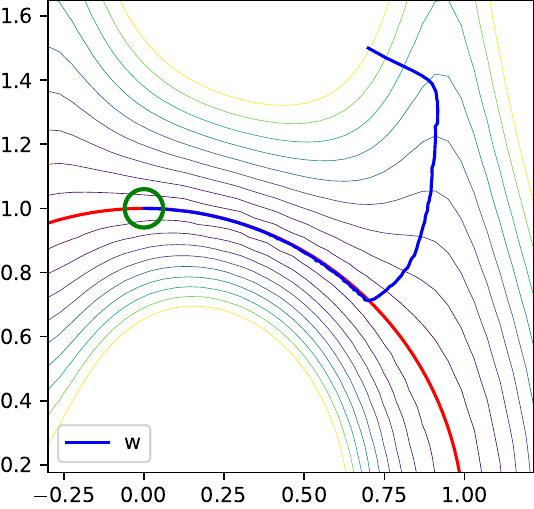}}}
	\qquad
	\subcaptionbox{Vertical coordinate of noisy and  constrained GD%
		\label{fig:nondeg-deg:nondeg:short-time}}
        {\adjustbox{valign=T, set depth=\ht0}
			{\includegraphics[scale=\scale]{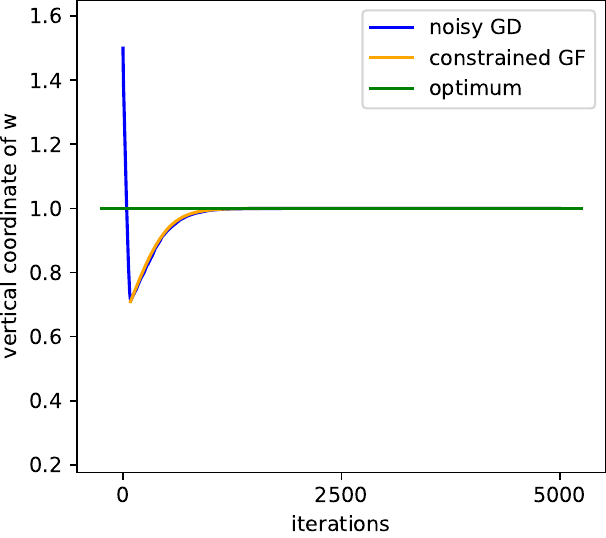}}}
	\\[6\jot]
	\subcaptionbox{Degenerate noisy GD, $\hat L(w,\eta) = L(w) + \frac12 |w|^2 \eta$%
	\label{fig:nondeg-deg:deg:in-contour}}
        {\adjustbox{valign=T, set depth=\ht0}
		{\includegraphics[scale=\scale]{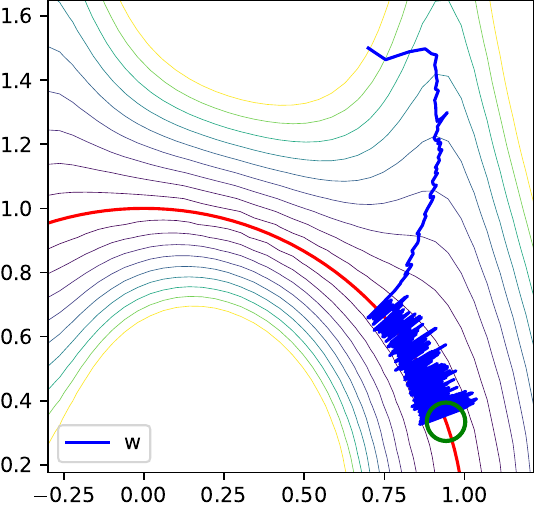}}}
	\qquad
	\subcaptionbox{Vertical coordinate of noisy and  constrained GD%
		\label{fig:nondeg-deg:deg:short-time}}
		{\adjustbox{valign=T, set depth=\ht0}
   {\includegraphics[scale=\scale]{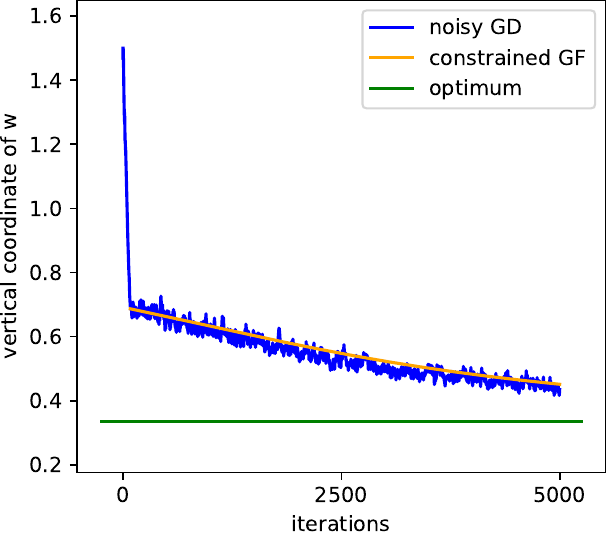}}}
	\end{varwidth}
	}
	\caption{Non-degenerate (top row) and degenerate evolution (bottom row); note how the evolution is much faster in the non-degenerate top row than in the degenerate bottom row. In both cases $\alpha = 0.1$ and $\eta$ is a scalar centered normal random variable with variance $\sigma^2=0.01$. In the top row,  $\hat L(w,\eta) = L(w) + \frac12 a(w)\eta^2$ with $a(w) = 1-0.7 \cos 2w_1$. For this case Theorem~\ref{fth:main1} applies and gives the regularizer as $\frac12 \Delta_\eta \hat L(w,0) = 2a(w)$. In the bottom row $\hat L(w,\eta) = L(w) + \frac12 |w|^2 \eta$, and Theorem~\ref{fth:main2} applies, with the limiting evolution~\eqref{eq:fthB:SDE} reduing to the deterministic constrained gradient flow  $\partial_t W = -P_\Gamma \nabla_w\frac14 \log \Delta_w L(W)$.  The green circles and lines mark the minimizers of the respective regularizers.
		 More details are given in Appendix~\ref{app:details-num-sim}.}
	\label{fig:nondeg-deg}
\end{figure}
\bigskip
These remarks motivate the following formal version of the second main result, Theorem~\ref{th:main:deg2}. For this theorem we assume that the function $\hat L$ has the more specific form
\begin{equation}
	\label{eq:Ldeg-intro}
	\hat{L}(w, \eta) = L(w) + f(w)\cdot\eta +  \tfrac12  H(w) : (\eta \otimes \eta)  + g(\eta),
\end{equation}
for certain smooth maps $f$, $H$, and $g$.  We assume that $g(0)=0$ and that each diagonal element $H_{ii}$ vanishes, so that $\Delta_\eta \hat L(w,0) = \sum_{i=1}^d H_{ii}(w) = 0$. It follows that the non-degenerate evolution~\eqref{eq:fthA:GF} is trivial, and by the argument above we expect the evolution to take place at the slower time scale $1/\alpha^2\sigma^2$
(see Definition~\ref{def:deg:noiseLoss} for details).
\begin{formaltheorem}[Degenerate case]
	\label{fth:main2}
	Let~$\hat L$ be a degenerate loss function as described above.
	For $\alpha_n\to0$ and $\sigma_n\to\sigma_0\geq 0$, let $(w_k^n)_{k\geq }$ be the noisy gradient descent~\eqref{eq:NGD-intro} and set
	
	\[
	W_n(\alpha_n^2 \sigma_n^2 k) := w_k\,.
	\]
	Then the sequence $W_n$ converges to a limit $W = (W(t))_{t>0}$ that satisfies the constrained stochastic differential equation
	\begin{equation}
		\label{eq:fthB:SDE}
	\dx W(t) = P_\Gamma \beta(W(t))\dx B_t
	 + F(W(t))\,{:}\,\beta(W(t))\beta(W(t))^\top \dx t, \qquad \text{with $W(t)\in \Gamma$ for $t>0$}.
	\end{equation}
	Here $\beta$ and $F$ are given in terms of $f$ and $H$, and $B$ is a multidimensional standard Brownian motion.
\end{formaltheorem}

\begin{remark}[Time scale]
	The time scale is now $1/\alpha^2\sigma^2$, a factor $1/\alpha$ longer than in the non-degenerate case. Figure~\ref{fig:nondeg-deg} illustrates this difference in speed with two functions $\hat L$ that are constructed from the same $L$ as in Figures~\ref{fig:ex1-intro} and~\ref{fig:ex2-intro}.
\end{remark}

\begin{remark}[Limiting equation]
	In contrast to the non-degenerate case, the limiting equation~\eqref{eq:fthB:SDE} is a constrained stochastic differential equation (SDE), not a constrained ODE. This difference arises from the fact that the second term on the right-hand side in~\eqref{eq:intro:approx:deg1} is now a mean-zero random variable of variance $O(\alpha^2\sigma^2)$, and the time scaling is such that we observe a sum of $O(1/\alpha^2\sigma^2)$ of these, leading to a random variable of variance $O(1)$.
\end{remark}

\begin{example}[Minibatch noise]
In Section~\ref{sec:examples:mini} we apply Theorem~\ref{fth:main2} to the case of minibatch noise. The resulting evolution is trivial \emph{even on the time scale $1/\alpha_n^2\sigma_n^2$}, meaning that the evolution on $\Gamma$ induced by minibatch noise is even slower than this. This also is consistent with the observation by Wojtowytsch~\cite{Wojtowytsch23,Wojtowytsch24} that for small step sizes minibatch SGD may collapse onto $\Gamma$ and become completely stationary.
\end{example}

\begin{example}[Label noise]
The result of~\cite{LiWangArora22} also is a case of degenerate noise, and we revisit it in Section~\ref{sec:examples:label}. For this case the limiting constrained SDE~\eqref{eq:fthB:SDE} is in fact a deterministic equation (i.e.\ the first term vanishes), in the form of a constrained gradient flow driven by the function $\Reg(w) := (2N)^{-1}\Delta_w L(w)$.
\end{example}

See Section~\ref{sec:examples} for more examples.

\subsection{Contributions}

The main contributions of the paper are:
\begin{enumerate}
	\item We introduce a specific class of `gradient-descent algorithms with noise injection' that unifies a number of existing noise-injection schemes, such as minibatch SGD, Bernoulli and Gaussian Dropout, Bernoulli and Gaussian DropConnect, label noise, stochastic gradient Langevin descent, `anti-correlated perturbed gradient descent', and others.
	\item For this class of noise-injected gradient-descent schemes we give an explicit and rigorous characterization of the regularization induced by the noise in the non-degenerate (Theorem~\ref{th:main}) and degenerate cases (Theorem~\ref{th:main:deg2}), and we identify the corresponding time scales.
	\item In particular, we prove convergence of gradient descent with Dropout and DropConnect noise to manifold-constrained gradient flows (Sections~\ref{ss:DropConnect} and~\ref{ss:examples:ClassicalDropout}). 
	\item We demonstrate that a version of minibatch SGD shows no regularization effect on the two fastest time scales (Section \ref{sec:examples:mini}) and does not change the form of the regularizer induced by Dropout or label noise (Section~\ref{sec:examples:label} and Remark~\ref{rem:combining-Dropout-with-Minibatching}).
\end{enumerate}

\paragraph*{Acknowledgements}
Mark Peletier and André Schlichting gratefully acknowledge several interesting discussions with Anton Arnold, Vlado Menkowski, Gijs Peletier, and Frank Redig. Mark Peletier and Anna Shalova are supported by the Dutch Research Council (NWO), in the framework of the program ‘Unraveling Neural Networks with Structure-Preserving Computing’ (file number OCENW.GROOT.2019.044).
André Schlichting is supported by the Deutsche Forschungsgemeinschaft (DFG, German Research Foundation) under Germany's Excellence Strategy EXC 2044 --390685587, Mathematics M\"unster: Dynamics--Geometry--Structure.

\section{Related Work}

\paragraph{Convergence results to stationary points.}
	In the asymptotic analysis of stochastic gradient schemes, such as for instance~\cite{Tadic2015,FehrmanGessJentzen2020}, a large body of literature considers also the question of convergence of the iterates $w_k$ to a single point as $k\to\infty$.  For instance Tadic~\cite[Section 3]{Tadic2015} proves such a statement for \emph{stochastic gradient algorithms with Markovian dynamics}, which takes a very similar form as the noisy gradient descent~\eqref{eq:intro1}. However, the main difference is that in~\cite{Tadic2015} the learning rate $\alpha$ is chosen to be $k$-dependent and to tend to zero as $k$ tends to infinity, whereas we consider a fixed learning rate. 
    
    More recent works \cite{Wojtowytsch23,DereichKassing2024} prove convergence of similar systems to a single point under a {\L}ojasiewicz condition on the loss manifold in comparison to the local quadratic behavior, which we assume below (see Assumption~\ref{ass:LossMfd}).

    These convergence results are consistent with the characterizations that we give in Theorems~\ref{fth:main1} and~\ref{fth:main2}, since vanishing learning rates may generate different types of behaviour from the fixed learning rates that we consider.

\paragraph{Noise injection: minibatch SGD.}
The most common source of noise in training algorithms results from the use of random minibatches in gradient descent, often simply called `stochastic gradient descent' (SGD). Experimentally this noise is known to lead to better generalisation, and various studies have focused on determining the dependence of this effect on parameters such as the batch size~\cite{KeskarMudigereNocedalSmelyanskiyTang16, JastrzebskiKentonArpitBallasFischerBengioStorkey17, WuMaE18,Wojtowytsch23} and learning rate~\cite{JastrzebskiKentonArpitBallasFischerBengioStorkey17, HofferHubaraSoudry17,WuMaE18,Wojtowytsch23,Andriushchenko2022}. The particular structure of minibatch noise has been shown to create a `collapse' effect~\cite{Wojtowytsch23,Wojtowytsch24} when the learning rate is small, and the same noise structure makes minibatch SGD incapable of selecting narrow minima~\cite{WuWangSu22}. In this paper we also apply the techniques to a modification of minibatch SGD (Section~\ref{sec:examples:mini}). 

\paragraph{Noise injection: Dropout.}
The most common form of Dropout is  `Bernoulli Dropout', i.e.\ randomly `dropping' neurons, input nodes, or weights with probability $p$~\cite{HintonSrivastavaKrizhevskySutskeverSalakhutdinov12,wager2013dropout,WanZeilerZhangLe-CunFergus13}, but the Gaussian Dropout in Section~\ref{ss:examples:ClassicalDropout} 
has also been reported to give good results~\cite{srivastava2014dropout,MolchanovAshukhaVetrov17}. Random networks generated by Dropout noise have the same universal-approximation properties as deterministic ones~\cite{ManitaPeletierPortegiesSandersSenen-Cerda22}, and the convergence of Dropout-Gradient-Descent algorithms has been rigorously characterized~\cite{Senen-CerdaSanders20,Senen-CerdaSanders22}.

A common explanation of the regularizing effect of Dropout centers on an interpretation of the Dropout-SGD iterates as a Monte-Carlo sampling of a deterministic augmented loss, where the loss term resembles $L^2$-type weight penalization. Various forms of this additional loss term have been derived, sometimes while assuming the Dropout noise to be `small'~\cite{BaldiSadowski13,wager2013dropout,wang2013fast,GoodfellowBengioCourville16,MianjyAroraVidal18,MianjyArora19,ZhangXu22}. Dropout-noise fluctuations have also been shown to have a significant effect in addition to this regularization-in-expectation~\cite{WeiKakadeMa20,ZhangXu22,MianjyArora20,ClaraLangerSchmidt-Hieber23TR}. Finally, very recently the effectiveness of Dropout regularisation has been connected to `weight expansion'~\cite{jin2022weight}. 

\paragraph{Other types of noise injection. }
`Label noise', the situation in which `mistakes' are present in the data set, is a challenge for the training of classification methods~\cite{FrenayVerleysen13}. One method to deal with this is to artifically perturb the labels during training; this can also be seen as a form of noise injection, and it recently has also been applied to the context of regression~\cite{BlancGuptaValiantValiant20,LiWangArora22}. {We comment on label noise in Section~\ref{sec:examples:label}.}

The `anti-correlated' noise injection of~\cite{OrvietoKerstingProskeBachLucchi22} is of the same form as we study in this paper, and we comment on this connection {in Section~\ref{sss:anti-correlated-PGD}}.

\paragraph{Evolution along the zero-loss manifold.} 
For overparameterized networks the training takes place in close proximity of the zero-loss set. The framework by Katzenberger~\cite{Katzenberger91} that we use provides a powerful tool to characterize such behaviour, and has been used extensively in the probability and other literature (see e.g.~\cite{FunakiNagai93,Funaki95,CalzolariMarchetti97,Parsons12TH,ParsonsRogers15TR} and also~\cite{FatkullinKovacicVanden-Eijnden10}). Li, Wang, and Arora used the same framework to characterize the behaviour of gradient descent with label noise in the limit of small step size~\cite{LiWangArora22}. This same point of view has also been used to analyse `local SGD'~\cite{GuLyuHuangArora23} and the impact of normalization~\cite{LiWangYu22}.

In addition to the asymptotically continuous-time approaches above, the discrete-time nature of gradient descent and stochastic gradient descent generates an additional implicit bias, which is related to the form of the loss landscape close to zero loss~\cite{AroraLiPanigrahi22,WuWangSu22}.

\paragraph{Training and flatness of minima.}
There is growing experimental and theoretical evidence that `flatter minima generalize better' (see e.g.~\cite{HochreiterSchmidhuber97,KeskarMudigereNocedalSmelyanskiyTang16,JastrzebskiKentonArpitBallasFischerBengioStorkey18,ChaudhariChoromanskaSoattoLeCunBaldassiBorgsChayesSagunZecchina19,JiangNeyshaburMobahiKrishnanBengio19}), and various properties of SGD and other training algorithms have been interpreted in this light (see e.g.~\cite{ZhangLiaoRakhlinMirandaGolowichPoggio18,JastrzebskiKentonArpitBallasFischerBengioStorkey18,WuMaE18,SmithElsenDe20,PesmePillaud-VivienFlammarion21,OrvietoKerstingProskeBachLucchi22,WuWangSu22}). The results of this paper relate to this  `flatness' hypothesis, since in many cases the regularization term in e.g.~\eqref{eqdef:th:main:limiting-dynamics} can be recognized as some  measure of `flatness'. The examples in the introduction illustrate this: for instance, the choice $\hat L(w,\eta) := L(w+\eta)$ in Figure~\ref{fig:ex2-intro} leads to the constrained gradient flow driven by $\Delta_wL(w)$, which is indeed a measure of the `flatness' of the loss landscape around~$\Gamma$, and the evolution moves towards the minimizer of this `flatness', as indicated by the green circle. In the case of label noise, it was already shown in~\cite{LiWangArora22} that the resulting regularizer also is proportional to $\Delta_w L(w)$, with the same effect. As yet another example, the regularizer that we derive for DropConnect in Section~\ref{ss:DropConnect} is closely related to the concept of `robustness' developed in~\cite{PetzkaKampAdilovaSminchisescuBoley21}, which those authors connect in turn to generalization performance (see Remark~\ref{rem:generalisation}).

\section{Notation and Preliminaries}
After setting up the notation in Section \ref{sec:notation} we introduce the notion of gradient descent algorithms with noise injection in Section \ref{sec:setting}. We discuss the assumptions on the zero-loss set and the behaviour of the system around it in Section \ref{sec:gamma}. We give the necessary background and results on the convergence of the suitable type of stochastic processes (sometimes called Katzenberger processes after \cite{Katzenberger91}) in Section \ref{sec:katz}. Finally we introduce the characterization of the limit map following the derivation from \cite{LiWangArora22} in Section \ref{sec:limit-map}. 

\subsection{Notation}\label{sec:notation}

\paragraph*{Convergence of random variables}
We work with an abstract filtered probability space $(\Omega,\calF,\bbF,\Prob)$, on which all random variables are defined. For a sequence of $\R^m$-valued random variables $(X_n)_{n\N}$ we denote with $X_n \Rightarrow X$ the convergence in distribution to some random variable $X$, that is $\Expect f(X_n) \to \Expectation f(X)$ for all $f\in C_b(\R^m)$.

We denote with $D_{\bbR^m}[0, \infty)$ the space of $m$-dimensional \cadlag processes equipped with the Skorokhod topology (see e.g.~\cite[Chapter 3]{Billingsley1968}).

For two \cadlag paths $F,M\in D_{\bbR^m}[0, \infty)$, we denote with $\TVar_t(F)$ and $[M](t)$ the total and quadratic variation (see~\eqref{eq:def:total-quadratic:variation} for their definitions in the case of pure-jump paths).

\paragraph*{Estimates and bounds}
We use the floor operation by $\lfloor\cdot \rfloor : \R \to \bbZ$ to give the largest integer smaller than the argument, that is $\lfloor x \rfloor = \max\set*{ z\in \bbZ: z\leq x}$. Indices in sums will be always integers and we use shorthand notation $\sum_{k\leq x}$ for real positive $x$ to denote the sum $\sum_{k=0}^{\lfloor x \rfloor}$.

The notation $a\lesssim b$ is understood as $a\leq C b$ for some constant $C$ depending on stated assumptions, but never on the hyperparameters $\alpha,\sigma$. We also use the Landau notation $f= O(g)$ and $f= o(g)$ to denote that $|f|/|g|$ is bounded or converges to zero, respectively, with the limit under consideration  made clear in the context. 

\paragraph*{Vector and matrix operations}
We denote with $\lvert \cdot \rvert$ the usual Euclidean norm on $\R^m$. 
For some $v,w\in \R^m$, the usual vector product is denoted with $v\cdot w :=\sum_{i=1}^m v_i w_i$. In addition, the vector $v\odot w \in \R^m$ is defined by component-wise multiplication $(v\odot w)_i := v_i w_i$ for $i=1,\dots,m$. Likewise, we denote with $v\otimes w\in \R^{m\times m}$ the standard tensor product of vectors defined by $(v\otimes w)_{i,j} = v_i w_j$ for $i,j=1,\dots,m$.
For $x_0 \in \bbR^m$ and $r\in \bbR^+$ we use $B(x_0, r)$ to denote a ball with center $x_0$ and radius $r$. 

The product of two matrices $A,B\in \R^{m\times n}$ is given by $A:B:=\sum_{i=1}^m\sum_{j=1}^n A_{ij}  B_{ij}$.
For a linear map $A:\R^n\to \R^m$ we write $A^\dag: \R^m \to \R^n$ for its Moore-Penrose pseudoinverse. 
For a symmetric positive definite matrix $A\in \R^{m\times m}_{\text{sym}}$, we denote with $\lvert A \rvert_+$ the product of the non-zero eigenvalues.

\paragraph*{Derivatives}
For a map $\Psi:\R^m\to \R^n$, the first and second variations in $x\in \R^m$ are denoted by $\partial \Psi(x)$ and $\partial^2 \Psi(x)$. The action of these linear maps on directions $v\in \R^m$ and $\varSigma\in \R^{m\times m}$ is denoted by
\[
  \partial \Psi(x)[v] := \sum_{i=1}^m \partial_{x_i} \Psi(\xi) v_i \qquad\text{and}\qquad \partial^2 \Psi(x)[\varSigma] := \sum_{i = 1}^{m}\sum_{j = 1}^m \partial^2_{x_i, x_j}\Psi(x)\varSigma_{ij} \,.
\]
By contrast, for a scalar map $\Psi:\R^m\to\R$, the notations $\nabla \Psi$ and $\nabla^2 \Psi$ indicate the vector and matrix of first and second partial derivatives. This distinction in notation allows us to write e.g.\ $\partial_\eta\nabla_w L(w,\eta)[v]$ to indicate that the direction $v$ is to be contracted against the derivatives in $\eta$.
By slight abuse of notation, if $\varSigma = \sigma\oti\sigma$ for some $\sigma \in \R^m$, then we also  write $\partial^2 \Psi(x)[\sigma,\sigma]$ instead of $\partial^2 \Psi(x)[\sigma\oti \sigma]$.

\subsection{Problem Setting}
\label{sec:setting}
In this work we study the following class of noisy gradient descent systems.
\begin{definition}[Noisy gradient descent]\label{def:noiseLoss}
  Given a \emph{loss function} $L\in C^3\bigr(\R^m,\R\bigl)$, any function $\hat L \in C^3\bigl(\R^{m+d},[0,\infty)\bigr)$
  with $\hat{L}(w, 0) = L(w)$ is called a \emph{noise-injected loss function}. 
  A \emph{noisy gradient descent} for the noise-injected loss $\hat L$ is the dynamics given by
  \begin{subequations}
  	\label{eq:NoisyGD}
  	\begin{alignat}2
  		\label{eq:NoisyGD:1}
  		w_{k+1} &= w_k - \alpha \nabla_w \hat{L}(w_k, \eta_k),\qquad && \text{for }  k\in \bbN_0;  \\
  		\eta_{k, i} &\sim \rho(\sigma), \text{ i.i.d., }
		&&\text{for } i = 1,\dots , d \,. 
  		\label{eq:NoisyGD:2}
  	\end{alignat}
  \end{subequations}
  Here $w_0\in \R^m$ is a given initialization and $\alpha>0$ is the step size. 
  The family of probability distributions $\set*{\rho(\sigma)\in \ProbMeas(\R^m)}_{\sigma>0}$ characterizes the noise injection and $\eta_{k,i}~\sim \rho(\sigma)$ for $i=1,\dots, d$ are i.i.d.\@ random variables distributed according to $\rho(\sigma)$ such that $\eta_k$ is an $m$-dimensional vector. The family $\set{\rho(\sigma)}_{\sigma>0}$ is assumed to be centered with variance $\sigma^2$, that is
\begin{equation}\label{eq:ass:noise}
\Expectation_{\eta\sim \rho(\sigma) } \eta= 0 \qquad\text{and}\qquad  \Var_{\eta\sim \rho(\sigma) } \eta = \sigma^2.
  \end{equation}
\end{definition}
We require the noisy loss $\hat{L}$ together with the distribution $\rho$ to satisfy the following compatibility and growth assumptions.
\begin{assumption}[Compatibility between $\hat L$ and $\rho(\sigma)$]
\label{ass:hatL}
There exists $p \in [1,\infty)$ such that
\begin{enumerate}
	\item for any compact $K\subset \R^m$ there exists $C_1=C_1(K)$ and $C_2=C_2(K)$ with
\begin{alignat}2
	\mkern-32mu \bigl|\nabla_w\nabla_\eta^2\hat{L}(w, \eta) - \nabla_w\nabla_\eta^2\hat{L}(\tilde w, \eta)\bigr| &\leq C_1(1+|\eta|^p)|w - \tilde w| \, , \ \ &&\text{ for all $w, \tilde w \in K$ and all $\eta\in \R^d$,} 
	\label{ass:compat-L-Lhat-difference-in-w}\\
	\mkern-32mu \bigl|\nabla_w\nabla_\eta^2\hat{L}(w, \eta) - \nabla_w\nabla_\eta^2\hat{L}(w, 0)\bigr| &\leq C_2|\eta|\bigl( 1 + |\eta|^{p-1} \bigr) \,, 
	&&\text{ for all $w \in K$ and all $\eta\in \R^d$.}
	\label{ass:compat-L-Lhat-difference-in-etas}
\end{alignat}
\item Setting
\[
	\M_k(\sigma) := \bbE_{\eta\sim \rho(\sigma) } |\eta|^{k} ,
\]
we have for all $k\in \{3 , \dots, 2(p+2)\}$,
\begin{subequations}
\label{eq:ass:moment}
\begin{align}
	\text{if $C_2>0$: }&\quad \M_k(\sigma) = o(\sigma^2) \,,\label{eq:ass:moment-C2>0}
	\\
	\text{if $C_2=0$: }&\quad \M_k(\sigma) = O(\sigma^2) \,.\label{eq:ass:moment-C2=0}
\end{align}	
\end{subequations}	
If $\sigma$ is clear from the context, 
we briefly write $\M_k:=\M_k(\sigma)$.
\end{enumerate}
\end{assumption}
\begin{remark}[Examples of noise distributions]
	\label{remark:noise}
	The conditions~\eqref{eq:ass:noise} and bounds~\eqref{eq:ass:moment} for any $p\geq 1$ are satisfied for the centered Gaussian distribution $\calN(0, \sigma^2)$ and the uniform distribution $\Uni(-\sqrt{3}\sigma, \sqrt{3}\sigma)$
	Note that in both cases we have $\M_p = O(\sigma_n^p)$ and thus for $q\geq 2(1+2)= 6$ we have $\M_{q} = O(\sigma_n^6)$.
\end{remark}

\begin{remark}[More general noise properties]
	In the proofs we actually do not require the noisy variables $\eta_{k, i}^n$ to be sampled from the same distribution. It is sufficient that the $\eta_{k, i}$  are independent and that for every $i$ the conditions~\eqref{eq:ass:noise} and~\eqref{eq:ass:moment} are satisfied. We stick to the assumption of $\eta_{k, i}$ being i.i.d.\ for the purpose of readability.
\end{remark}

A noisy gradient system as in Definition~\ref{def:noiseLoss} is characterized by two hyperparameters, the step size $\alpha$ and the noise variance $\sigma$. For the main results of this paper we give ourselves two positive sequences $(\alpha_n)_{n\in\bbN}$ and $(\sigma_n)_{n\in\bbN}$ and study the limit of interpolations $(\tilde{w}_n)_{n\in \bbN}$, where $\tilde{w}_n : \bbR_+ \to \bbR^m$ is the $m$-dimensional \cadlag process defined by
\begin{subequations}\label{eq:system:w}
\begin{equation}\label{eq:def:w:interpolation}
\tilde{w}_n(t) = w^n_{\lfloor\frac{t}{\alpha_n}\rfloor} \qquad\text{for every $t\in \bbR_+$}\,.
\end{equation}
Here $\lfloor x \rfloor $ is the integer part of $x$,
and $\set{w^n_{k}}_{k\in \bbN_0}$ is the noisy gradient descent for hyperparameters $(\alpha_n)_{n\in\bbN}$ and $(\sigma_n)_{n\in\bbN}$, given by
\begin{alignat}2
	\label{eq:NoisyGDn:1}
	w_{k+1}^n &= w_k^n - \alpha_n \nabla_w \hat{L}(w_k^n, \eta_k^n),\qquad && \text{for }  l\in \bbN_0;  \\
	\eta_{k, i}^n &\sim \rho(\sigma_n), &&\text{for } i = 1,\dots , d \,. 
	\label{eq:NoisyGDn:2}
\end{alignat}
\end{subequations}
Depending on the structure of $\hat{L}$ we consider two scaling regimes, where either both $\alpha_n, \sigma_n \to 0$ or only $\alpha_n \to 0$ for constant noise variance $\sigma_n=\sigma\geq 0$.  Consider the space  $D_{\bbR^m}[0, \infty)$ of all $m$-dimensional \cadlag processes equipped with the Skorokhod topology. We  study the limiting behaviour of the interpolations $\tilde{w}_n(t)$ in $D_{\bbR^m}[0, \infty)$ for the two different cases and characterize the limit processes in terms of the noise-injected loss function $\hat L$.
 
\subsection{Zero-loss set}
\label{sec:gamma}

The key assumption in our analysis is the structure of the zero-loss set $\Gamma = \{w\in \R^m: L(w)=0\}$. We consider systems in which the zero-loss set is locally a $C^2$ manifold satisfying some non-degeneracy assumptions. 
The zero loss set is defined in terms of the $\omega$-limit set of the flow associated to the gradient flow of $L$ on $\R^m$, that is
\begin{equation}\label{eq:ds}
	\dot x(t) = -\nabla L(x(t)) \quad\text{with}\quad x(0)= x_0\in \R^m\,.
\end{equation}
We define the flow map $\phi: \R^m \times [0,\infty) \to \R^m$ as the solution of~\eqref{eq:ds}, and we have the integral representation
\begin{equation}\label{eq:flow}
   \phi(x, t) = x - \int_0^t\nabla L\bra*{\phi(x,s)} \dx{s} \,.
\end{equation}
Then, any initial point $x_0\in \R^m $ has an $\omega$-limit set which is denoted by 
\begin{equation}
	\omega(x_0) = \bigcap_{s>0} \overline{\set*{ \phi(x_0,t) : t >s}} . 
\end{equation}
To establish local attractiveness of the loss manifold, we require the loss function $L:\R^m\to [0,\infty)$ to satisfy a number of non-degeneracy assumptions.
\begin{assumption}[Non-degeneracy of the loss manifold]\label{ass:LossMfd}$ $ \\
The set $\Gamma :=  \set*{x\in \R^m: L(x)=0}$ is  a \emph{non-degenerate loss manifold} in the following sense: 
\begin{enumerate}    
    \item \emph{manifold: } it forms an $M$-dimensional $C^2$ manifold;
    \item \emph{constant rank: } the rank of $\nabla^2 L$ is constant and maximal on $\Gamma$, that is $\Rank( \nabla^2 L(x))= m-M$ for all $x\in \Gamma$;
    \item \emph{spectral gap: } there exists $\delta > 0$ such that for all~$x\in \Gamma$, all non-zero eigenvalues $\lambda$ of $\nabla^2 L(x)$ satisfy $\lambda >\delta$.
\end{enumerate}
\end{assumption}

By the manifold condition in Assumption~\ref{ass:LossMfd}, there exists a tangent space $T_{x_0}\Gamma$  for each $x_0\in \Gamma$. Moreover, for any $x_0\in \Gamma$, there exists an $\epsilon >0$ such that the projection operator $P_\Gamma: B_{\epsilon}(x_0) \to \Gamma$ is well-defined.
Then, the rank and spectral gap condition in Assumption~\ref{ass:LossMfd} ensure that locally for every $x_0 \in \Gamma$ and every non-tangent direction $v\notin T_{x_0}\Gamma$ the loss function locally grows quadratically, that is there exists some $c>0$ such that 
\[
	L(x_0 + \epsilon v) \geq  c \epsilon^2 \, |(I - P_\Gamma)v|^2 \,.
\]
The assumption is satisfied by many existing overparameterized systems. In particular, overparameterized linear models~\cite{LiWangArora22} and feedforward neural networks~\cite{cooper2018loss} have been shown to satisfy conditions similar to Assumption~\ref{ass:LossMfd}.

\begin{remark}[Localizing the smoothness requirement]
For many interesting losses, for instance those based on $\ReLU$ neural networks, the function $L$ and the zero-loss set $\Gamma$ are not differentiable, and Assumption~\ref{ass:LossMfd} is not satisfied. For those systems, the results of this paper  can be adapted by localizing.  Since the main theorems characterize training  behaviour up to the time of leaving a compact set $K$, the results can be applied within a set $K$ such that $\Gamma\cap K$ and $L|_K$ do satisfy Assumption~\ref{ass:LossMfd}.
\end{remark}

We study the behaviour of the noisy gradient descent~\eqref{eq:system:w} in the proximity of $\Gamma$ and thus require initial conditions that guarantee the convergence of the flow map $\psi$ in~\eqref{eq:flow} to $\Gamma$. 
In terms of the $\omega$-limit, we define a \emph{locally attractive} neighbourhood $U$ of the zero loss manifold~$\Gamma$ such that the deterministic gradient flow trajectory~\eqref{eq:ds} with initial conditions within $U$ converges to a point on $\Gamma$.
\begin{definition}\label{def:LocAttract}
An open set $U\in \bbR^m$ is a \emph{locally attractive} neighbourhood of $\Gamma$ if $\Gamma \subset U$ and there exists a map $\Phi \in C^2(U;\Gamma)$ which satisfies $\omega(x)=\set*{\Phi(x)}$ for all $x\in U$. 
In this case, $\Phi$ is called the \emph{limit map} and satisfies 
\begin{equation}
	\label{eqdef:Phi}
\Phi(x) = \lim_{t\to\infty} \phi(x, t)
\quad\text{for all }x\in U.
\end{equation}
\end{definition}
The existence of a locally attractive neighbourhood is guaranteed by~\cite[Proposition~3.5]{Katzenberger91} whenever~$\Gamma$ satisfies Assumption~\ref{ass:LossMfd}.
In addition, we use the regularity of the limit map for which a direct application of~\cite[Theorem 5.1]{falconer1983differentiation} guarantees that if $L$ is three times differentiable with locally Lipschitz third derivative, the limit map satisfies $\Phi\in C^2(U)$~\cite[Corollary 5.1]{Katzenberger91}.

\begin{remark} Note that for initial conditions $x_0\in U$ the (unperturbed) gradient flow~\eqref{eq:ds} converges to $x^* =\Phi(x)$. As remarked in the introduction, the noise injection drastically changes this behaviour, since instead of converging to a stationary point the system then approximately follows a manifold-constrained deterministic or stochastic flow.
\end{remark}

\begin{proposition}[Exponential convergence of the flow]\label{prop:exp:convergence}
	Let $U$ be a locally attractive neighbourhood (Definition~\ref{def:LocAttract}) of the loss manifold $\Gamma$ satisfying Assumption~\ref{ass:LossMfd}, then there exists $\beta>0$ such that for any $W(0)\in U$ exists $C>0$ with
	\[
	\bigl| \phi\bigl(W(0), t\bigr) - \Phi\bigl(W(0)\bigr) \bigr| \leq C e^{-\beta  t} \,.
	\]
	In particular, the constants $C>0$ can be chosen uniformly among $W(0)\in K \cap U$ for any compact set $K$.
\end{proposition}
\begin{proof}
	The proof is not exactly contained in~\cite{Katzenberger91}, however it is implied by~\cite[Lemma 3.2]{Katzenberger91} together with a standard localization procedure as used in the proof of~\cite[Lemma 3.3]{Katzenberger91}. 
\end{proof}

\subsection{Katzenberger's Therorem}
\label{sec:katz}
In this section we present a simplified version of the main tool of our analysis, Katzenberger's Theorem~6.3~\cite{Katzenberger91}. 
We fix a filtered probability space $(\Omega,\calF,\bbF,\Prob)$ and consider formal `stochastic differential equations' of the form
\begin{equation}
 \label{eq:ds-per}   \dd X_n = -\nabla L(X_n) \dd A_n + \dd Z_n,
\end{equation}
or in integrated form
\[
X_n(t) = X_n(0) - \int_0^t \nabla L(X_n(s)) \dd A_n(s)
+  Z_n(t),
\]
where $(A_n)_{n\in \bbN}$ is a sequence of deterministic non-decreasing \cadlag processes, also called an~\emph{integrator sequence}, and $(Z_n)_{n\in \bbN}$ is a sequence of $\bbR^m$-valued semimartingales with respect to $\Prob$. We assume that $Z_n(0)=0$. 

\begin{remark}[Generality of~\eqref{eq:ds-per}]
Katzenberger's setup in~\cite{Katzenberger91} is more general than~\eqref{eq:ds-per}; in particular, it allows for a splitting of the process $Z_n$ into a `noise' and a `mobility' part, with different assumptions on these two parts. For our purposes the simpler form~\eqref{eq:ds-per} suffices.
\end{remark}

We impose the following conditions on the integrator sequence $A_n$.
\begin{assumption}[Assumptions on $(A_n)_{n\in \bbN}$] \label{ass:Katzenberger:noise}
	The family $(A_n)_{n\in \bbN}$ of non-decreasing \cadlag deterministic processes in~\eqref{eq:ds-per} satisfies	
\begin{enumerate} 
	\item $A_n(0) = 0$;
    \item $A_n$ asymptotically puts infinite mass onto every interval, that is, for every $\delta>0$, 
    \begin{equation}\label{ass:infinite:noise}
    \inf_{0\leq t\leq T} \bigl(A_n(t+\delta) - A_n(t)\bigr) \longrightarrow \infty, \qquad\text{ as $n\to \infty$\,;}
	\end{equation}
    \item $A_n$ is asymptotically continuous, that is, 
    \begin{equation}\label{ass:An:asympt:cont}
      \sup_{t\geq 0} \bigl( A_n(t) - A_n(t-)\bigr) \longrightarrow 0, \qquad\text{as } n\to \infty,
\end{equation}
    with $A_n(t-):= \lim_{s\nearrow t} A_n(s)$.
\end{enumerate}
\end{assumption}
For any compact $K \subset U_\Gamma$ we define the stopping time
\begin{equation}\label{eq:def:stopping_time}
\lambda_n(K)=\inf \left\{t \geq 0 \mid X_n(t) \notin K^{\circ} \right\}
\end{equation}
and for a process $Y: [0,\infty)\to \R^d$ the notation $Y^\lambda$ for denotes the stopped process
\begin{equation}\label{eq:def:stopped}
	Y^\lambda(t):= Y(t{\wedge} \lambda) \qquad\text{ for all } t\in [0,\infty) \,.
\end{equation} 
The family $(Z_n)_{n\in \bbN}$ of $\bbR^m$-valued semimartingales in~\eqref{eq:ds-per} needs to satisfy the following two assumptions. 
\begin{assumption}[Vanishing increments of $(Z_n)_{n\in \bbN}$]\label{ass:Katzenberg:Zn}
	 For all  $T \in \bbR_+$ and all compact $K\subset U_\Gamma$ we have 
	\[ 
		\sup_{0<s\leq T \wedge \lambda_n(K)} |\Delta Z_n(s)| \Rightarrow 0, \qquad\text{as } n\to\infty ,
	\]
	where $\Delta Z_n(s) := \Delta Z_n(s) - \Delta Z_n(s-)$ denotes the increment. 
\end{assumption}
In the following we use the notation $\TVar_t(F)$ and $[M](t)$ to denote the total and quadratic variation of \cadlag paths $F,M\in D_{\mathbb{R}^d}[0, \infty)$. For pure-jump paths, which is the only case we will be using, these are  defined by
\begin{equation}\label{eq:def:total-quadratic:variation}
	\TVar_t(F) := \sum_{0<s<t} \abs*{ \Delta F(s)} \qquad\text{and}\qquad [M](t) := \sum_{0<s<t} \abs*{ \Delta M(s)}^2 \,.
\end{equation}
\begin{assumption}
\label{assum:4.2}
The family $(Z_n)_{n\bbN}$ is a sequence of semimartingales with sample paths in $D_{\mathbb{R}^d}[0, \infty)$. 
For every $n \geq 1$ there exist stopping times $\left\{\tau_n^k \mid k \geq 1\right\}$ and a decomposition of $Z_n$ into a local martingale $M_n$ plus a finite variation process $F_n$ such that 
\[
\Prob\bigl[\tau_n^k \leq k\bigr] \leq 1 / k\qquad\text{and}\qquad \bigl([M_n](t \wedge \tau_n^k)+\TVar_{t \wedge \tau_n^k}(F_n)\bigr)_{n\in \bbN}
\]
 is uniformly integrable for every $t \geq 0, k \geq 1$ and
\begin{equation}\label{eq:ass:4.2:VarVanish}
\lim _{\gamma \rightarrow 0} \limsup _{n \rightarrow \infty} \Prob\biggl[\sup _{0 \leq t \leq T}\bigl(\TVar_{t+\gamma}(F_n)-\TVar_t(F_n)\bigr)>\varepsilon\biggr]=0,
\end{equation}
for every $\varepsilon>0$ and $T>0$.
\end{assumption}

\begin{remark}
Assumption \ref{assum:4.2} requires both the quadratic variation of the martingale part of the noise and the total variation of the drift to be bounded. One can interpret this as a requirement to accumulate an order one perturbation for any fixed time $t>0$ in the limit $n\to \infty$. In our analysis we encounter cases when one of the components dominates, namely the drift in the general case and the martingale part in the degenerate case. Nevertheless, in the general case Katzenberger's theorem allows for a balanced contribution of both terms.
\end{remark}

Before introducing Katzenberger's result we need to do a shift of the variable $X_n$.
If we consider a solution $X_n$ of the system~\eqref{eq:ds-per} with $X_n(0) = x_0\in U_\Gamma$ but $x_0 \notin \Gamma$, then for any $t>0$ we have  $A_n(t) \to \infty$ and thus by definition of $U_\Gamma$ we have
\[ 
  \phi(x_0, A_n(t)) \longrightarrow \Phi(x_0)\in \Gamma \qquad\text{ as } n\to \infty \,.
\]
As we can take arbitrarily small $t$, by a simple Gronwall inequality argument, we obtain that the limiting process $X$ for $n\to \infty$  must be discontinuous at $t=0$. To avoid this discontinuity we consider the shifted process
\begin{equation}\label{eq:def:Yn}
	Y_n(t)=X_n(t)-\phi\left(X_n(0), A_n(t)\right)+\Phi\left(X_n(0)\right).
\end{equation}
Note that at $t=0$ we have $\phi\left(X_n(0), A_n(0)\right) = X_n(0)$ and thus $Y_n(0) = \Phi\left(X_n(0)\right) \in \Gamma$. The main theorem of~\cite{Katzenberger91} states that the limiting dynamics of $Y_n(t)$ lie on $\Gamma$ and can be expressed in terms of the limit map $\Phi$ and the limit~$Z$ of the process~$Z_n$.

\begin{theorem}[{\cite[Theorem 6.3]{Katzenberger91}}]
\label{th:Katzenberg}
Assume that the loss manifold $\Gamma$ satisfies Assumption~\ref{ass:LossMfd}.
Assume that $L\in C^3(U_\Gamma)$, for a neighbourhood $U_\Gamma$ of $\Gamma$.
Assume that the shifted processes $(X_n)_{n\in\bbN}$~\eqref{eq:ds-per} satisfy Assumptions~\ref{ass:Katzenberger:noise}, \ref{ass:Katzenberg:Zn}, and \ref{assum:4.2} with $X_n (0) \Rightarrow X(0) \in U_{\Gamma}$. For a compact $K \subset U_{\Gamma}$, let
\begin{equation}\label{eq:def:KatzStop}
\mu_n(K) := \inf\bigl\{t \geq 0 \mid Y_n(t) \notin K^{\circ} \bigr\} \,.
\end{equation}
Then, for every compact $K \subset U_{\Gamma}$, the sequence $\bigl(Y_n^{\mu_n(K)}, Z_n^{\mu_n(K)}, \mu_n(K)\bigr)_{\bbN}$ of stopped processes and stopping times is relatively compact in $D_{\mathbb{R}^d \times \mathbb{R}^m}[0, \infty) \times[0, \infty]$. If $(Y, Z, \mu)$ is a limit point of this sequence, then $(Y, Z)$ is a continuous semimartingale, $Y(t) \in \Gamma$ for a.e.\@ $t\in [0,\infty)$, $\mu \geq \inf \left\{t \geq 0 \mid Y(t) \notin K^{\circ}\right\}$ a.s., and
\begin{align}
\label{eq:th:Katzenberg:equation-for-Y}
Y(t)=  Y(0) & +\int_0^{t \wedge \mu} \partial \Phi(Y) \dx Z 
+\frac{1}{2} \sum_{i j} \int_0^{t \wedge \mu} \partial^2_{i j} \Phi(Y) \dx [Z^i, Z^j] .
\end{align}
\end{theorem}
\begin{remark}
For a given process $G_t$ on the locally attractive neighbourhood $U_\Gamma$ (Definition~\ref{def:LocAttract}), the process $\Phi(G_t)$ satisfies $\Phi(G_t) \in \Gamma$ by definition. In addition, we note that $\Phi$ satisfies $\Phi(x) = x$ for all $x\in\Gamma$ so if $G_t\in \Gamma$ almost surely, we also have $G_t = \Phi(G_t)$ almost surely. Thus, if a sequence of stochastic processes converges to a process on the zero-loss manifold, the limiting process must coincide with its image under map $\Phi$. Thus, Theorem \ref{th:Katzenberg} can be interpreted as an application of It\=o's lemma to the limit map $\Phi$. 
\end{remark}

\subsection{Characterization of the limit map}
\label{sec:limit-map}
The characterization of the limit behaviour in Theorem~\ref{th:Katzenberg} makes use of the first and second derivatives of the  limit map $\Phi$ that was defined in Definition~\ref{def:LocAttract}. In this section we recall the relevant characterizations of these derivatives that were proved in~\cite{LiWangArora22}.

For a linear map $A$ we write $A^\dag$ for its Moore-Penrose pseudoinverse. 
For $H\in \bbR^{k\times k}$,  we define the \emph{Lyapunov} operator $\calL_H: \R^{k\times k} \to \bbR^{k\times k}$ by
\[
\calL_H(X) := H^\top  X + X H.
\]
For a non-negative symmetric matrix $A$, we define $|A|_+$ to be its `pseudo-determinant', the product of its non-zero eigenvalues.
With this preliminary considerations, we can refer to the first and second derivatives of $\Phi$ contained in~\cite[Lem.\ 4.5 and Cor.\ 5.1 \& 5.2]{LiWangArora22}.
\begin{lemma}[First and second derivatives of $\Phi$]
	\label{t:derivatives-of-Phi}
Let $L \in C^3(\R^m,\R_+)$ and assume that $\Gamma$ is a lower-dimensional manifold in~$\R^m$ of class~$C^1$ satisfying Assumption~\ref{ass:LossMfd}. 
\begin{enumerate}
\item For any $\xi_0\in \Gamma$ the limit below exists, and the identities hold:
\begin{equation}
	\label{eq:lem:derivatives-of-Phi:first-derivative}
\nabla \Phi(\xi_0) = \lim_{t\to\infty}e^{-t\nabla^2L(\xi_0)} = P_{T_{\xi_0}\Gamma}. 
\end{equation}
\end{enumerate}
\noindent
Here $P_{T_{\xi_0}\Gamma}$ is 
the orthogonal projection onto $T_{\xi_0}\Gamma$. 
We write $P = P_{{T_{\xi_0}\Gamma}}$ for short, and $Q := Q_{T_{\xi_0}\Gamma} := I - P_{T_{\xi_0}\Gamma}$ for the corresponding orthogonal projection onto $T_{\xi_0}\Gamma^\perp$. 
\begin{enumerate}[resume]
\item The second derivative $\partial^2 \Phi$ is characterized  by
\begin{equation}
\label{char:Phi''-a}
\begin{split}
\partial^2\Phi(\xi_0)[\varSigma]
&= (\nabla^2L)^{\dagger} \partial^2(\nabla L)\bigl[P\varSigma P\bigr]
-P \partial^2(\nabla L)\bigl[\mathcal{L}_{\nabla^2 L}^{\dagger}(Q\varSigma Q)\bigr] \\
&\qquad {}+2 P\partial^2(\nabla L)\bigl[(\nabla^2L)^{\dagger} Q\varSigma P\bigr],
\end{split}
\end{equation}
for any symmetric $\varSigma \in \bbR^{m\times m}$.

\item For the special case of the identity matrix, $\varSigma = I_m$, we have
\begin{equation}
\label{char:Phi''-Sigma=id}
\begin{split}
\partial^2\Phi(\xi_0)[I_m]
&= (\nabla^2L)^{\dagger} \partial^2(\nabla L)\bigl[P\bigr]
-P \nabla \log |\nabla^2 L|_+ \,.
\end{split}
\end{equation}

\item For the special case $\varSigma = \nabla^2L(\xi_0)$ we have 
\begin{equation}
\label{char:Phi''-Sigma=nabla2L}
\begin{split}
\partial^2\Phi(\xi_0)[\nabla^2 L]
&= -\frac12 P \nabla \Delta L(\xi_0).
\end{split}
\end{equation}
\end{enumerate}
\end{lemma}

\section{Main Results}
\label{s:main-results}

The main theorems of this paper are Theorems~\ref{th:main} and~\ref{th:main:deg2} below, which were already mentioned in the introduction as Theorems~\ref{fth:main1} and~\ref{fth:main2}. Both characterize convergence of time-rescaled noisy gradient systems to a limiting evolution that is a constrained ODE or SDE on $\Gamma$. 

The starting point for both theorems is the dynamics of the parameters $w^n$ from~\eqref{eq:NGD-intro} or equivalently~\eqref{eq:NoisyGDn:1}, which we rewrite as 
\begin{align}
 &= w^n_k - \alpha_n \nabla_w L(w^n_k) + \alpha_n\bigl(\nabla_w \hat{L}(w^n_k, 0) - \nabla_w \hat{L}(w^n_k, \eta^n_k)\bigr) \,,
\end{align}
where we used that $\hat L(w,0)= L(w)$ by Definition~\ref{def:noiseLoss}. As described in the Introduction, in order to follow the evolution on $\Gamma$ we need to speed up the process, by a factor $1/\alpha_n\sigma_n^2$ in the non-degenerate case or $1/\alpha_n^2\sigma_n^2$ in the degenerate case.

\subsection{Non-degenerate case}
In the non-degenerate case the rescaling by $1/\alpha_n/\sigma_n^2$ leads us to consider the sequence of processes 
\begin{equation}
	\label{eqdef:Wn}
W_n(t) = \tilde{w}_n\biggl(\frac t{\sigma_n^2}\biggr) = w^n_{\bigl\lfloor \frac t{\alpha_n\sigma_n^2}\bigr\rfloor},
\end{equation}
where $\tilde w$ and $w^n$ are  the solution to~\eqref{eq:system:w}. Here $\alpha_n, \sigma_n \to 0$ are chosen in some way to be specified.
Out of the sequences $(\alpha_n)_{n\in \bbN}$, $ (\sigma_n)_{n\in \bbN}$ we define  the sequence of integrators 
\begin{equation}
	\label{eqdef:An}
	A_n(t) := \alpha_n\bigl\lfloor\frac{t}{\alpha_n\sigma_n^2}\bigr\rfloor.
\end{equation}
For any~$n$ the dynamics of $W_n$ can then be written in the form~\eqref{eq:ds-per} as
\begin{subequations}\label{eq:def:W}
\begin{align}
\dx W_n(t) &= -\nabla L(W_n(t))\dd A_n(t) + \dd Z_n(t), \label{eq:def:general:Wn} \\
Z_n(t) &= \sum_{s\leq t} \Delta Z_n(s) = \alpha_n  \sum_{k\leq \frac{t}{\alpha_n\sigma_n^2}}\bigl(\nabla_w \hat{L}(W_n(\alpha_n\sigma_n^2 k), 0) - \nabla_w \hat{L}(W_n(\alpha_n \sigma_n^2k), \eta^n_k)\bigr), \label{eq:def:general:Zn}\\
\eta^n_{k,i} &\sim  \rho(\sigma_n). \label{eq:def:general:etan}
\end{align}
\end{subequations}

\begin{theorem}[Main convergence theorem in the non-degenerate case]
\label{th:main}
Consider a loss function $L$ and noise-injected loss $\hat L$ in the sense of Definition~\ref{def:noiseLoss} with loss manifold $\Gamma$ satisfying Assumption~\ref{ass:LossMfd}. Let $W_n(0) \Rightarrow W_0\in U_{\Gamma}$ for some locally attractive neighbourhood $U_\Gamma$ of $\Gamma$ in the sense of Definition~\ref{def:LocAttract}. Let $\rho$ satisfy the assumption~\eqref{eq:ass:noise} and let $\alpha_n,\sigma_n \to 0$ as $n\to \infty$. Let $\hat{L}$ and $\rho$ satisfy the compatibility Assumption \ref{ass:hatL} for some $p \geq 1$ and assume 
\begin{equation}\label{ass:speed}
\sup_{k \leq 1/\alpha_n\sigma_n^2}\alpha_n|\eta^n_k|^{p+2} \Rightarrow 0\,, \qquad\text{ as } n\to \infty.
\end{equation}
Let $W_n$ be a solution to~\eqref{eq:def:W} and the shifted process $Y_n$ be defined as in~\eqref{eq:def:Yn} by
\begin{equation}\label{eq:def:general:Yn}
Y_n(t):=W_n(t)-\phi\bigl(W_n(0), A_n(t)\bigr)+\Phi\left(W_n(0)\right).
\end{equation}
For compact $K \subset U_{\Gamma}$, define the exit time of $K$ by
\[
\mu_n(K):=\inf\bigl\{t \geq 0 \mid Y_n(t-) \notin K^{\circ} or \ Y_n(t) \notin K^{\circ} \bigr\} \,.
\]

Then for any compact set $K \subset U_{\Gamma}$, the sequence $\bigl(Y^{\mu_n(K)}_n, \mu_n\bigr)_{n\in \bbN}$ is relatively compact in the Skorokhod topology. Moreover, for any limit point $(Y, \mu)$ of $\bigl(Y^{\mu_n(K)}_n, \mu_n\bigr)$, $Y$ is a continuous function of time, it satisfies $Y(t) \in \Gamma$ a.s. for any $t$,  and 
\begin{equation}
  \label{eqdef:th:main:limiting-dynamics}
Y(t)=  \Phi(W_0) - \int_0^{t \wedge \mu}P_{T_{Y(s)}\Gamma} \nabla_w \Reg(Y(s)) \dx s,
\qquad
\Reg(w) := \frac12 \Delta_\eta \hat{L}(w, 0)\,,
\end{equation}
where $P_{T_{\xi}\Gamma}$ is the orthogonal projection onto the tangent space of $\Gamma$ at the point $\xi\in \Gamma$. 
In addition, 
\begin{equation}
	\label{ineq:th:main:mu-and-tau}
\mu \geq  \inf\set[\big]{t\geq 0\mid Y(t)\notin K^\circ}.	
\end{equation}
\end{theorem}

\noindent
The limiting equation~\eqref{eqdef:th:main:limiting-dynamics} is a deterministic evolution equation for $Y$, and as long as $t<\mu$ it can be written in differential form as
\begin{equation}
	\label{eq:constr-GF}	
\frac{d}{dt} Y(t) = -P_{T_{Y(t)}\Gamma} \nabla_w \Reg(Y(t)),
\qquad\text{with } Y(t)\in \Gamma \text{ for  }t\geq 0,\quad\text{and } Y(0) = \Phi(W_0).
\end{equation}
It is a constrained gradient flow of the functional $\Reg$, under the constraint that $Y\in \Gamma$; the role of the projection $P_{T_Y\Gamma}$ is to project the vector field $-\nabla \Reg$ onto the tangent space of $\Gamma$ at~$Y$, which is necessary to maintain $Y\in \Gamma$.  

\begin{remark}[Convergence of the whole sequence]
If the functional $\Reg$ is Lipschitz continuous (e.g.\ if $\hat L\in C^4$) then solutions of the constrained gradient flow~\eqref{eqdef:th:main:limiting-dynamics} or~\eqref{eq:constr-GF} are unique up to time $\mu$. This implies that defining
\[
\tau := \inf\set[\big]{t\geq 0\mid Y(t)\notin K^\circ},
\]
we have the Skorokhod convergence of the full sequence up to time $\tau$,
\[
Y_n^{\mu_n(K)\wedge \tau} \Rightarrow Y^\tau. 
\] 
In addition, the exponential estimate~\eqref{ineq:exponential-convergence-W-Y} below then implies that for any $\delta>0$, we get the convergence $W_n^{\mu_n(K)\wedge \tau}|_{[\delta,\infty)}\Rightarrow Y^\tau|_{[\delta,\infty)}$ in Skorokhod topology.

Note that we can not exclude the existence of different limit points $(Y,\mu)$. Imagine, for instance, that~$\partial K$ contains part of a solution curve of the gradient flow~\eqref{eq:constr-GF}; then one can easily understand how for some $n$ the hitting time $\mu_n(K)$ may be triggered earlier than for others. The uniqueness argument above, however, shows that as long as $Y(t)\in K^\circ$, i.e.\ $t< \tau$, all limit points $Y$ coincide. 
\end{remark}

\begin{remark}[Skorokhod and locally-uniform convergence]
	\label{rem:Skorokhod-loc-unif}
Convergence to a continuous process in Skorokhod topology implies uniform convergence on compact time intervals. Moreover, convergence in distribution to a deterministic object implies convergence in probability, so for any $\e> 0$ we have
\[
\lim_{n\to \infty}\bbP\biggl[\sup_{0\leq s\leq t} \bigl\lvert Y^{t \wedge\mu_n(K)}_n(s) -Y^{t \wedge\mu(K)}(s) \bigr\rvert > \e\biggr] \to 0. 
\]
\end{remark}

\begin{remark}[Convergence of $Z_n$]
	In the proof we also show that the sequence of noise processes $Z_n$ converges to a deterministic process. This behaviour corresponds to the case in which the total variation term in Assumption~\ref{assum:4.2} dominates the quadratic variation.
\end{remark}

\begin{proof}[Proof of Theorem~\ref{th:main}]
The proof consists of two parts: the first part is an application of Theorem~\ref{th:Katzenberg}, which leads to the characterization~\eqref{eq:th:Katzenberg:equation-for-Y}. In the second step we convert that equation to a more explicit form, by giving an explicit characterization of the limit process~$Z$.

For both parts it is convenient to collect a number of properties of the process $Z_n$. We set 
\begin{subequations}\label{seq:Lem45}
\begin{align}
\Delta \sfZ_n(w,\eta) &:= \alpha_n \bigl(\nabla_w \hat{L}(w, 0) - \nabla_w \hat{L}(w, \eta)\bigr)\,,\quad \text{so that}\quad 
\Delta Z_n(\alpha_n \sigma_n^2 k)
= \Delta\sfZ_n(W_n(\alpha_n\sigma_n^2 k), \eta_k^n)\,;\\
\Delta \sfF_n(w) &:= \Expectation_{\eta\sim\rho(\sigma_n)} \pra*{\Delta\sfZ_n(w,\eta)}\, ; \label{eq:def:DeltaFn} \\
\Reg(w) &:= \frac12 \Delta_\eta \hat L(w,0)\, .
\end{align}
\end{subequations}
Similarly to $Z_n$ we assemble the jumps $\Delta \sfF_n$ into a process with jumps at times separated by~$\alpha_n\sigma_n^2$, 
\[
F_n(t) := \sum_{s\leq t} \Delta F_n(s) := 
\sum_{k\leq \frac{t}{\alpha_n\sigma_n^2}} \Delta \sfF_n(W_n(\alpha_n\sigma_n^2 k)).
\]
\begin{lemma}[Properties of $Z_n$ and $F_n$]
	\label{lem:props-Zn}
Let the compact set $K\subset \R^m$ and the sequences $
\alpha_n$, $\sigma_n$ be as in Theorem~\ref{th:main}. Let $T>0$. Let $\hat K = K + \overline {B(0, C)}$, where $C$ is the corresponding constant in Proposition~\ref{prop:exp:convergence}. Then $W_n(t) \in \hat K$ for all $t\leq T\wedge \mu_n(K)$ and
\begin{align}
&\sup_{t\leq T\wedge \mu_n(K)} \abs*{\Delta Z_n(t)} \Rightarrow 0 \qquad \text{as }n\to\infty\,;\label{lem:props-Zn:Ass39}\\
&\sup_{w\in \hat K} \abs*{\Delta\sfF_n(w) +\alpha_n\sigma_n^2 \nabla \Reg(w)} = o(\alpha_n\sigma_n^2)
\qquad \text{as }n\to\infty\, ;\label{lem:props-Zn:Fn-Reg}\\
&\Expectation_{\eta\sim\rho(\sigma_n)} \pra[\Big]{\sup_{w\in \hat K} \bra*{\Delta \sfZ_n(w,\eta)}^2}+ \sup_{w\in \hat K} \bra*{\Delta \sfF_n(w)}^2 = o(\alpha_n\sigma_n^2)\qquad \text{as }n\to\infty\, .
\label{lem:props-Zn:quadratic-variation}
\end{align}
\end{lemma}

We prove this lemma below, and continue in the meantime with the proof of Theorem~\ref{th:main}. 
To apply Theorem~\ref{th:Katzenberg} we verify Assumptions~\ref{ass:Katzenberger:noise}, \ref{ass:Katzenberg:Zn}, and~\ref{assum:4.2}. Note that we fix the set $K$ for once and for all, and we show that the \emph{stopped} processes $Z_n^{\mu_n(K)}$ satisfy Assumptions~\ref{ass:Katzenberg:Zn} and~\ref{assum:4.2}. 

\medskip
\noindent
\emph{Part 1: Verification of the assumptions of Theorem~\ref{th:Katzenberg}.} 

\emph{Verification of Assumption~\ref{ass:Katzenberger:noise}.}
The condition~\eqref{ass:infinite:noise} on the integrator sequence $A_n(t) = \alpha_n\bigl\lfloor\frac{t}{\alpha_n\sigma_n^2}\bigr\rfloor$ is verified for $\sigma_n \to 0$ and for any $\delta > 0$ by the estimate
\[
    \inf_{0<t\leq T} \bigl(A_n(t+\delta) - A_n(t)\bigr) =   \inf_{0<t\leq T}\biggl( \alpha_n\biggl\lfloor\frac{t + \delta}{\alpha_n\sigma_n^2}\biggr\rfloor - \alpha_n\biggl\lfloor\frac{t}{\alpha_n\sigma_n^2}\biggr\rfloor \biggr)\geq \alpha_n\biggl\lfloor\frac{\delta}{\alpha_n\sigma_n^2}\biggr\rfloor \to \infty \,.
\]
Likewise, we verify the assumption~\eqref{ass:An:asympt:cont} by using at the same time $\alpha_n \to 0$ to deduce 
\[
     \sup_{t\geq 0}\bigl( A_n(t) - A_n(t-)\bigr)  = \alpha_n \to 0 \,.
\]
\noindent%
\emph{Verification of Assumption~\ref{ass:Katzenberg:Zn}.} Assumption~\ref{ass:Katzenberg:Zn} for the stopped process $Z_n^{\mu_n(K)}$ follows immediately from~\eqref{lem:props-Zn:Ass39}.

\medskip\noindent
\emph{Verification of Assumption~\ref{assum:4.2}.}
We use the martingale decomposition of $Z_n = M_n + F_n$, with martingale part 
\begin{align*}
M_n(t) &:= Z_n(t) - F_n(t) = \sum_{s\leq t} \bra[\big]{\Delta Z_n(s) - \Delta F_n(s)}\\
&= \sum_{k\leq \frac{t}{\alpha_n\sigma_n^2}} 
\pra[\Big]{\Delta \sfZ_n(W_n(\alpha_n\sigma_n^2 k,\eta_k^n))
+ \Delta \sfF_n(W_n(\alpha_n\sigma_n^2 k))}.
\end{align*}
The expected quadratic variation of $M_n^{\mu_n(K)}$ is estimated by
\begin{align}
	\label{eq:conv:th:main:Mn}
\Expectation [M_n](t\wedge \mu_n(K)) &= \Expectation[Z_n-F_n](t\wedge \mu_n(K)) 
\leq 2\Expectation\pra[\big]{[Z_n] + [F_n]}(t\wedge \mu_n(K))\\
&= 2\;\Expectation \sum_{k\leq \frac{t\wedge \mu_n(K)}{\alpha_n\sigma_n^2}} \pra[\Big]{\Delta \sfZ_n(W_n(\alpha_n\sigma_n^2 k,\eta_k^n))^2
+ \Delta \sfF_n(W_n(\alpha_n\sigma_n^2 k))^2} \notag\\
&\leftstackrel{\eqref{lem:props-Zn:quadratic-variation}} \longrightarrow 0 \qquad \text{as }n\to \infty\, . \notag
\end{align}
This proves the uniform integrability of $[M_n](t\wedge \mu_n(K))$ in Assumption~\ref{assum:4.2} (where we can take $\tau^k_n = +\infty$ for all $k$ and $n$). 

We next show condition~\eqref{eq:ass:4.2:VarVanish}. From~\eqref{lem:props-Zn:Fn-Reg} and the regularity of $\Reg$ we have 
\[
\sup_{w\in \hat K}\abs[\big]{\Delta \sfF_n(w)}
\leq \sup_{w\in \hat K}\abs[\big]{\Delta \sfF_n(w) + \alpha_n\sigma_n^2 \nabla \Reg(w)}+ \alpha_n\sigma_n^2 \sup_{w\in \hat K} \abs[\big]{\nabla \Reg(w)}
\lesssim \alpha_n\sigma_n^2.
\]
Therefore, since $W_n(t) \in \hat K$ as defined in Lemma \ref{lem:props-Zn},
\begin{align}
V_{t+\gamma}(F_n^{\mu_n(K)}) - V_t (F_n^{\mu_n(K)}) 
&= \sum_{\frac{t\wedge \mu_n(K)}{\alpha_n\sigma_n^2} < k \leq \frac{(t+\gamma)\wedge \mu_n(K)}{\alpha_n\sigma_n^2}} \abs[\big]{\Delta \sfF_n(W_n(\alpha_n\sigma_n^2k))} \ \lesssim \ \gamma \, ,
\label{est:Vtgamma-Vt-nondeg}
\end{align}
and the property~\eqref{eq:ass:4.2:VarVanish} follows. Finally, the estimate~\eqref{est:Vtgamma-Vt-nondeg} for $t=0$ implies that $V_t (F_n^{\mu_n(K)})$ is almost surely bounded uniformly in $n$, thus establishing the unform integrability condition on $V_t(F_n^{\mu_n(K)})$ in Assumption~\ref{assum:4.2} (again with $\tau_n^k = +\infty$).

\medskip

Having verified the three assumptions, we apply Theorem~\ref{th:Katzenberg} and conclude that the triplet $(Y_n^{\mu_n(K)}, Z_n^{\mu_n(K)}, \mu_n(K))$ converges along a subsequence to a limit that satisfies the equation~\eqref{eq:th:Katzenberg:equation-for-Y}. We now show that that equation reduces to~\eqref{eqdef:th:main:limiting-dynamics}.

\medskip
\noindent
\emph{Part 2: Characterization of $Z$ and identification of the limiting dynamics.}

To recover the limiting dynamics we study the limit of the sequence of processes $Z_n$. 
Note that Theorem \ref{th:Katzenberg} establishes joint convergence of triplets $(Y_n^{\mu_n}, Z_n^{\mu_n}, \mu_n)$ but does not state any explicit result on the process $W_n$ in~\eqref{eq:def:general:Wn} and its limit. Even though $Z_n$ is defined through $W_n$ but not $Y_n$ in~\eqref{eq:def:general:Yn}, the initial jump does not play a role for the limiting behaviour of $Z_n$. 
Hence, we introduce an intermediate process defined in terms of $Y_n$ by
\begin{equation}\label{eq:def:ZY}
Z^Y_n(t) :=  \sum_{k\leq \frac{t}{\alpha_n\sigma_n^2}}\alpha_n\bigl(\nabla_w \hat{L}(Y_n(\alpha_n\sigma_n^2 k), 0) - \nabla_w \hat{L}(Y_n(\alpha_n\sigma_n^2 k), \eta^n_k)\bigr) \,.
\end{equation}
We then estimate
\begin{align}
	\MoveEqLeft\sup_{0<t\leq T\wedge\mu_n(K)} \left|Z_n(t) + \int_0^{t}  \nabla \Reg(Y_n(s)) \dx s\right| \notag \\
	&\leq \sup_{0<t\leq T\wedge\mu_n(K)} \left|Z_n(t) - Z^Y_n(t)\right| + \sup_{0<t\leq T\wedge\mu_n(K)}\biggl|Z^Y_n(t) + \int_0^{t}  \nabla \Reg(Y_n(s)) \dx s \biggr|. \label{eq:Zsplit}
\end{align}

We first show that the second term vanishes in probability. Note that $Y_n$ is again a pure-jump process, so that 
\[
\int_0^{t}  \nabla \Reg(Y_n(s)) \dx s 
= \alpha_n\sigma_n^2 \sum_{k\leq \frac{t}{\alpha_n\sigma_n^2}} \nabla \Reg(Y_n(\alpha_n\sigma_n^2 k))\, .
\]
Writing the corresponding martingale
\[
M_n^Y(t) = Z_n^Y(t) - F_n^Y(t) = \sum_{k\leq \frac{t}{\alpha_n\sigma_n^2}} 
  \bra[\big]{\Delta \sfZ_n(Y_n(\alpha_n\sigma_n^2 k,\eta_k^n)) + \Delta \sfF_n(Y_n(\alpha_n\sigma_n^2 k))}\,,
\]
we have by the same argument as in~\eqref{eq:conv:th:main:Mn} that 
\[
\Expectation [M_n^{Y,\mu_n(K)}](t) \longrightarrow 0 \qquad \text{as }n\to\infty\,,
\]
and by Doob's inequality that for all $t$
\[
\Expectation\Bigl[\sup_{0\leq s\leq t}\lvert M_n^{Y,\mu_n(K)}\rvert^2 (s) \Bigr]
\leq 4\,\Expectation\bigl[[M_n^{Y,\mu_n(K)}](t)\bigr] \longrightarrow 0.
\]
We then estimate
\begin{align*}
\MoveEqLeft\Expectation \sup_{0<t\leq T\wedge\mu_n(K)}\biggl|Z^Y_n(t) + \int_0^{t}  \nabla \Reg(Y_n(s)) \dx s \biggr|\\
&\leq \Expectation \sup_{0<t\leq T\wedge\mu_n(K)}|M_n(t)|
+ \Expectation \sup_{0<t\leq T\wedge\mu_n(K)}\biggl| F_n^Y(t) + \int_0^{t}  \nabla \Reg(Y_n(s)) \dx s \biggr|\\
&\leq o(1) + \Expectation \sum_{k\leq \frac{T\wedge\mu_n(K)}{\alpha_n\sigma_n^2}}
  \abs[\Big]{\Delta \sfF_n(Y_n(\alpha_n\sigma_n^2 k)) + \alpha_n\sigma_n^2\nabla \Reg(Y_n(\alpha_n\sigma_n^2 k)) }, 
\end{align*}
and the right-hand side vanishes by~\eqref{lem:props-Zn:Fn-Reg}.

For the first term in the splitting~\eqref{eq:Zsplit}, we get
\begin{align}
	\left|Z_n(t) - Z^Y_n(t)\right|
	&= \alpha_n\biggl|\sum_{k\leq \frac{t}{\alpha_n\sigma_n^2}}
	\bigl(\nabla_w \hat{L}(W_n(\alpha_n\sigma_n^2 k), 0) - \nabla_w \hat{L}(W_n(\alpha_n\sigma_n^2 k), \eta^n_k)\bigr) \notag \\
	&\qquad\qquad\qquad\quad - \bigl(\nabla_w \hat{L}(Y_n(\alpha_n\sigma_n^2 k), 0) - \nabla_w \hat{L}(Y_n(\alpha_n\sigma_n^2 k), \eta^n_k)\bigr)  \biggr| \notag \\
	&\leq \alpha_n\,\biggl|\sum_{k\leq \frac{t}{\alpha_n\sigma_n^2}}\Big(\nabla_w \nabla_\eta\hat{L}(W_n(\alpha_n\sigma_n^2 k), 0) 
	- \nabla_w \nabla_\eta\hat{L}(Y_n(\alpha_n\sigma_n^2 k), 0) \Big)\eta^n_k\biggr| \label{eq:est:ZZY}\\
	&\quad + \alpha_n \,\biggl| \sum_{k\leq \frac{t}{\alpha_n\sigma_n^2}} \sum_{i,j=1}^d g_{i,j}(W^n_k, Y^n_k, \eta_k^n) \eta_{k,i}^n \eta_{k,j}^n \biggr| \,,  \label{eq:est:gZZY}
\end{align}
where we use Taylor's theorem to write the remainder term as 
\begin{equation*}
	g_{ij}(W^n_k, Y^n_k, \eta_k^n) := \int_0^{1}(1-r) \partial_{\eta_i\eta_j}\bigl(\nabla_w \hat{L}(W_n(\alpha_n\sigma_n^2 k), r\eta_k^n) - \nabla_w \hat{L}(Y_n(\alpha_n\sigma_n^2 k), r\eta_k^n) \bigr)   \dx r \,.
\end{equation*}
Recall that the definition of $Y_n$ in~\eqref{eq:def:general:Yn} implies
\[
W_n(t)  -Y_n(t)  = \phi\left(W_n(0), A_n(t)\right)-\Phi\bigl(W_n(0)\bigr) \,.
\]
Then for the first term~\eqref{eq:est:ZZY} we have
\begin{align*}
	\MoveEqLeft  \alpha_n\sum_{k\leq \frac{t}{\alpha_n\sigma_n^2}}
	\Bigl(\nabla_w \nabla_\eta\hat{L}(W_n(\alpha_n\sigma_n^2 k), 0) 
	- \nabla_w \nabla_\eta\hat{L}(Y_n(\alpha_n\sigma_n^2 k), 0) \Bigr)\eta^n_k \\
	&= \sqrt{\alpha_n}\int_0^t \Bigl(\nabla_w \nabla_\eta\hat{L}(W_n(s), 0) 
	- \nabla_w \nabla_\eta\hat{L}(Y_n(s), 0) \Bigr)\dx h_n.
\end{align*}
Here $h_n$ is defined as
\[
h_n(t) := \sqrt{\alpha_n}\sum_{k\leq \frac{t}{\alpha_n\sigma_n^2}} \eta_k^n
\]
and $h_n$ converges to a standard Brownian motion by the Functional Central Limit Theorem~\cite[Th.~8.2]{Billingsley1968}. 
Since $ \nabla_w \nabla_\eta\hat{L}(w, 0)$ is bounded on $\hat K$, the characterization \cite[Prop.~4.4]{Katzenberger91} then implies that  the  term~\eqref{eq:est:ZZY} converges to zero in probability.

For the second term~\eqref{eq:est:gZZY} we use Assumption~\ref{ass:hatL} and obtain
\begin{align*}
	\bigl| g_{ij}(W^n_k, Y^n_k, \eta_k^n) \bigr| &= \biggl|\int_0^{1} (1-r)\partial_{\eta_i \eta_j}\bigl(\nabla_w \hat{L}(W_n(\alpha_n\sigma_n^2 k), r\eta_k^n)  - \nabla_w \hat{L}(Y_n(\alpha_n\sigma_n^2 k), r\eta_k^n) \bigr)\dx r\biggr| \\
	&\leq \sup_{0\leq r\leq 1} \Bigl|\nabla_{\eta}^2\bigl(\nabla_w \hat{L}(W_n(\alpha_n\sigma_n^2 k), r\eta_k^n) - \nabla_w \hat{L}(Y_n(\alpha_n\sigma_n^2 k), r\eta_k^n) \bigr)\Bigr| \\
	&\lesssim \bigl(1+ \lvert\eta^n_k\rvert^p\bigr)\bigl\lvert W_n(\alpha_n\sigma_n^2 k) - Y_n(\alpha_n\sigma_n^2 k)\bigr\rvert .
\end{align*}
By applying this bound to the residual in estimate~\eqref{eq:est:gZZY}, we obtain
\begin{align*}
	\alpha_n \biggl| \sum_{k\leq \frac{t}{\alpha_n\sigma_n^2}}  \sum_{i,j=1}^d g_{ij}(W^n_k, Y^n_k, \eta_k^n) \eta_{k,i}^n \eta_{k,j}^n \biggr| 
	&\lesssim \alpha_n  \ \sum_{\mathclap{k\leq \frac{t}{\alpha_n\sigma_n^2}}}\  \bigl(\lvert\eta^n_k\rvert^2+ \lvert\eta^n_k\rvert^{p+2}\bigr) \bigl\lvert W_n(\alpha_n\sigma_n^2 k) - Y_n(\alpha_n\sigma_n^2 k)\bigr\rvert. 
\end{align*}

By an application of Proposition~\ref{prop:exp:convergence}, we have the exponential convergence
\begin{equation}
	\label{ineq:exponential-convergence-W-Y}
|W_n(t)  -Y_n(t)| = \bigl| \phi\bigl(W_n(0), A_n(t)\bigr) - \Phi\bigl(W_n(0)\bigr) \bigr| \lesssim e^{-\beta A_n(t)} \,.
\end{equation}

Hence, we conclude
\begin{align*}
	\MoveEqLeft \Expectation\biggl[ \alpha_n \  \sum_{\mathclap{k\leq \frac{t}{\alpha_n\sigma_n^2}}}\  \bigl(\lvert\eta^n_k\vert^2+ \lvert\eta^n_k\rvert^{p+2}\bigr)e^{-\beta A_n(t)}\biggr]
	=
	\alpha_n \ \sum_{\mathclap{k\leq \frac{t}{\alpha_n\sigma_n^2}}}\ 
	e^{-\beta A_n(t)}\bbE\bigl[\vert\eta^n_k\vert^2+ \vert\eta^n_k\vert^{p+2}\bigr]  \\
	&\leq \bigl(\sigma_n^2 +\M_{p+2}(\sigma_n)\bigr) \int_0^t e^{-\beta A_n(t)}\dx A_n(t)
	\longrightarrow 0 \qquad \text{as }n\to\infty \,,
\end{align*}
because $M_{p+2}(\sigma_n) = O(\sigma_n^2)$ by Assumption~\ref{ass:hatL} and $\sigma_n \to 0$. So we conclude
\[
\eqref{eq:est:gZZY}\;
\lesssim\;
\alpha_n\ \sum_{\mathclap{k\leq \frac{t}{\alpha_n\sigma_n^2}}}\ \bigl(\lvert\eta^n_k\rvert^2+ \lvert\eta^n_k\rvert^{p+2}\bigr) \bigl\lvert W_n(\alpha_n\sigma_n^2 k) - Y_n(\alpha_n\sigma_n^2 k)\bigr\rvert \Rightarrow 0.
\]
Hence, we have shown that 
\[
Z_n\Rightarrow Z, \qquad\text{where}\qquad Z(t) := -\int_0^t \nabla\Reg(Y(s))\dx s.
\]
Note that the limit $Z$ is a continuous process of finite variation and therefore $[Z,Z]=0$. Combining this with~\eqref{eq:lem:derivatives-of-Phi:first-derivative} we find that~\eqref{eq:th:Katzenberg:equation-for-Y} reduces to~\eqref{eqdef:th:main:limiting-dynamics}. Since the right-hand side of~\eqref{eqdef:th:main:limiting-dynamics} is continuous in~$t$ for any~$Y$, the process~$Y$ is a continuous function of~$t$. By the definition~\eqref{eq:def:general:Yn} we have  $Y_n(0) = \Phi(W_n(0))\Rightarrow \Phi(W_0)$, implying that $Y(0) = \Phi(W_0)$.

This concludes the proof of Theorem~\ref{th:main}.
\end{proof}
We still owe the reader the proof of Lemma~\ref{lem:props-Zn}.
\begin{proof}[Proof of Lemma~\ref{lem:props-Zn}]
To begin with, note that the stopping time $\mu_n(K)$ restricts $Y_n(t\wedge \mu_n(K))$ to the compact set $K$, which combined with Proposition~\ref{prop:exp:convergence} guarantees that $W_n(t\wedge \mu_n(K))$ is restricted to the larger but still compact set $\hat K$. We first prove~\eqref{lem:props-Zn:Ass39}. 
We use Taylor expansion in $\eta$ to obtain 
\begin{align}
\MoveEqLeft \nabla_w \hat{L}(w, 0) - \nabla_w \hat{L}(w, \eta)  = -\nabla_{w}\partial_\eta \hat{L}(w, 0)[\eta] 
- \frac12  \nabla_{w}\partial^2_{\eta}\hat{L}(w, 0)[\eta, \eta]
- R_n(w,\eta) \,,
 \label{eq:taylor1} 
\end{align}
where the error term $R_n$ is given by 
\begin{equation}\label{eq:def:TaylorRemainder}
	R_n(w,\eta):= \int_0^1(1-\xi) \Bigl(\partial^2_{\eta} \nabla_{w}\hat{L}(w,\xi\eta)[\eta, \eta]- \partial^2_\eta\nabla_{w}\hat{L}(w, 0)[\eta, \eta]\Bigr) \dx\xi \,.
\end{equation}
Using the compactness of $\hat K$ and Assumption~\ref{ass:hatL}, we then bound
\begin{align}
\!\!\sup_{0<t\leq T \wedge \mu_n(K)}\!\! |\Delta Z_n(t)|  &\leq \sup_{k\leq \frac{T \wedge \mu_n(K)}{\alpha_n\sigma_n^2}} \alpha_n|\nabla_{w}\partial_\eta\hat{L}(W_n(\alpha_n \sigma_n^2k), 0)[\eta_k^n] |  \notag\\
&\quad +\sup_{k\leq \frac{T \wedge \mu_n(K)}{\alpha_n\sigma_n^2}} \alpha_n\abs*{\frac12  \nabla_{w}\partial^2_{\eta}\hat{L}(W_n(\alpha_n \sigma_n^2k), 0)[\eta_k^n, \eta_k^n] }\notag\\
&\quad + \sup_{k\leq \frac{T \wedge \mu_n(K)}{\alpha_n\sigma_n^2}} \alpha_n\bigl|R_n(W_n(\alpha_n \sigma_n^2k),\eta_k^n) \bigr| \notag\\
&\lesssim \!\!\sup_{k\leq \frac{T \wedge \mu_n(K)}{\alpha_n\sigma_n^2}}\!\! \alpha_n |\eta_k^n | 
+ \!\!\sup_{k\leq \frac{T \wedge \mu_n(K)}{\alpha_n\sigma_n^2}}\!\! \alpha_n |\eta_k^n |^2
+ \!\!\sup_{k\leq \frac{T \wedge \mu_n(K)}{\alpha_n\sigma_n^2}}\!\! \alpha_n \bigl(|\eta_k^n|^3 + |\eta_k^n|^{p+2}\bigr) \,.
\label{est:lem-props-Zn-|Zn|}
\end{align}
Thus, by the assumption~\eqref{ass:speed}, we conclude with the estimate
\[
\sup_{0<t\leq T \wedge \mu_n(K)} |\Delta Z_n(t)| \lesssim \alpha_n \biggl\{ \;1+  \sup_{k\leq \frac{T \wedge \mu_n(K)}{\alpha_n\sigma_n^2}} | \eta_k^n|^{p+2} \biggr\} \Rightarrow 0.
\]

\medskip

We next prove~\eqref{lem:props-Zn:Fn-Reg}. First note that for any $w\in \hat K$
\begin{equation}
	\label{est:l:props-Zn:C2Mp+2}
\Expectation_\eta \abs[\big]{R_n(w,\eta)}
\stackrel{\eqref{ass:compat-L-Lhat-difference-in-etas}}\lesssim 
\Expectation_\eta  C_2 \int_0^1(1-\xi) \bra[\big]{|\eta|^3 + |\eta|^{p+2}}\dx \xi  =\frac12 C_2 \bra[\big]{\M_3 + \M_{p+2}}
\stackrel{\eqref{eq:ass:moment}}= o(\sigma_n^2) \qquad \text{as }n\to \infty\,.
\end{equation}
We then write, using the independence of the coordinates of $\eta$, 
\begin{align*}
\abs[\big]{\Delta \sfF_n(w) + \alpha_n\sigma_n^2 \nabla \Reg(w)}
&= \alpha_n \abs*{\Expectation_\eta \pra[\big]{\nabla_w \hat{L}(w, 0) - \nabla_w \hat{L}(w, \eta)} + \frac12 \sigma_n^2\nabla_w \Delta_\eta \hat L(w,0)}\\
&= \alpha_n \bigg|\underbrace{- \frac12\Expectation_\eta \pra*{  \nabla_{w}\partial^2_{\eta}\hat{L}(w, 0)[\eta, \eta]}
 + \frac12 \sigma_n^2\nabla_w \Delta_\eta \hat L(w,0)}_{=0} - 
 \Expectation_\eta R_n(w,\eta)\bigg|\\
 &= o(\alpha_n\sigma_n^2) \qquad \text{as }n\to \infty. 
\end{align*}
Since the estimates are independent of $w\in \hat K$, this proves~\eqref{lem:props-Zn:Fn-Reg}.

We finally show~\eqref{lem:props-Zn:quadratic-variation}. Since $\nabla \Reg$ is continuous on $\hat K$, it is bounded, and therefore 
\begin{align*}
\abs[\big]{\Delta \sfF_n(w)} ^2
&\leq 2 \abs[\big]{\Delta \sfF_n(w) + \alpha_n\sigma_n^2 \nabla\Reg(w)} ^2
+ 2 \alpha_n^2\sigma_n^4 \abs[\big]{\nabla\Reg(w)} ^2
\stackrel{\eqref{lem:props-Zn:Fn-Reg}}= o(\alpha_n\sigma_n^2)\,.
\end{align*}
For $\Delta \sfZ_n$ we use a similar calculation as in~\eqref{est:lem-props-Zn-|Zn|} above and estimate
\[
\Expectation \sup_{w\in K}\abs[\big]{\Delta Z_n(w,\eta)}^2
\lesssim \alpha_n^2 \Expectation \pra[\Big]{|\eta|^2 + |\eta|^4 + |\eta|^{2p+4}} = \alpha_n^2 (\M_2 + \M_4 + \M_{2p+4})
= o(\alpha_n\sigma_n^2) \quad\text{ as }n\to\infty\,.
\]
This concludes the proof of Lemma~\ref{lem:props-Zn}.
\end{proof}

\subsection{Degenerate case}
\label{ss:degenerate-case}
As discussed in the Introduction, for some forms of noisy gradient descent the limiting dynamics of $W_n$ characterized by Theorem~\ref{th:main} is trivial. Examples of this are label noise, minibatching, and stochastic gradient Langevin  descent; we discuss these in more detail in Section~\ref{sec:examples}. In this section we present the second main convergence result, in the more strongly accelerated regime, which applies to functions $\hat L$ for which the evolution~\eqref{eq:constr-GF} is trivial. 

Label noise does not only have a trivial limiting dynamics according to Theorem~\ref{th:main},  but it happens to have an even more specific structure which allows to prove convergence under milder assumptions. Namely, label noise has a loss function that is polynomial in $\eta$, and this  both significantly simplifies the calculations and does not require the variance of the noise to vanish. It is enough to consider the limit of infinite-small step size $\alpha_n\to 0$ with  $\sigma_n$ either constant or converging to some $\sigma_0\geq 0$. The same structure is shared by a number of other degenerate functions $\hat L$ that we study in Section~\ref{sec:examples}.

To characterize this structure we make a more specific assumption on the noise-injected loss function $\hat{L}(w, \eta)$ that augments Definition~\ref{def:noiseLoss}.
\begin{definition}[Degenerate quadratic noise-injected loss]\label{def:deg:noiseLoss}
For a given \emph{loss function} $L\in C^3\bigr(\R^m,[0,\infty)\bigl)$, a noise-injected loss function $\hat L \in C^3\bigl(\R^{m+d},[0,\infty)\bigr)$ in the sense of Definition~\ref{def:noiseLoss} is called \emph{degenerate quadratic} provided that
\begin{equation}
\label{eq:Ldeg}
\hat{L}(w, \eta) = L(w) + f(w)\cdot\eta +  \tfrac12  H(w) : (\eta \otimes \eta)  + g(\eta)
\end{equation}
for some $f\in C^3(\R^m,\R^d)$ and a field of quadratic forms $H \in C^3(\R^m, \R^{d\times d}_{\text{sym}})$ with zero diagonal entries $H_{ii}(w) =0$ for all $w\in \R^m$, where $H: (\eta\otimes \eta ) = \sum_{i,j} H_{ij} \eta_{i} \eta_{j}$.

Moreover, $g\in C^3(\R^d,\R)$ with $g(0)=0$.
\end{definition}
For any $\hat L$ of the form given in Definition~\ref{def:deg:noiseLoss} we obtain $\nabla_w\Delta_\eta \hat{L}(w, \eta)=0$, indeed leading to trivial dynamics on the time scale $\alpha_n \sigma_n^2$.
We define the  corresponding sequence of integrators $\hat{A}_n(t) = \alpha_n\lfloor\frac{t}{\alpha_n^2\sigma_n^2}\rfloor$ (note the difference with~\eqref{eqdef:An} in the power of $\alpha_n$ in the denominator).
The corresponding accelerated process of noisy gradient descent~\eqref{eq:system:w} are then of the form
\begin{subequations}\label{eq:subequs:hatW}
\begin{equation}\label{eq:def:hatW}
\hat W_n(t) = w^n_{\bigl\lfloor \frac t{\alpha_n^2\sigma_n^2}\bigr\rfloor} \,.
\end{equation}
Due to the specific structure of the noise-injected loss from Definition~\ref{def:deg:noiseLoss}, we can bring the process into the form~\eqref{eq:ds-per} (compare also with~\eqref{eq:def:W} in the first case), and obtain a simplified expression for the noise process $\tilde{Z}_n$ given by
\begin{align}
d\hat{W}_n &= -\nabla L(\hat{W}_n)\dx\hat{A}_n + \dx\hat{Z}_n, \\
\hat{Z}_n(t) &= \sum_{s\leq t} \Delta \hat{Z}_n(s) =  \alpha_n \ \sum_{\mathclap{k\leq \frac{t}{\alpha_n^2 \sigma_n^2}}}\  \bigl( \nabla_w f(\hat W_n(\alpha_n^2\sigma_n^2 k))\cdot \eta_k^n + \tfrac12 \nabla_wH(\hat W_n(\alpha_n^2\sigma_n^2 k))[\eta_k^n, \eta_k^n] \bigr) \, , \\
\eta^n_{k,i} &\sim  \rho(\sigma_0).
\end{align}
\end{subequations}
We formulate the analogue of Theorem \ref{th:main} for the rescaled process defined by~\eqref{eq:subequs:hatW}.
\begin{theorem}[Main convergence theorem in the degenerate case]
\label{th:main:deg2}
Assume that $L, \hat{L}$ are $\calC^3$, $\hat{L}$ is of the form~\eqref{eq:Ldeg}, and $L$ satisfies Assumption~\ref{ass:LossMfd}. Let $\hat W_n(0) \Rightarrow \hat W(0) \in U_{\Gamma}$ for some locally attractive neighbourhood $U_\Gamma$ of $\Gamma$.
Fix a sequence $(\sigma_n>0)_{n\in\bbN}$ such that $\sigma_n \to \sigma_0\geq0$ as $n\to 0$. 
Let $(\rho_n\in \ProbMeas(\R^d))_{n\in \bbN}$ satisfy $\Expectation_{\eta\sim \rho_n}|\eta|^2 = \sigma_n^2$, and let $\alpha_n \to 0$ such that 
\begin{equation}\label{ass:speed:deg1}
	\alpha_n \sup_{k \leq 1/\alpha_n^2\sigma_n^2}|\eta^n_k|^2 \Rightarrow 0 \qquad\text{for i.i.d. } \eta_{k,i}^n\sim \rho_n \text{ as } n\to \infty \,.
\end{equation}
Let $\hat{W}_n$ be a solution of~\eqref{eq:subequs:hatW} and define
\[
\hat{Y}_n(t)=\hat{W}_n(t)-\phi\bigl(\hat{W}_n(0), \hat{A}_n(t)\bigr)+\Phi\bigl(\hat{W}_n(0)\bigr) \,.
\]
For compact $K \subset U_{\Gamma}$, define the stopping time 
\[
\hat{\mu}_n(K)=\inf \left\{t \geq 0 \mid \hat{Y}_n(t-) \notin K^{\circ} or \ \hat{Y}_n(t) \notin K^{\circ} \right\} .
\]
Then for any $K \subset U_{\Gamma}$ the sequence $(\hat{Y}^{\hat{\mu}_n(K)}_n, \hat{\mu}_n)$ is relatively compact in Skorokhod topology. 
Moreover, any limit point $(\hat{Y}, \hat{\mu})$ of $(\hat{Y}_n, \hat{\mu}_n)$ satisfies $Y(t) \in \Gamma$ a.s. for any $t$ and solves
\begin{equation}
	\label{eq:th:main2:constrained-SDE}
\hat{Y}(t)=  \hat{Y}(0) + \int_0^{t \wedge \hat{\mu}} \!\!\!\!\! P_{T_{\hat Y(s)}\Gamma}\bra*{ \nabla_w f(\hat Y(s)) \cdot \dx b_s +
	\sigma_0 \nabla_w H(\hat Y(s)) : \dx B_s }
+\frac{1}{2}  \int_0^{t \wedge \hat{\mu}} \!\!\!\! 
\partial^2 \Phi(\hat{Y}) [\varSigma(\hat{Y}(s))] \dx s \,,
\end{equation}
where $(b_s)_{s\geq 0}$ is a $d$-dimensional Brownian motion and $B_{s,k,\ell}$ for $1\leq k< \ell\leq d$ are Brownian motions with $B_{s,\ell,k}:=B_{s,k,\ell}$ and
\[
\varSigma(w)_{ij} := \partial_{w_i} f(w) \cdot \partial_{w_j} f(w) + \sigma_0^2 \partial_{w_i} H(w) : \partial_{w_j} H(w)  \,, \qquad\text{for } 1\leq i,j,\leq d \,.
\]
Here $P_{T_y\Gamma}$ is the orthogonal projection onto the tangent space $T_y\Gamma$ of $\Gamma$ at $y$. 
In addition, 
\begin{equation}
	\label{ineq:th:main2:hat-mu-and-tau}
	\hat \mu \geq  \inf\set[\big]{t\geq 0\mid \hat Y(t)\notin  K^\circ} \qquad \text{almost surely.}	
\end{equation}
\end{theorem}

\begin{remark}[Convergence of the whole sequence]
Similarly to the remark after Theorem~\ref{th:main}, whenever the weak solution of~\eqref{eq:th:main2:constrained-SDE} is unique, the usual argument implies that the whole sequence $\hat Y_n^{\mu_n(K)}$ converges up to the time $\tau$, 
\[
	\tau := \inf\set[\big]{t\geq 0\mid \hat Y(t)\notin 
	K^\circ}.	
\]
A sufficient condition for such weak uniqueness is local Lipschitz continuity of the drift and mobility functions in~\eqref{eq:th:main2:constrained-SDE} (see e.g.~\cite[Th.~1.1.10]{Hsu02}). 
\end{remark}

\begin{remark}[Implicit It\=o correction terms]
Since~\eqref{eq:th:main2:constrained-SDE} is a stochastic differential equation in It\=o form, the noise term should be accompanied by a corresponding drift term in order to preserve the condition $\hat Y(t)\in \Gamma$. Li, Wang, and Arora~\cite{LiWangArora22} discuss this aspect in some detail, and show how the final integral contains, as part of this integral, the necessary `It\=o correction terms'.
\end{remark}

\begin{remark}[Comparison with~{\cite{LiWangArora22}}]
At a high level of abstraction, Theorem~\ref{th:main:deg2} above is similar to~\cite[Th.~4.6]{LiWangArora22}. However, in Theorem~\ref{th:main:deg2} the conditions on $\hat L$ and on the noise are much more general, opening the door to e.g.\  the applications to minibatching and stochastic gradient Langevin descent in Section~\ref{sec:examples}.
\end{remark}

\begin{proof}
As in the proof for the general case in Theorem~\ref{th:main}, we begin by checking the assumptions of Theorem~\ref{th:Katzenberg}.

\medskip
\noindent
\emph{Verification of Assumption~\ref{ass:Katzenberger:noise}.}
By definition of $\hat{A}_n$ we get for $\alpha_n \to 0$ and bounded $\sigma_n$ for any $\delta > 0$
\[
\inf_{0<t\leq T} (\hat{A}_n(t+\delta) - \hat{A}_n(t)) =   \inf_{0<t\leq T}\biggl( \alpha_n\biggl\lfloor\frac{t + \delta}{\alpha_n^2\sigma_n^2}\biggr\rfloor - \alpha_n\biggl\lfloor\frac{t}{\alpha_n^2\sigma_n^2}\biggr\rfloor \biggr)\geq \alpha_n\biggl\lfloor\frac{\delta}{\alpha_n^2\sigma_n^2}\biggr\rfloor \to \infty.
\]
At the same time\[
\sup_{t\geq 0} \hat{A}_n(t) - \hat{A}_n(t-)  = \alpha_n \to 0 \qquad\text{for $\alpha_n \to 0$}\,.
\]
\noindent%
\emph{Verification of Assumption~\ref{ass:Katzenberg:Zn}.} 
Since $\hat{L}$ is quadratic in $\eta$, we get from assumption~\eqref{ass:speed:deg1} the estimate
\[
\sup_{0<t\leq T \wedge \hat\mu_n(K)} |\Delta \hat{Z}_n| 
=\alpha_n \bra[\bigg]{ 1+ \sup_{0<t\leq T \wedge \hat\mu_n(K)}|\eta_k^n |^{2}}  
\Rightarrow 0.
\] 

\medskip\noindent
\emph{Verification of Assumption~\ref{assum:4.2}.}
Note that due to the structure of $\hat{L}$ the noise process is a martingale because $\nabla_{w}H_{i, i} = 0$ and has a simplified form:
\begin{align*}
\Delta \hat Z_n = \alpha_n\nabla_{w}f(\hat{W}^{\hat{\mu}_n(K)}_n(\alpha_n^2 \sigma_n^2k))\eta_k^n 
+ \tfrac12 \alpha_n \nabla_{w}H\bigl(\hat{W}^{\hat{\mu}_n(K)}_n(\alpha_n^2 \sigma_n^2k)\bigr)[\eta_k^n, \eta_k^n]
\end{align*}
We verify that $\bigl[\hat{Z}^{\hat{\mu}_n(K)}_n\bigr](t)$ is uniformly integrable for any $t\in \bbR^+$ by estimating
\begin{align*}
\bigl[\hat{Z}^{\hat{\mu}_n(K)}_n\bigr](t) \leq  2\alpha_n^2 \sum_{k\leq \frac{t}{\alpha_n^2\sigma_n^2}}\sup_{w\in K}|\nabla_{w}f(w))|^2\sum_{i=1}^d (\eta_{k,i}^n)^2 
+ \alpha_n^2\sum_{k\leq \frac{t}{\alpha_n^2\sigma_n^2}}\sup_{w\in K}|\nabla_{w}H(w)|^2\sum_{\substack{i,j=1 \\ i\neq j}}^d (\eta_{k, i}^n\eta_{k, j}^n)^2.
\end{align*}
By assumption $\Var(\eta_{k, i}^n)= \sigma_n^2$ and as $\Expectation\eta_{k, i}^n = 0$ also $\Var(\eta_{k, i}^n\eta_{k, j}^n)=\Var(\eta_{k, i}^n) \Var(\eta_{k, j}^n)= \sigma_n^4$, so by the Functional Central Limit Theorem (see e.g.~\cite[Th.~8.2]{Billingsley1968}) we have for $i=1,\dots,d$ and $1\leq \ell\ne m\leq d$ the convergence
\begin{align}
\label{eq:hn}
 \sum_{k\leq \frac{t}{\alpha_n^2\sigma_n^2}}\alpha_n\eta^n_{k,i} =: b_{t,i}^n \Rightarrow b_{t,i}, 
 \quad\text{and}\quad
\sum_{k\leq \frac{t}{\alpha_n^2\sigma_n^2}}\alpha_n\eta^n_{k,\ell}\eta^n_{k,m} =: B_{t,\ell,m}^n  \Rightarrow \sigma_0 B_{t,\ell,m}, \quad  \text{ as } n\to \infty\,,
\end{align}
where $b_{t,i}$ for $i=1,\dots,d$ and $B_{t,\ell,m}$ for $1\leq \ell\ne m \leq d$ are standard Brownian motions.
We have the symmetry $B_{t,\ell,m}=B_{t,m,\ell}$. 
We note that $b_{t,i}^n$ for $1\leq i\leq d$ and $B_{t,\ell,m}^n$ for $1\leq \ell < m \leq d$ are uncorrelated for any $n$. Hence, so are the limits and as normal random variables, we find that  $b_{t,i}$ for $i=1,\dots,d$ and $B_{t,\ell,m}$ for $1\leq \ell< m\leq d$ are  independent standard Brownian motions.
We then combine \cite[Proposition 4.3]{Katzenberger91} and \cite[Proposition 4.4]{Katzenberger91} and conclude that $\hat{Z}^{\hat{\mu}_n(K)}_n$ satisfies Assumption \ref{assum:4.2}. 

We then apply Theorem~\ref{th:Katzenberg} and obtain the convergence of $\hat Y_n$ to a limiting process $\hat Y$ that follows the limiting evolution~\eqref{eq:th:Katzenberg:equation-for-Y}.

\medskip
\noindent
\emph{Identification of the limiting dynamics.} 
For studying the limiting dynamic of~$\hat{Z}_n$, we introduce the intermediate process
\[
\hat{Z}^Y_n(t)  =   \alpha_n \sum_{k\leq \frac{t}{\alpha_n^2\sigma_n^2}} \bigl(\nabla_w f(\hat Y_n(s))\cdot \eta_k^n + \tfrac12\nabla_wH(\hat Y_n(s))[\eta_k^n,\eta_k^n] \bigr) \,,
\]
with which we obtain the estimate
\begin{align}
	\MoveEqLeft\sup_{0<t\leq T} \biggl|\hat Z_n(t) - \int_0^{t} \nabla_w f(\hat Y(s)) \cdot \dx b_s -  \sigma_0\int_0^{t} \nabla_w H(\hat Y(s)) : \dx B_s\biggr| \label{eq:est:d} \\
	&\leq \sup_{0<t\leq T} \Bigl|\hat{Z}_n(t) - \hat{Z}^Y_n(t)\Bigr| + 
	\sup_{0<t\leq T}\biggl|\hat{Z}^Y_n(t) - \int_0^{t} \nabla_w f(\hat Y(s))\cdot \dx b_s -  \sigma_0\int_0^{t} \nabla_w H(\hat Y(s)): \dx B_s \biggr| \notag ,
\end{align}
where $(b_{s,i})_{s\geq 0}$, $(B_{s,k,l})$ for $i=1,\dots ,d$ and $1\leq k< l \leq d$ are independent 
Brownian motions and we set $B_{s,l,k}:= B_{s,k,l}$.
For the first term analogously to the main theorem we get
\begin{align}
\Bigl|\hat{Z}_n(t) - \hat{Z}^Y_n(t)\Bigr| &\leq \alpha_n\Biggl|\sum_{k\leq \frac{t}{\alpha_n^2\sigma_n^2}}\Bigl(\nabla_w f(\hat W_n(s)) -
	 \nabla_w f(\hat Y_n(s))\Bigr) \cdot \eta^n_k\Biggr| \label{eq:est:ZZYd1}\\
	&\quad +\frac12\alpha_n \Biggl|\sum_{k\leq \frac{t}{\alpha_n^2\sigma_n^2}}
	\bigl(\nabla_wH(\hat W_n(s)) - \nabla_wH(\hat Y_n(s)) \bigr)[\eta^n_{k}, \eta_{k}^n] \Biggr|. \label{eq:est:ZZYd2} 
	\end{align}
Recall from Proposition~\ref{prop:exp:convergence} that
\[
\hat{W}_n(t)  -\hat{Y}_n(t)  = \phi\bigl(\hat{W}_n(0), A_n(t)\bigr)-\Phi\bigl(\hat{W}_n(0)\bigr) = u(t)e^{-\beta A_n(t)}
\]
for some $u \in C_b([0,T])$. Note that both terms are martingales, so $\Expectation\bigl((X^n_T)^2\bigr) \to 0$ implies $\sup_{t\leq T}|X^n_t| \Rightarrow 0$. Then by compactness of $K$ and regularity of $\hat{L}$ similarly to the proof of Theorem~\ref{th:main}, we conclude that 
\begin{align*}
\MoveEqLeft \Expectation\biggl[\alpha_n\biggl|\sum_{k\leq \frac{t}{\alpha_n^2\sigma_n^2}}\Bigl(\nabla_w f(\hat W_n(s)) -
	 \nabla_w f(\hat Y_n(s))\Bigr)\cdot\eta^n_k\biggr|\biggr]^2 \\
	&\lesssim \alpha_n^2\sum_{k\leq \frac{t}{\sigma_n^2}}
	 e^{-2\beta A_n(\alpha_n^2\sigma_n^2k)}\Expectation\lvert\eta_k^n\rvert^2 = \alpha_n\sigma_n^2\int_0^t e^{-2\beta A_n(s)}\dx A_n(s) \\
	&= \frac{\alpha_n\sigma_n^2}{\beta}(1- e^{-\beta A_n(t)}) \to 0 \qquad \text{as }n\to\infty \,.
\end{align*}
Hence, we obtain for term~\eqref{eq:est:ZZYd1} the convergence
\[
\sup_{0\leq t\leq T}\alpha_n\biggl|\sum_{k\leq \frac{t}{\alpha_n^2\sigma_n^2}}\Big(\nabla_w f(\hat W_n(s)) -
	 \nabla_w f(\hat Y_n(s))\Big)\eta^n_k\biggr| \Rightarrow 0 \,.
\]
Analogous bound holds for the second term \eqref{eq:est:ZZYd2}.

We split the second component in \eqref{eq:est:d} into the two contributions
\begin{align}
	\MoveEqLeft  \sup_{0<t\leq T}\biggl|\hat{Z}^Y_n(t) - \int_0^{t} \nabla_w f(\hat Y(s))\cdot \dx b_s -  \sigma_0\int_0^{t} \nabla_w H(\hat Y(s)): \dx B_s \biggr| \label{eq:est:deg} \\
	&\leq \sup_{0<t\leq T}\Biggl|  \alpha_n \sum_{k\leq \frac{t}{\alpha_n^2\sigma_n^2}}  \nabla_w f(\hat Y_n(s))\cdot \eta_k^n - \int_0^{t} \nabla_w f(\hat Y(s))\cdot \dx b_s  \Biggr|\notag \\
	&\quad + \sup_{0<t\leq T}\Biggl| \alpha_n \sum_{k\leq \frac{t}{\alpha_n^2\sigma_n^2}}   \nabla_w H(\hat Y_n(s))[\eta_k^n,\eta_k^n] - \sigma_0\int_0^{t} \nabla_w H(\hat Y(s)): \dx B_s \Biggr|.\notag 
\end{align}
By using the processes $(b_t^n)_{t\geq 0}$ and $B_t^n)_{t\geq 0}$ from~\eqref{eq:hn}, we decompose the error terms further into
\begin{align}
	\MoveEqLeft \sup_{0<t\leq T}\biggl|\hat{Z}^Y_n(t) - \int_0^{t} \nabla_w f(\hat Y(s))\cdot \dx b_s -  \sigma_0\int_0^{t} \nabla_w H(\hat Y(s)): \dx B_s \biggr| \\
	&\leq \sup_{0<t\leq T}\biggl|  \int_0^{t} \nabla_w f(\hat Y(s))\dx b^n_s - \int_0^{t} \nabla_w f(\hat Y(s)) \cdot \dx b_s \biggr| \label{eq:est:deg:1} \\
	&\qquad+ \sup_{0<t\leq T}\biggl|  \int_0^{t} \nabla_w f(\hat Y_n(s))\cdot \dx b^n_s -  \int_0^{t} \nabla_w f(\hat Y(s))\cdot \dx b^n_s  \biggr|\label{eq:est:deg:2}  \\
	&\quad + \sup_{0<t\leq T}\biggl| \int_0^{t} \nabla_w H(\hat Y(s)) : \dx B^n_s - \sigma_0 \int_0^{t} \nabla_w H(\hat Y(s)) : \dx B_s \biggr|\label{eq:est:deg:3}  \\
	&\qquad+ \sup_{0<t\leq T}\biggl| \int_0^{t} \nabla_w H(\hat Y_n(s)): \dx B^n_s -  \int_0^{t} \nabla_w H(\hat Y(s)) : \dx B^n_s \biggr| \,.\label{eq:est:deg:4} 
\end{align}
We recall the convergence from~\eqref{eq:hn} and observe that all the integrals are relatively compact by \cite[Proposition 4.4]{Katzenberger91}. 
By noting that $\hat Y(s), \hat Y_n(s) \in K$ with $K$ compact, 
by using assumption on the regularity of~$\hat L$, 
by~\cite[Theorem 2.2]{KurtzProtter1991} we conclude that the terms ~\eqref{eq:est:deg:1} and~\eqref{eq:est:deg:3} converge $\Rightarrow 0$. 
The terms~\eqref{eq:est:deg:2} and~\eqref{eq:est:deg:4} convergence by the continuous mapping theorem 
\[
\nabla_w f(\hat Y_n(s))\Rightarrow  \nabla_w f(\hat Y(s)) \qquad \text{and} \qquad \nabla_w H_f(\hat Y_n(s))\Rightarrow  \nabla_wH_f(\hat Y(s)) \,.
\]
Hence, by an application of \cite[Theorem 2.2]{KurtzProtter1991}, we get the convergence
\begin{align*}
	\MoveEqLeft \sup_{0<t\leq T}\biggl|  \int_0^{t} \nabla_w f(\hat Y_n(s))\cdot  \dx b^n_s -  \int_0^{t} \nabla_w f(\hat Y(s)) \cdot \dx b^n_s \biggr|\\
	&\leq \sup_{0<t\leq T}\biggl|  \int_0^{t} \nabla_w f(\hat Y_n(s))\cdot \dx b^n_s -  \int_0^{t} \nabla_w f(\hat Y(s)) \cdot \dx b_s \biggr| \\
	&\quad + \sup_{0<t\leq T}\biggl|  \int_0^{t} \nabla_w f(\hat Y(s)) \cdot \dx b^n_s -  \int_0^{t} \nabla_w f(\hat Y(s)) \cdot \dx b_s \biggr| \Rightarrow 0.
\end{align*}
And similarly \eqref{eq:est:deg:4} $\Rightarrow 0$.
Finally note that
\[
\dx [\hat{Z}^i, \hat{Z}^j]_s =  \partial_{w_i} f(\hat{Y}(s)) \cdot \partial_{w_j} f(\hat{Y}(s)) +  \sigma_0\nabla_{w_i}H_f(\hat Y(s)) : \sigma_0\nabla_{w_j}H_f(\hat Y(s))  = \varSigma(Y(s))_{ij} \dx s \,.
\]
Hence, the stated result in Theorem~\ref{th:main:deg2} follows by applying Theorem~\ref{th:Katzenberg}.
\end{proof}

\section{Examples}
\label{sec:examples}

In this section we apply Theorems~\ref{th:main} and~\ref{th:main:deg2} to a number of examples. Table~\ref{table:examples} gives an overview. 
\begin{table}[ht]
	\small
	\def\arraystretch{1.3}
	\centering
\begin{tabular}{lcccc}
\textbf{Type} & $\hat L(w,\eta)$ & $\Reg(w)$ & \textbf{Rate} & \textbf{Section}\\ \midrule[1pt]
\multicolumn{5}{l}{\emph{Independent of the structure of $L$, with $\eta$ having the same dimension as $w$:}}\\[\jot]
\rlap{Gaussian D-Connect} & $L(w{\odot}(1{+}\eta))$ & $\frac12 \partial^2L(w)[w,w]$ & $\alpha_n\sigma_n^2$ & \ref{sss:DropConnect}\\
\rlap{Bernoulli D-Connect} & $L(w{\odot}(1{+}\eta))$ & \eqref{eqdef:Reg-BDropConnect}& $\alpha_n\sigma_n^2$ & \ref{sss:DropConnect}\\
SGLD & $L(w) + \frac12 w^\top \eta$ & $\frac12 \log |\nabla^2 L|_+$ & $\alpha_n^2\sigma_n^2$ & \ref{sec:examples:SGLD}\\
Anti-PGD & $L(w+\eta)$ & $\frac12 \Delta_w L(w)$ & $\alpha_n\sigma_n^2$ & \ref{sss:anti-correlated-PGD}\\
\midrule
\multicolumn{5}{l}{\emph{Variables $\eta$ indexed by $i$, the index of the samples:}}\\[\jot]
Mini-\rlap{batching} & $\frac1N \sum_{i=1}^N (1{+}\eta_i) \ell(w,x_i,y_i)$ & 0 & $\ll\alpha_n^2\sigma_n^2$ & \ref{sec:examples:mini}\\
Label noise & $\frac1N \sum_{i=1}^N (f_w(x_i)-y_i-\eta_i)^2$ & $\frac1{2N} \Delta L(w)$ & $\alpha_n^2\sigma_n^2$ & \ref{sec:examples:label}\\
\midrule
\multicolumn{5}{l}{\emph{Classical Dropout, with mean-square loss:}}\\[\jot]
OLM-DO & $f_w^{\mathrm{drop}}(x,\eta) = \langle u^{\odot2}{-}v^{\odot2},x{\odot}(1{+}\eta)\rangle$ & $\frac1{2N} \sum_{i,j} (u_j^2-v_j^2)^2 x_{ij}^2$ & $\alpha_n\sigma_n^2$ & \ref{sss:OLM-DO}\\
ShNN-DO & $f_w^{\mathrm{drop}}(x,\eta) = \sum_{j} a_j (1{+}\eta_j)s(b_j^\top x)$ & $\frac1{N} \sum_{i,j} a_j^2 s(b_j^\top x_i)^2$ & $\alpha_n\sigma_n^2$ & \ref{sss:FFNN}\\
DeepNN-DO  & \eqref{eqdef:deep-ReLU-NN} & $\frac12  \Delta_\eta\hat L(w, 0)$ & $\alpha_n \sigma_n^2$ & \ref{sss:FFNN}
\\[\jot]
\bottomrule
\end{tabular}
\caption{List of examples discussed in Section~\ref{sec:examples}. The column `Rate' indicates the rate at which iterations should be mapped to continuous time in order to follow the evolution; for instance, `$\alpha_n\sigma_n^2$' means that step $k$ is mapped to time $t_k := \alpha_n\sigma_n^2k$. Therefore smaller numbers mean slower evolution, requiring stronger speedup to see non-trivial evolution. See the referenced sections for details.}
\label{table:examples}
\end{table}

\bigskip

In each of the examples below we apply either Theorem~\ref{th:main} or Theorem~\ref{th:main:deg2} to obtain a convergence result. Because the statements of those theorems involve a specific type of convergence, which deals with the initial jump and requires a proper speedup and restriction to a compact set, we introduce the notion of \emph{Katzenberger convergence} to simplify the formulation of the results below.

\medskip
As in Section~\ref{s:main-results}, consider a sequence of stochastic processes $(W_n)_{n\in\bbN}$ and a sequence of integrators $(A_n)_{n\in \bbN}$, and let $Y_n$ be the process after the jump correction
\begin{equation}
Y_n(t):=W_n(t)-\phi\bigl(W_n(0), A_n(t)\bigr)+\Phi\left(W_n(0)\right).
\end{equation}
For a compact set $K$ we again set 
\[
\mu_n(K):=\inf\bigl\{t \geq 0 \mid Y_n(t-) \notin K^{\circ} \text{ or } Y_n(t) \notin K^{\circ} \bigr\} .
\]
Let $Y$ be a deterministic or stochastic process with continuous sample paths and set 
\[
\tau := \inf \set[\big]{t\geq 0\mid Y(t)\notin K^\circ}.	
\]
Theorems~\ref{th:main} and~\ref{th:main:deg2} provide convergence in distribution; by the Skorokhod representation theorem, we can assume without loss of generality that the processes $W_n$, $Y_n$, and $Y$ are defined on the same probability space. 

\begin{definition}[Katzenberger convergence]
	\label{def:Katzenberger-convergence}
We say that $W_n$ converges \emph{in the sense of Katzenberger} in $U$ to $Y$ if for every compact subset $K \subset U$
we have
\[
Y_n^{\mu_n(K)\wedge\tau} \Rightarrow Y^\tau	
\qquad\text{and for all $\delta>0$, \quad}
W_n^{\mu_n(K)\wedge \tau}\bigm\vert_{[\delta,\infty)}
\Rightarrow Y^\tau\bigm\vert_{[\delta,\infty)} .
\]
Here the convergences are in Skorokhod topology. 
\end{definition}

In Section~\ref{s:main-results} we discussed that Theorems~\ref{th:main} and~\ref{th:main:deg2} provide such convergence when the limiting process $Y$ has a uniquely defined deterministic or stochastic evolution up to time~$\tau$.
    
\subsection{Methods independent of \texorpdfstring{$L$}{L}: DropConnect and SGLD}
\label{ss:DropConnect}

\subsubsection{DropConnect}\label{sss:DropConnect}
\emph{Dropout} is a regularization technique originally introduced for neural networks in~\cite{wager2013dropout, srivastava2014dropout}. The idea is to balance the importance of all the neurons by temporarily disabling some of them at every iteration of the optimization. In this section we study the specific case of \emph{DropConnect}~\cite{WanZeilerZhangLe-CunFergus13}, which is more general than other types of Dropout, in the sense that it can be applied to `any' loss function $L$, not just those generated by neural networks.

In DropConnect, given a loss function $L$, we construct the noise-injected loss $\hat L$ by
\begin{equation}
\label{eqdef:Lhat-DropConnect}
\hat L(w,\eta) := L(w\odot (1+\eta))\, .
\end{equation}
In the most common form the filters $\eta$ are chosen to be i.i.d.\ Bernoulli random variables, 
\begin{equation}
\label{eqdef:Bernoulli-filters}
\eta_{i} = \begin{cases}
        -1 & \text{ w.p. } p\,,\\
        \dfrac{p}{1-p} & \text{ w.p. }  1-p\,.\\
\end{cases}
\end{equation}
Note that $\bbE\eta_{i} = 0$ and $\sigma^2 := \Var (\eta_{i})= {p}/{(1-p)} \to 0$ as $p\to 0$. 
When $\eta_{i}= -1$, the corresponding parameter is zeroed, or dropped out, which explains the name. Note that the limit of small~$\sigma^2$ is the limit in which $p\downarrow0$, i.e.\ in which vanishingly few parameters are dropped. 

As an alternative for Bernoulli DropConnect we also consider Gaussian DropConnect, in which the $\eta_{i}$ are i.i.d.\ normal random variables
\begin{equation}
	\label{eqdef:Gaussian-filters}
\eta_{i} \sim \calN(0, \sigma^2).
\end{equation}

\paragraph*{Gaussian DropConnect.}
The case of Gaussian noise variables $\eta$ fits directly into the structure of Theorem~\ref{th:main}, and we now state the corresponding result.

\begin{corollary}[Convergence for Gaussian DropConnect]
\label{cor:GaussianDropConnect}
Let $L\in C^4(\R^m)$ satisfy Assumption~\ref{ass:LossMfd}, and assume that $L$ has at most polynomial growth. Define $\hat L$ by~\eqref{eqdef:Lhat-DropConnect}, and let $W_n$ be the process characterized in~\eqref{eq:def:W}, where the $\eta_{k,i}$ are i.i.d.\ Gaussian variables as in~\eqref{eqdef:Gaussian-filters}. Assume that ${W}_n(0) \Rightarrow W_0 \in U_{\Gamma}$ for some locally attractive neighbourhood $U_\Gamma$ of $\Gamma$. Finally, let $\alpha_n$ and $\sigma_n$ be arbitrary sequences converging to zero. 

Then for any compact set $K \subset U_\Gamma$, the process $W_n$ converges to $Y$ in the Katzenberger sense (see Definition~\ref{def:Katzenberger-convergence}), where $Y$ solves the constrained gradient-flow equation 
\[
\frac{dY}{dt} =- P_{T\Gamma} \nabla_w\Reg (Y), \qquad Y(t)\in \Gamma,
\qquad \text{and}\qquad 
Y(0) = \Phi(W_0),
\]
with
\begin{equation}
	\label{eqdef:Reg-DropConnect}
\Reg(w) := \frac12 \Delta_\eta \hat L(w,0) = \frac12 \partial^2L(w)[w,w]
= \frac12 \sum_{j=1}^m w_j^2 \partial^2_{w_jw_j}L(w)\,.
\end{equation}
\end{corollary}

\begin{proof}[Proof of Corollary~\ref{cor:GaussianDropConnect}]
	The result follows fairly directly from Theorem~\ref{th:main}. If $L$ grows at most polynomially, then $\hat L$ and the Gaussian random variables satisfy Assumption~\ref{ass:hatL} for some $p$, and by Lemma~\ref{l:supgaus} the condition~\eqref{ass:speed} also is satisfied. The assertion of Corollary~\ref{cor:GaussianDropConnect} is then a translation of the assertion of Theorem~\ref{th:main}.
\end{proof}

\begin{remark}[Empirical loss]
If $L$ is an empirical loss of the form
\begin{equation}
	\label{eqdef:EmpLoss-section-DropConnect}
L(w) := \frac1N \sum_{i=1}^N \bra[\big]{f(w,x_i) - y_i}^2\,,
\end{equation}
then on $\Gamma$ the function $\Reg$ can be written as  
\begin{equation}
	\label{eqdef:Reg-DropConnect-MSEloss}
\Reg(w) = \frac1N \sum_{i=1}^N \bigl\lvert w\odot \nabla_w f(w,x_i) \bigr\rvert^2\,.
\end{equation}
\end{remark}

\begin{remark}[Connection with generalisation]\label{rem:generalisation}
The form of the function $\Reg$ in~\eqref{eqdef:Reg-DropConnect} is very similar to the `relative flatness' introduced by Petzka \emph{et al.}~\cite{PetzkaKampAdilovaSminchisescuBoley21}, and for which they prove rigorous generalisation benefits. 

Petzka and co-authors consider functions $f$ of the form $f(w,x) = g(w \phi(x))$, where $w$ is organized as a matrix, with $w\in \R^{m_1\times m_2}$ and $x\in \R^{m_2}$. Taking $\phi(x) = x$ to connect with this paper, the `relative flatness'~\cite[Def.~3]{PetzkaKampAdilovaSminchisescuBoley21} of the empirical loss of such an $f$ can be written as 
\begin{equation}
	\label{eqdef:relative-flatness}
\frac2N \sum_{i=1}^N\sum_{s,s'=1}^{m_1} \sum_{j,j'=1}^{m_2} w_{sj}w_{s'j}
   \partial_{w_{sj'}}f(w, x_i) \partial_{w_{s'j'}}f(w, x_i)
\end{equation}
To compare, we can write the expression~\eqref{eqdef:Reg-DropConnect-MSEloss} in a very similar form,
\begin{equation}
	\label{eqdef:Reg-DropConnect-MSEloss-v2}
\frac1N \sum_{i=1}^N\sum_{s,s'=1}^{m_1} \sum_{j,j'=1}^{m_2} w_{sj}w_{s'j}
   \partial_{w_{sj'}}f(w, x_i) \partial_{w_{s'j'}}f(w, x_i)\ \delta_{ss'}\delta_{jj'}\,.
\end{equation}
Note how~\eqref{eqdef:Reg-DropConnect-MSEloss-v2} is a `diagonal' form of~\eqref{eqdef:relative-flatness}.

The two expressions~\eqref{eqdef:Reg-DropConnect-MSEloss-v2} and~\eqref{eqdef:relative-flatness} share a number of properties that are consistent with good generalisation behaviour. To start with, as also remarked in~\cite{PetzkaKampAdilovaSminchisescuBoley21}, both are invariant under coordinate-wise rescaling of the parameters $w$, because each derivative is accompanied by a multiplication with the corresponding coordinate of $w$. They also scale quadratically in $f$, just as the loss function~\eqref{eqdef:EmpLoss-section-DropConnect}, and if one considers $f$ to be a  neural network of depth $D$ (see below) then the scaling in terms of $w$ is of the form $w^{2D}$, both for the loss and for the expressions~\eqref{eqdef:Reg-DropConnect-MSEloss-v2} and~\eqref{eqdef:relative-flatness}. These scaling properties are necessary for a functional to be admissible as a measure of generalisation ability.
\end{remark}

\paragraph{Bernoulli DropConnect.}

The case of Bernoulli-distributed $\eta$ as in~\eqref{eqdef:Bernoulli-filters} is not covered by Theorem~\ref{th:main}, because all higher moments $\M_k$ for $k\geq3$ scale as $\sigma^2$, and therefore violate the condition~\eqref{eq:ass:moment-C2>0}. In general, we have  $C_2>0$ in Assumption~\ref{ass:hatL}, and therefore they also do not satisfy condition~\eqref{eq:ass:moment-C2=0}, except for the special case of shallow neural networks---see Section~\ref{sss:FFNN} below. 
As it turns out, however, the proof of Theorem~\ref{th:main} does apply to this situation, provided we prove the three statements of Lemma~\ref{lem:props-Zn} for this particular case. This leads to the following proposition.

\begin{proposition}[Convergence for Bernoulli DropConnect]
\label{cor:BernoulliDropConnect}
Let $L\in C^4(\R^m)$ satisfy Assumption~\ref{ass:LossMfd}. Define $\hat L$ by~\eqref{eqdef:Lhat-DropConnect}, and let $W_n$ be the process characterized in~\eqref{eq:def:W}, where the $\eta_{k,i}$ are i.i.d.\ Bernoulli random variables as in~\eqref{eqdef:Bernoulli-filters}, with dropout probability $p_n$. Assume that ${W}_n(0) \Rightarrow W_0 \in U_{\Gamma}$ for some locally attractive neighbourhood $U_\Gamma$ of $\Gamma$. Finally, let $\alpha_n\to0$ and $\sigma_n^2 := p_n/(1-p_n) \to 0$.

Then for any compact set $K \subset U_\Gamma$, the process $W_n$ converges to $Y$ in the Katzenberger sense (see Definition~\ref{def:Katzenberger-convergence}), where $Y$ solves the constrained gradient-flow equation 
\[
\frac{dY}{dt} =- P_{T\Gamma} \nabla_w\Reg (Y), \qquad Y(t)\in \Gamma,
\qquad \text{and}\qquad 
Y(0) = \Phi(W_0),
\]
with
\begin{equation}
	\label{eqdef:Reg-BDropConnect}
	\Reg(w) := \partial L(w)[w] +\sum_{j=1}^m \bra[\big]{L(w\odot \invunit_j) - L(w)}.
\end{equation}
Here $\invunit_j$ is the inverted unit vector in $\R^m$,
\[
\invunit_j := (1,\dots, 1, \underset {\mathclap{j^{\mathrm{th}}\text{ position}}}{\ \ \ 0\ \ \ ,} 1, \dots, 1 )	 = 1 - \mathrm e_j.
\]
\end{proposition}

\begin{remark}[Structure of the Bernoulli DropConnect regularizer]
\label{rem:Structure-BDO-Reg}
{\newcommand{\mej}{\mathrm e_j}
The functional form of $\Reg$ in~\eqref{eqdef:Reg-BDropConnect} can be interpreted as a version of the Gaussian regularizer~\eqref{eqdef:Reg-DropConnect} with the local derivative replaced by a non-local one. 
Indeed, writing $\invunit_j = 1-\mej$ we Taylor-develop $L$ as
\begin{align*}
L(w\odot \invunit_j) &= L(w\odot (1-\mej))\\
&= L(w) - \partial L(w)[\mej\odot w] + \frac12 \partial^2 L(w) [\mej \odot w, \mej \odot w] + O(\lvert \mej\odot w\rvert^3).
\end{align*}
If we pretend for the moment that $\mej$ is `small', and discard the final term, then
\begin{align*}
\sum_{j=1}^m \bra[\big]{L(w\odot \invunit_j)-L(w)}
&\approx \sum_{j=1}^m \pra[\Big]{-\partial L(w)[\mej\odot w]
+ \frac12 \partial^2 L(w) [\mej \odot w, \mej \odot w]}\\
& = - \partial L(w) [w] + \frac12 \sum_{j=1}^m w_j^2 \partial^2_{w_jw_j} L(w).
\end{align*}
In this approximation, therefore, the Bernoulli regularizer~\eqref{eqdef:Reg-BDropConnect} reduces to the Gaussian regularizer~\eqref{eqdef:Reg-DropConnect}.
}
\end{remark}

\begin{proof}
As described above, we re-prove the assertions of Lemma~\ref{lem:props-Zn}, i.e.\ properties~(\ref{lem:props-Zn:Ass39}--\ref{lem:props-Zn:quadratic-variation}),  for this case. Given those assertions the proof of Theorem~\ref{th:main} applies to this situation. 

We first estimate $\Delta\sfZ_n(w,\eta) = \alpha_n \nabla L(w) - \alpha_n \nabla_w \bra[\big]{L(w\odot(1+\eta))}$ for this setup. Let $\tilde K\supset \hat K$ be a compact set that is large enough to contain $w(1+\eta)$ for all $\eta$ and all $w\in \hat K$. Using the boundedness of derivatives of $\hat L$ on $\tilde K$, we then have for all $\eta$ and all $w\in \hat K$ 
\[
|\Delta \sfZ_n(w,\eta)|
\leq
\alpha_n \bra*{1 + \frac 1{1-p_n}}\|\nabla L\|_{L^\infty(\tilde K)} = O(\alpha_n),
\]
which proves~\eqref{lem:props-Zn:Ass39}. To prove the estimate on $|\Delta \sfZ_n|^2$ in~\eqref{lem:props-Zn:quadratic-variation} we note that on the event $E$ that all coordinates of $\eta$ are non-zero, which has probability $(1-p_n)^m = 1-O(\sigma_n^2)$, we have the additional estimate
\[
\sup_{w\in K}|\Delta \sfZ_n(w,\eta)|
\leq 
\alpha_n \|\nabla^2 L\|_{L^\infty(\tilde K)} \underbrace{\frac{p_n}{1-p_n}}_{\sigma_n^2}\sup_{w\in K}|w|
= O(\alpha_n\sigma_n^2).
\]
It follows that for all $w\in \hat K$, 
\[
\Expectation_\eta \sup_{w\in \hat K}|\Delta \sfZ_n(w,\eta)|^2 
\leq O(\alpha_n^2 \sigma_n^4)\Prob(E) + O(\alpha_n^2)  \Prob(E^c)
= o(\alpha_n \sigma_n^2),
\]
implying the first estimate in~\eqref{lem:props-Zn:quadratic-variation}.

We next estimate $\Delta \sfF_n$ as defined in~\eqref{eq:def:DeltaFn}.
The expectation over the set of independent Bernoulli random variables can be expressed in terms of $\{0,1\}^m$-valued outcomes $b$ in the form
\begin{align*}
\Delta \sfF_n(w) &= \Expectation_\eta \Delta\sfZ_n(w,\eta)
= \alpha_n \nabla L(w) - \alpha_n \Expectation_\eta \nabla_w \bra[\big]{L(w\odot (1+\eta))}\\
&= \alpha_n \nabla L(w) - \alpha_n \sum_{b\in \{0,1\}^m} p_n^{m-|b|}(1-p_n)^{|b|-1}\,b\odot \nabla L\bra*{\frac1{1-p_n} w\odot b}.
\end{align*}
We estimate the term inside the sum according to the number of zero coordinates in $b$:
\begin{multline*}
b\odot (\nabla L)\biggl(\frac1{1-p_n} w\odot b\biggr) 
= \begin{cases}
\nabla L(w) + \frac{p_n}{1-p_n} \partial \nabla L(w)[w]  + O(p_n^2) & \text{if }|b|=m,\\
\invunit_j \odot (\nabla L) (w\odot \invunit_j) + O(p_n) & \text{if }|b|=m-1, \  b = \invunit_j,\\
O(1) &\text{if }|b|\leq m-2,
\end{cases}
\end{multline*}
with the $O$ symbols being uniform in $w\in \hat K$.

Using the expression for the gradient of the regularizer~\eqref{eqdef:Reg-BDropConnect},
\[
\nabla \Reg(w) = \partial\nabla L(w)[w] - (m-1)\nabla L(w) + \sum_{j=1}^m \invunit_j\odot (\nabla L)(w\odot \invunit_j).	
\]
we then  split the sum over $b$ into three parts and estimate accordingly
\[
\Delta\sfF_n(w) + \alpha_n\sigma_n^2 \nabla \Reg(w) = \mathrm I+\mathrm{II}+\mathrm{III}, 	
\]
with 
\begin{align*}
\frac1{\alpha_n}\mathrm{I} &= \nabla L(w) - (1-p_n)^{m-1} \nabla L(w) - {p_n}(1-p_n)^{m-2} \partial \nabla L(w)[w]  + O(p_n^2)\\
& \qquad {}
+ \sigma_n^2 \partial\nabla L(w) [w] - \sigma_n^2 (m-1)\nabla L(w),\\
\frac1{\alpha_n}\mathrm{II} &= \sum_{j=1}^m \pra[\Big]{
	-p_n(1-p_n)^{m-2}\, \invunit_j \odot (\nabla L)(w\odot \invunit_j)
	+ \sigma_n^2\, \invunit_j \odot (\nabla L)(w\odot \invunit_j)} + O(p_n^2),\\
\frac1{\alpha_n}\mathrm{ III} &= \sum_{|b|\leq m-2} p_n^{m-|b|}(1-p_n)^{|b|-1}b\odot \nabla L\bra*{\frac1{1-p_n} w\odot b}.
\end{align*}
We estimate the terms one by one. For $\mathrm I$ we find
\begin{align*}
\frac1{\alpha_n}\mathrm{ I} & = \!\pra[\Big]{1 - \bigl(1-(m-1)p_n +O(p_n^2)\bigr) - \sigma_n^2 (m-1)} \nabla L(w)
+ \bra[\big]{-p_n(1-p_n)^{m-2} + \sigma_n^2 }\partial\nabla L(w)[w]\\
&= O(p_n^2) = O(\sigma_n^4), \qquad \text{uniformly in $w\in \hat K$}.
\end{align*}
For $\mathrm{II}$ we write
\begin{align*}
\frac1{\alpha_n}\mathrm{II} &= \sum_{j=1}^m \pra[\Big]{
-p_n(1-p_n)^{m-2}
+ \sigma_n^2}\, \invunit_j \odot (\nabla L)(w\odot \invunit_j) + O(p_n^2) = O(p_n^2) = O(\sigma_n^4), 
\end{align*}
and for the third term we immediately find $\alpha_n^{-1} \mathrm{III }= O(p_n^2)= O(\sigma_n^4)$. Combining all these estimates we conclude that
\[
\sup_{w\in K}\abs[\big]{\Delta\sfF_n(w) + \alpha_n\sigma_n^2 \nabla \Reg(w)} = O(\alpha_n \sigma_n^4), 
\]
thereby establishing~\eqref{lem:props-Zn:Fn-Reg}.

Finally, to prove~the second estimate in~$\eqref{lem:props-Zn:quadratic-variation}_2$ we write
\[
|\Delta \sfF_n(w)|
\leq |\Delta\sfF_n(w) + \alpha_n\sigma_n^2 \nabla \Reg(w)| + |\alpha_n\sigma_n^2 \nabla \Reg(w)|, 
\]
and the estimate~$\eqref{lem:props-Zn:quadratic-variation}_2$ follows from~\eqref{lem:props-Zn:Fn-Reg} and the regularity of $\Reg$.
\end{proof}

\subsubsection{Stochastic Gradient Langevin Descent}
\label{sec:examples:SGLD}

Stochastic Gradient Langevin Descent~\cite{GelfandMitter1991,RaginskyRakhlinTelgarsky17,mei2018mean1} is a form of gradient descent in which at each iteration a centered Gaussian perturbation is added to the gradient. In our notation this corresponds to 
\begin{equation}
	\label{eqdef:Lhat-SGLD}
\hat L(w,\eta) := L(w) + \frac12 w^\top \eta,
\qquad \text{with } \eta\sim \calN(0,\sigma^2 I_m) \text{ i.i.d.}.
\end{equation}
This structure is of the degenerate form~\eqref{eq:Ldeg} and therefore is covered by Section~\ref{ss:degenerate-case}. This results in the following Corollary. 

\begin{corollary}[Convergence for Stochastic Gradient Langevin Descent] $ $\\
Let $L\in C^4(\R^m)$ satisfy Assumption~\ref{ass:LossMfd}. Define $\hat L$ by~\eqref{eqdef:Lhat-SGLD}, and let $W_n$ be the process characterized in~\eqref{eq:def:W}, where the $\eta_{k,i}$ are i.i.d.\ centered normal random variables with variance $\sigma_n^2$. Assume that ${W}_n(0) \Rightarrow W_0 \in U_{\Gamma}$ for some locally attractive neighbourhood $U_\Gamma$ of $\Gamma$. Finally, let $\alpha_n \to0$ and $ \sigma_n \to \sigma_0\geq 0$.

Then for any compact set $K \subset U_\Gamma$, the process $W_n$ converges to $Y$ in the Katzenberger sense (see Definition~\ref{def:Katzenberger-convergence}), where $Y$ solves the SDE 
\begin{equation}
	\label{eq:constr-SDE-SGLD}
\dx  Y(t) = P_{T\Gamma} \dx b(t) + \frac12 (\grad^2L( Y))^\dagger \partial^2 \nabla L( Y) [P_{T\Gamma}] \, dt 
- \frac12 P_{T\Gamma} \nabla \log \lvert\grad^2 L( Y)\rvert_+\, \dx t,
\end{equation}
constrained to $Y(t)\in \Gamma$, with $Y(0) = \Phi(W_0)$. (Recall that $\lvert A\rvert_+$ is the product of the non-zero eigenvalues of $A$, and~$\dagger$ indicates the Moore-Penrose pseudoinverse).
\end{corollary}

\begin{remark}
	As observed in~\cite{LiWangArora22}, the first two terms in~\eqref{eq:constr-SDE-SGLD} form a geometrically intrinsic Brownian motion on the manifold $\Gamma$, while the final term acts as a drift parallel to~$\Gamma$.
\end{remark}

\begin{proof}
The assumptions of Theorem~\ref{th:main:deg2} are satisfied, with $f(w) = w/2$ and $H(w) =0$ in~\eqref{eq:Ldeg} and using Lemma~\ref{l:supgaus} to show~\eqref{ass:speed:deg1}. We conclude that the evolution converges to the constrained SDE~\eqref{eq:th:main2:constrained-SDE}, which reduces in this case to 
\begin{equation*}
\dx  Y(t) = P_{T\Gamma} \dx b(t) + \frac12 \partial^2 \Phi(Y(t))[I_m]\dx t,
\qquad \text{and}\qquad Y(t) \in \Gamma,
\end{equation*}
where $b$ is an $m$-dimensional Brownian motion and $I_m$ is the identity matrix.
The formulation~\eqref{eq:constr-SDE-SGLD} follows from applying the characterization~\eqref{char:Phi''-Sigma=id}.
\end{proof}

\begin{remark}[SGLD and generalisation]
Raginsky, Rakhlin, and Telgarsky~\cite{RaginskyRakhlinTelgarsky17} show quantitative generalisation bounds on SGLD, using optimal transport and Logarithmic Sobolev inequalities. Their result is meaningful in a limit of very many data points, and therefore focuses on an underparameterized setting, rather than the overparameterized setting as in this paper. 
\end{remark}

\subsubsection{`Anti-correlated perturbed gradient descent'}
\label{sss:anti-correlated-PGD}
Orvieto and co-workers~\cite{OrvietoKerstingProskeBachLucchi22} gave the name `anti-correlated gradient descent' to the simple scheme used in the introduction, 
\[
\hat L(w,\eta) := L(w+\eta).	
\]
As discussed there, this is an example of non-degenerate noisy gradient descent, with convergence to the constrained gradient flow driven by 
\[
\Reg(w) := \frac12 \Delta_w L(w).	
\]
Orvieto \emph{et al.} also directly investigate the generalisation properties of this type of noise injection. 

\subsection{Examples where \texorpdfstring{$\eta$}{the injected noise} is indexed by the sample}
\label{ss:eta-by-sample}

We now consider the supervised-learning context, in which the loss $L$ is an empirical average of a local loss function $\ell\geq 0$ over a set $\{(x_i, y_i)_{1\leq i\leq N}\}$ of `training data points',
\begin{equation}
	\label{eq:loss}
	L(w) = \frac1N\sum_{i=1}^N\ell(w, x_i, y_i).
\end{equation}
In this section we consider noise variables $\eta$ indexed by the same indices as the data points $(x_i,y_i)$, i.e.\ $\eta$ is a random  element of $\R^N$.

\subsubsection{Minibatching}
\label{sec:examples:mini}

One of the most widely used noisy gradient descent algorithms is the one with noise induced by the random sampling of the data points, usually referred simply as \emph{stochastic gradient descent}~(SGD). In the standard setting the whole dataset is randomly split into disjoint `minibatches' $\{B_k\}$ of the same size $m$, and at every iteration the gradient is calculated only for samples in one minibatch $B_k$:
\[
w_{k+1} = w_k - \alpha \nabla_w \Big(\frac1m\sum_{(x_i, y_i) \in B_k}\ell(w, x_i, y_i)\Big).
\]
We slightly modify this formulation by introducing i.i.d.\ random variables $\eta_{k,i}$ at every iteration, with distribution
\[
\eta_{k,i} = \begin{cases}
        -1  &\text{w.p. } 1-\dfrac{m}{N},\\
        \dfrac{N-m}{m}  &\text{w.p. }  \dfrac{m}{N},\\
\end{cases}
\]
and the corresponding noisy loss function $\hat{L}$:
\[
\hat{L}(w, \eta) =\frac{1}{N} \sum_{i=1}^N(1+\eta_i) \ell(w, x_i, y_i).
\]
One can see that indeed $\hat{L}(w, 0) = L(w)$. Moreover, if we write the corresponding `minibatch' as 
\[
B_k = \left\{(x_i, y_i): \eta_{k,i} = \frac{N-m}{m}\right\},
\]
then the noisy gradient descent~\eqref{eq:NGD-intro} becomes
\[
w_{k+1} = w_k - \alpha \nabla_w \hat{L}(w, \eta_k)=  w_k - \alpha \nabla_w \Big(\frac1m\sum_{(x_i, y_i) \in B_k}\ell(w, x_i, y_i)\Big).
\]
This version of SGD can be interpreted as deciding at every iteration which data point to include, independently for each data point and independently for each iteration. In contrast to the standard SGD algorithm, in such a setting the minibatch size is not fixed, and during every epoch the same data point might not occur or can occur multiple times. The parameter $m$ also is not a deterministic minibatch size but the expectation of the minibatch size. 

We show that for this modification of minibatch SGD the limiting dynamics are trivial on the time scale $1/\alpha_n^2$ for any fixed parameter $m$, by applying Theorem~\ref{th:main:deg2}. Similarly, one can show that the joint limit $m_n \to N$, $\alpha_n \to 0$ also results in a trivial process by applying Theorem~\ref{th:main}. Thus, we argue that minibatch noise affect the training dynamics around the zero-loss manifold only in a weak way, and additional forms of noise injection might be required to ensure good generalization properties on this time scale.
\begin{corollary}
\label{crl:mini}
Let $\hat{L}, \eta_{k, i}$ be as defined above. Assume that $\ell$ is $C^3$ in $w$ and that $ L$ satisfies Assumption~\ref{ass:LossMfd}. Let $\hat{W}_n(0) \Rightarrow  W_0 \in U_{\Gamma}$ for some locally attractive neighbourhood $U_\Gamma$ of $\Gamma$. Let $\alpha_n\to0$ and $\sigma_n\to \sigma_0\geq0$. Then $\hat W_n$ given by~\eqref{eq:def:hatW} converges in the sense of Katzenberger (see Definition~\ref{def:Katzenberger-convergence}) to the trivial process
\[
\hat{Y}(t) = \Phi( W_0) \qquad \text{for all }t\geq0.
\]
\end{corollary}

\begin{proof}

The function $\hat{L}(w, \eta) $ is of the form~\eqref{eq:Ldeg} with $f(w)_i = N^{-1}\ell(w,x_i,y_i)$, $i=1,\dots, N$,  and applying  Theorem~\ref{th:main:deg2} we find the limiting dynamics
\begin{equation}
\hat Y(t)=\Phi(\hat W(0)) + \frac1N\int_0^{t \wedge \mu} P_{T\Gamma} \nabla_wf(w) db_s+\frac12\int_0^{t \wedge \mu}  \partial^2 \Phi(\hat Y(s)) [\varSigma(s)] ds,\label{eq:minibatch}
\end{equation}
where
\[
\varSigma(s) = \frac{1}{N^2}\sum_{i=1}^N \nabla_w\ell(\hat Y(s), x_i, y_i) \otimes \nabla_w\ell(\hat Y(s), x_i, y_i) \,.
\]
As $\hat Y(s)\in \Gamma$ a.s., we have 
\[
L(\hat Y(s)) = \frac1N\sum_{i=1}^N\ell(\hat Y(s), x_i, y_i) = 0,
\]
implying that 
\begin{equation}
 \nabla_w\ell(\hat Y(s), x_i, y_i) =0. \label{eq:mb}
\end{equation}
In turn this implies that $f$ and $\varSigma$ vanish on $\Gamma$, resulting in the trivial dynamics $\hat Y(t) = \Phi(\hat W(0))$ for all $t$. 
\end{proof}

\subsubsection{Label Noise}
\label{sec:examples:label}
\emph{Label noise} is the specific case of mean squared error $L$ with noisy perturbation of the labels~$y_i$, which in our setting takes the form
\begin{equation}
\label{eq:labnoise}
\hat{L}(w, \eta) = \frac1N\sum_{i=1}^N(f_w (x_i) - y_i - \eta_i)^2.
\end{equation}
The consequences of label noise for the training iterates were studied by Blanc \emph{et al.}~\cite{BlancGuptaValiantValiant20} and the already-mentioned Li, Wang, and Arora~\cite{LiWangArora22}.

The function $\hat L$ is of the degenerate form~\eqref{eq:Ldeg}, with $f(w)_i = -2(f_w(x_i)-y_i)/N$, $H(w)=0$, and $g(\eta) = N^{-1} \sum_{i=1}^N \eta_i^2$. As a result,  Theorem~\ref{th:main:deg2} provides  the limiting dynamics at rate $1/\alpha_n^2\sigma_n^2$, both for $\sigma_n\to0$ and $\sigma_n\to \sigma_0>0$.

\begin{corollary}[Convergence for Label Noise]
\label{crl:label}
Let $\hat{L}$ be as defined in \eqref{eq:labnoise} for some  family of functions $f_w$ such that $L$ is of class $C^3$ and satisfies Assumption~\ref{ass:LossMfd}. Let $\tilde{W}_n(0) \Rightarrow W_0 \in U_{\Gamma}$ for some locally attractive neighbourhood $U_\Gamma$ of $\Gamma$. Let $\alpha_n\to0$ and $\sigma_n\to \sigma_0\geq0$. 
Let $\eta_{k,i} \sim \rho(\sigma_n)$ i.i.d., where $\rho$ satisfies~\eqref{eq:ass:noise} and~\eqref{ass:speed:deg1}.

Then $\hat W_n$ given by~\eqref{eq:def:hatW} converges in $U_\Gamma$ in the sense of Katzenberger to the constrained gradient flow $Y$ given by
\[
\frac{dY}{dt} =- P_{T\Gamma} \nabla_w\Reg (Y), \qquad Y(t)\in \Gamma,
\qquad \text{and}\qquad 
Y(0) = \Phi(W_0),
\]
with 
\begin{equation}
\label{eqdef:Reg-LabelNoise}
\Reg(w) = \frac 1{2N}  \Delta L(w).
\end{equation}
\end{corollary}

Corollary~\ref{crl:label} is effectively the same result as~\cite[Cor.~5.2]{LiWangArora22}; we give the proof both for completeness and to allow us to re-use the arguments in Corollary~\ref{cor:combination-labelnoise-minibatching}. 

\begin{proof}
As remarked above, the function $\hat L$ in~\eqref{eq:labnoise} is of the form~\eqref{eq:Ldeg} with $f(w)_i = -2(f_w(x_i)-y_i)/N$, $H(w)=0$, and $g(\eta) = N^{-1} \sum_{i=1}^N \eta_i^2$. Theorem~\ref{th:main:deg2} then provides Katzenberger convergence to the limiting SDE~\eqref{eq:th:main2:constrained-SDE}.  

To simplify this SDE, note that 
\[
\nabla_w f(w)_i  = -\frac2N\nabla_{w}(f_w(x_i) - y_i - \eta_i) = -\frac2N \nabla_{w}f_w(x_i),
\]
and for $w\in \Gamma$
\begin{align*}
\nabla^2{L}(w) &= \frac2N\nabla_w\sum_{i=1}^N(f_w(x_i) - y_i)\nabla_w f_w(x_i)= \frac2N\sum_{i=1}^N \nabla_{w}f_w(x_i)\nabla_{w}f_w(x_i)^\top .
\end{align*}
It follows that $\nabla_w f(w)\in \Range(\nabla^2{L}(w))$, and therefore $P_{T\Gamma} \nabla_w f(w) = 0$, implying that the noise term in~\eqref{eq:th:main2:constrained-SDE} vanishes. 

We also find that 
\[
\varSigma_{k\ell} := \sum_{i=1}^N \partial _{w_k} f(w)_i \partial_{w_\ell}f(w)_i = \frac4{N^2} \sum_{i=1}^N \partial_{w_k}f_w(x_i)\partial_{w_\ell}f_w(x_i),
\]
implying that $\varSigma = 2\nabla^2L(w) /N$. 
Applying~\eqref{char:Phi''-Sigma=nabla2L} to the final term in~\eqref{eq:th:main2:constrained-SDE} yields the expression~\eqref{eqdef:Reg-LabelNoise}.
\end{proof}

A minor modification of the discussion above shows that combination of minibatching and label noise results in the same limiting dynamics. Consider the following loss that combines minibatching noise variables $\tilde\eta$ and label noise variables $\eta$,
\begin{equation}
\label{eq:label+mini}
\hat{L}(w, \tilde\eta, \eta) =\frac{1}{N} \sum_{i=1}^N(1+\tilde \eta_i)(f_w (x_i) - y_i - \eta_i)^2,
\end{equation}
and for simplicity let both distributions have the same variance $\Var(\eta_{k, i}) = \Var (\tilde \eta_{k, j})= \sigma^2$.  The loss~\eqref{eq:label+mini} has the degenerate form~\eqref{eq:Ldeg},
\begin{align*}
\hat L(w,\eta,\tilde \eta)
&= L(w) + f(w)\cdot \eta + \tilde f(w) \cdot \tilde \eta + H(w): (\eta\oti\tilde \eta) + g(\eta,\tilde \eta),
\end{align*}
with
\begin{multline*}
f(w)_i = -\frac2N (f_w(x_i)-y_i),
\qquad
\tilde f(w)_i = \frac1N (f_w(x_i)-y_i)^2,\\
\quad\text{and}\quad
H(w)_{ij} = -\frac2N (f_w(x_i)-y_i)\delta_{ij}.
\end{multline*}
From Theorem~\ref{th:main:deg2} we obtain the following characterization.

\begin{corollary}[Convergence for combined Label Noise and Minibatching]
	\label{cor:combination-labelnoise-minibatching}
Let $\hat{L}$ be as defined in \eqref{eq:label+mini} for some  family of functions $f_w$ such that $L$ is of class $C^3$ and satisfies Assumption~\ref{ass:LossMfd}. Let $\tilde{W}_n(0) \Rightarrow W_0 \in U_{\Gamma}$ for some locally attractive neighbourhood $U_\Gamma$ of $\Gamma$. Let $\alpha_n\to0$ and $\sigma_n\to \sigma_0\geq0$. 
Let $\eta_{k,i},\tilde \eta_{k,i} \sim \rho(\sigma_n)$ i.i.d., where $\rho$ satisfies~\eqref{eq:ass:noise} and~\eqref{ass:speed:deg1}.

Then $\hat W_n$ given by~\eqref{eq:def:hatW} converges in the sense of Katzenberger to the constrained gradient flow $Y$ given by
\[
\frac{dY}{dt} =- P_{T\Gamma} \nabla_w\Reg (Y), \qquad Y(t)\in \Gamma,
\qquad \text{and}\qquad 
Y(0) = \Phi(W_0),
\]
with 
\begin{equation}
\label{eqdef:Reg-LabelMiniBatchNoise}
\Reg(w) = \frac{\sqrt{1+\sigma_0^2}}{2N} \Delta L(w).
\end{equation}
\end{corollary}

\begin{proof}
Similarly to the cases of Corollaries~\ref{crl:mini} and~\ref{crl:label} we have on $\Gamma$ that 
\begin{align*}
\nabla_w f(w)_i  &= -\frac2N \nabla_{w}f_w(x_i), \\
\nabla_w \tilde f(w)_i &= -\frac2N\nabla_{w}(f_w(x_i) - y_i )^2 =0,
\end{align*}
and 
\[
\nabla_w H(w)_{ij} =- \frac{2}{N}  \nabla_w f_w (x_i) \delta_{ij}, 
\]
and thus by the same orthogonality argument as in Corollary~\ref{crl:label} we conclude that the noise is orthogonal to $T\Gamma$. We apply Theorem~\ref{th:main:deg2} and find the resulting evolution~\eqref{eq:constr-SDE-SGLD}, where the first integral vanishes by orthogonality and we have the expression for $\varSigma$,
\begin{align*}
\varSigma_{k\ell} &= \frac4{N^2} \sum_{i=1}^N \partial_{w_k} f_w(x_i) \partial_{w_\ell} f_w(x_i)
+ \sigma_0^2 \frac4{N^2} \sum_{i=1}^N \partial_{w_k} f_w(x_i) \partial_{w_\ell} f_w(x_i)\\
 &= \frac4{N^2} (1+\sigma_0^2)\sum_{i=1}^N \partial_{w_k} f_w(x_i) \partial_{w_\ell} f_w(x_i).
\end{align*}
Similarly to Corollary~\ref{crl:label} it follows that $\varSigma = 2(1+\sigma_0^2)\Delta L(w)/N$, resulting in the expression~\eqref{eqdef:Reg-LabelMiniBatchNoise}.
\end{proof}

\subsection{Classical Dropout}
\label{ss:examples:ClassicalDropout}

We finally discuss the case of `classical' Dropout as applied in neural networks and other systems. We consider the mean-square error loss for a function $f_w^{\mathrm{drop}}(x,\eta)$ that depends both on the parameter $w$ and the noise variable $\eta$:
\begin{equation}
	\label{eqdef:Lhat-ClassicalDropout}
\hat{L}(w, \eta) = \frac1N\sum_{i=1}^N(f_w^{\mathrm{drop}}(x_i, \eta) - y_i)^2.
\end{equation}

\subsubsection{Overparameterized linear models}
\label{sss:OLM-DO}

The first example of this type is the class of  \emph{overparameterized linear models}, in which $ w  = (u, v)$, $u, v \in \bbR^{\din}$, and $f_ w $ has the following form:
\begin{equation}
\label{eqdef:OLM}
    f_ w (x) = \left<u^{\odot 2} - v^{\odot 2}, x\right>,
	\qquad x\in \R^{\din}.
\end{equation}
Note that the model is linear in $x$ but non-linear in parameters $ w $. We introduce Dropout filters $\eta \in \bbR^{\din}$ and a corresponding function $f^{\mathrm{drop}}_ w $ as
\begin{equation}
\label{eqdef:OLM-drop}
    f_ w ^{\mathrm{drop}}(x, \eta) = \left<u^{\odot 2} - v^{\odot 2}, x\odot(1+\eta)\right>.
\end{equation}

In comparison to DropConnect, the filter $\eta_i$ alters the magnitude of the $i^{\mathrm{th}}$ feature of the input vector, rather than the corresponding feature of $w$. The following corollary is a direct consequence of Theorem~\ref{th:main}.

\begin{corollary}[Convergence for Dropout in overparametrized linear models]
	\label{cor:OLM-DO}$ $\\
Let $f_ w ^{\mathrm{drop}}$ be as defined in~\eqref{eqdef:OLM-drop} and $\hat L$ as in~\eqref{eqdef:Lhat-ClassicalDropout}. Assume that $L(w) := \hat L(w,0)$ satisfies Assumption~\ref{ass:LossMfd}.  Let $\tilde{W}_n(0) \Rightarrow W_0 \in U_{\Gamma}$ for some locally attractive neighbourhood $U_\Gamma$ of $\Gamma$. Let $\eta_{k,i}$ be Bernoulli or Gaussian dropout noisy variables. Let $\alpha_n, \sigma_n \to 0$.

Then $ \tilde W_n$ as defined in~\eqref{eq:def:W} converges to $Y$ in the sense of Katzenberger (see Def.~\ref{def:Katzenberger-convergence}), where $Y$ is the constrained gradient flow defined as
\[
\frac{dY}{dt} =- P_{T\Gamma}\nabla_w\Reg(Y),
\qquad Y(t)\in \Gamma,
\qquad \text{and}\qquad 
Y(0) = \Phi(W_0),
\]
and 
    \begin{equation}
		\label{eqdef:Reg-OLM-DO}
    \Reg(w) := \frac12\Delta_\eta\hat{L}(w, 0) = \frac1{2N}\sum_{j=1}^{\din} (u_j^2 - v_j^2)^2\sum_{i=1}^Nx_{ij}^2, 
    \end{equation}
    and where $x_{ij}$ is the $j^{\mathrm{th}}$ feature of the $i^{\mathrm{th}}$ data point $x_i$. 
\end{corollary}

\begin{remark}
The expression~\eqref{eqdef:Reg-OLM-DO} coincides with a weighted $L^2$-regularization term in linear models---treating $u^{\odot2}-v^{\odot2}$ as a single linear parameter---in which the weights are chosen equal to the average amplitude of the corresponding feature $\frac1N\sum_{i=1}^Nx_{ij}^2$. 
\end{remark}

\begin{remark}
In~\cite[Lemma 6.2, Lemma 6.3]{LiWangArora22} sufficient conditions are given for overparametrized linear models to satisfy the manifold Assumption~\ref{ass:LossMfd}.
\end{remark}

\begin{remark}[Two Laplacians as regularizer]
Both for Dropout and for label noise the regularizer is a Laplacian, but one is with respect to $\eta$ and the other with respect to $w$; thus these two forms of regularization may lead to different types of behaviour. We illustrate this on the example of overparameterized linear models, where $f_w(x) = \left<u^{\odot 2} - v^{\odot 2}, x\right>$  as in~\eqref{eqdef:OLM}.

With Dropout noise we obtain the regularizer~\eqref{eqdef:Reg-OLM-DO}. 
With label noise the regularizer is given by~\eqref{eqdef:Reg-LabelNoise} instead, and we now make this label-noise regularizer more explicit. On the zero loss manifold we have $\left<u^{\odot 2} - v^{\odot 2}, x_i\right> = y_i$ for all $i$ and thus
\begin{align}
	\notag
\frac1{2N} \Delta_w L &= \frac1{2N^2}\Delta_u\sum_{i=1}^N\left(\left<u^{\odot 2} - v^{\odot 2}, x_i\right> - y_i \right)^2 + \frac1{2N^2}\Delta_v\sum_{i=1}^N\left(\left<u^{\odot 2} - v^{\odot 2}, x_i\right> - y_i \right)^2 \\
&=\frac1{2N^2}\sum_{i=1}^N \sum_{j=1}^{\din}4u_j^2 x_{ij}^2 + \frac1{2N^2}\sum_{i=1}^N \sum_{j=1}^{\din}4v_j^2 x_{ij}^2 = \frac2{N^2}\sum_{i=1}^N \sum_{j=1}^{\din}(u_j^2 +v_j^2) x_{ij}^2. 
\label{eq:Laplacian-w-LN}
\end{align}
Note how the Dropout regularizer~\eqref{eqdef:Reg-OLM-DO}  regularizes the difference $(u_j^2 - v_j^2)^2$, and the label noise regularizer above penalizes  both $u$ and $v$ separately. Also note how the two regularizers differ in their scaling in $N$, with label noise leading to an additional slow-down factor~$1/N$.
\end{remark}

\subsubsection{Feedforward neural networks}
\label{sss:FFNN}
We say that $f_ w : \bbR^{\din} \to \bbR$ is a $p$-layered feedforward neural network if it is a composition of $p$ blocks in which  each block consists of a linear layer and a pointwise nonlinearity. Each block is a mapping $\phi^k: y^k \to y^{k+1}\  (\bbR^{d_k} \to \bbR^{d_{k+1}})$ of the form:
\begin{align*}
    z^{k+1} &= W^{k+1}y^k + b^{k+1}, \\
y^{k+1} &= s(z^{k+1}),
\end{align*}
where $W^{k+1} \in \bbR^{d_{k+1}\times d_k}$, $b^{k+1} \in \bbR^{d_{k+1}}$, the input dimension is $d_0 = \din$ and the output dimension $d_p = 1$. The map $f_ w (x)$ then takes the form:
\[
f_ w (x) = \phi^p(\phi^{p-1} (\cdots \phi^1(x))).
\]

We introduce dropout filters $\eta^k$  that  perturb the input features at the $k$-th layer. The resulting maps $\phi_d^k: (\eta^k, y^k) \to y^{k+1} \ (\bbR^{2d_k} \to \bbR^{d_{k+1}})$ then are
\begin{align*}
    z^{k+1} &= W^{k+1}(y^k \odot(1+\eta^k))+ b^{k+1}, \\
y^{k+1} &= s(z^{k+1}),
\end{align*}
and the corresponding $f_ w ^{\mathrm{drop}}$ is given by
\begin{equation}
\label{eqdef:deep-ReLU-NN}
f_ w ^{\mathrm{drop}}(x, \eta) = \phi_d^p(\eta^p, \phi_d^{p-1} (\eta^{p-1} \cdots \phi_d^1(\eta^1, x))),
\end{equation}
Here we distinguish the same two choices as in Section~\ref{ss:DropConnect}, namely Bernoulli and Gaussian filters. 

In the Bernoulli case $\eta_{k,i}$ are i.i.d.\@ Bernoulli random variables for some $0<p<1$:
\[
\eta_{k,i} = \begin{cases}
        -1, & \text{ w.p. } p\\
        \frac{p}{1-p}  &\text{ w.p. }  1-p.\\
        \end{cases}
\]
In the Gaussian case $\eta_{k, i}$ are i.i.d.\@ normal random variables,
\[
\eta_{k,i} \sim \calN(0, \sigma^2).
\]
With Gaussian filters, no input is ever ignored, as $\bbP(\eta_{k, i} =-1)=0$. Note that Gaussian noise also allows to change sign of some inputs. 

First consider feedforward neural networks with a single hidden layer:
\begin{equation}
\label{eqdef:shallow-ReLU-NN}
    f_ w ^{\mathrm{drop}}(x, \eta) = \sum_{j=1}^{n}a_j(1+\eta_j)s(b_j^\top  x),
\end{equation}
where dropout is only applied to the output layer.
As an example we consider the \emph{smooth rectified linear unit} as activation function, 
\begin{equation}
	\label{eqdef:smooLU}
s(x) =\begin{cases}\begin{array}{lr}
       0,  \qquad \ \ \ \  x\leq 0 \\
        xe^{-1/x}, \quad x> 0.\\
        \end{array}\end{cases}.
\end{equation}

\begin{corollary}
\label{crl:shallow}
Let $f_ w ^{\mathrm{drop}}$ as defined in~\eqref{eqdef:shallow-ReLU-NN} with rectified smooth activation function~\eqref{eqdef:smooLU}.
Assume that $L(w) := \hat L(w,0)$ satisfies Assumption~\ref{ass:LossMfd}. 
Let $\tilde{W}_n(0) \Rightarrow \tilde{W}(0) \in U_{\Gamma}$ for some locally attractive neighbourhood $U_\Gamma$ of $\Gamma$. Let $\eta_{k,i}$ be Bernoulli or Gaussian Dropout noisy variables as defined above. Let $\alpha_n, \sigma_n \to 0$. 

Then $ \tilde W_n$ as defined in~\eqref{eq:def:W} converges to $Y$ in the sense of Katzenberger, where $Y$ is the constrained gradient flow defined as
\begin{equation}
	\label{eqdef:constr-GF-ShNN-DO}
\frac{dY}{dt} =- P_{T\Gamma}\nabla_w\Reg(Y), \qquad Y(t)\in \Gamma,
\qquad \text{and}\qquad 
Y(0) = \Phi(W_0),
\end{equation}
with
\begin{equation}
	\label{eqdef:Reg-ShNN-DO}
    \Reg(w) = \frac12 \Delta_\eta\hat{L}( w , 0) = \frac1N\sum_{j=1}^{n} \sum_{i=1}^N a_j^2\, s(b_j^\top x_i)^2.
\end{equation}
\end{corollary}

\begin{remark}[Known properties of the zero-loss manifold]
In \cite{cooper2018loss} it is proven that for overparameterized shallow neural networks with the rectified smooth activation function~\eqref{eqdef:smooLU} the zero-loss set of $L(w) := \hat L(w,0)$ is a smooth manifold and $ L$ satisfies Assumption~\ref{ass:LossMfd}. The result also holds for deep feedforward neural networks if the size of the last layer is greater than the size of the training dataset, i.e.\ $d_k > N$.

Similar results are known for the $\ReLU$ activation function $\ReLU(y) = \max(0, y)$~\cite{petzka2020notes, dereich2022minimal}. Even though the $\ReLU$ is not differentiable in $y = 0$, the manifold is locally smooth away from the hyperplanes $Q_i = \{x: w_i^\top x= 0\}$, so that the analysis holds for every set $K$ satisfying $K\cap Q_i = \emptyset$ for every $Q_i$. In~\cite{boursier2022gradient} the authors study the structure of the neighbourhood $U$ and propose an initialization scheme guaranteeing convergence of the gradient flow for shallow $\ReLU$ networks. 
\end{remark}

\begin{proof}
First note that the smoothness of $s$  guarantees the regularity of $\hat{L}$, and since $s$ has sublinear growth, $\hat L$ has at most growth of order $6$ in $(w,\eta)$.

Next, note that we can write $\hat L$ in the form
\[
\hat L(w,\eta) = L(w) + f(w)\cdot \eta + \frac12 H(w):(\eta\oti\eta),	
\]
with 
\[
f(w)_j = -\frac2N a_j \sum_{i=1}^N s(b_j^\top  x_j)y_i
\qquad\text{and}\qquad
H_{jj'} = \frac2N \sum_{i=1}^N a_j s(b_j^\top  x_i) a_{j'} s(b_{j'}^\top  x_i).
\]
This structure is similar to~\eqref{eq:Ldeg}, which we require for the degenerate case, but note that for the degenerate case we require $\tr H=0$, while here $\tr H$ does not vanish. Consequently the evolution takes place at the time scale $1/\alpha_n\sigma_n^2$. 
This function $\hat L$ satisfies conditions~\eqref{ass:compat-L-Lhat-difference-in-w} and~\eqref{ass:compat-L-Lhat-difference-in-etas} with $C_2=0$. 

We now turn to condition~\eqref{eq:ass:moment}. Gaussian noise satisfies this condition as mentioned in Remark~\ref{remark:noise}, and for the $q$-th absolute moment of Bernoulli random variables, where $q \geq 1$, we have 
\[
\M_q = 1\cdot p + \frac{p^q}{(1-p)^q}\cdot (1-p) = \frac{p^q + p\cdot (1-p)^{q-1}}{(1-p)^{q-1}}   = O(p) = O(\sigma^2)\qquad \text{as }p\to0.
\]
Since $C_2=0$, condition~\eqref{eq:ass:moment} also is satisfied. The condition~\eqref{ass:speed} in Theorem~\ref{th:main} is satisfied for Bernoulli filters by construction and for Gaussian filters by Lemma~\ref{l:supgaus}.

Theorem~\ref{th:main} then yields Katzenberger convergence to the constrained gradient flow~\eqref{eqdef:constr-GF-ShNN-DO}, with $\Reg(w) = \frac12 \Delta_\eta \hat L(w,0) = \frac12 \tr H(w)$ as given in~\eqref{eqdef:Reg-ShNN-DO}.
\end{proof}

For deep neural networks, with multiple hidden layers, and for Gaussian filters Theorem~\ref{th:main} again implies a convergence result. 
Due to the complexity of deep neural networks we do not provide an explicit form of the regularizer in this case.

\begin{corollary}[Convergence for Dropout noise in deep neural networks]
Let $f_ w ^{\mathrm{drop}}$ be as defined in \eqref{eqdef:deep-ReLU-NN} for some $p\geq 2$ with rectified smooth activation function $s$~\eqref{eqdef:smooLU}
Assume that $L(w) := \hat L(w,0)$ satisfies Assumption~\ref{ass:LossMfd}. 
Let $\tilde{W}_n(0) \Rightarrow \tilde{W}(0) \in U_{\Gamma}$ for some locally attractive neighbourhood $U_\Gamma$ of $\Gamma$. Let $\eta_{k,i}$ be Gaussian filter variables as defined above. Let $\alpha_n, \sigma_n \to 0$. 

Then $ W_n$ as defined in~\eqref{eq:def:W} converges to $Y$ in the sense of Katzenberger, where $Y$ is the constrained gradient flow defined as
\[
\frac{dY}{dt} =- \frac12  P_{T\Gamma}\nabla_ w\Delta_\eta\hat L(Y, 0), \qquad Y(t)\in \Gamma,
\qquad \text{and}\qquad 
Y(0) = \Phi(W_0).
\]
\end{corollary}
\begin{proof}
The proof is analogous to the proof of Corollary \ref{crl:shallow} with the only difference caused by the composition of several layers. Using the global boundedness of the derivatives of activation function $|s'(x)|, |s''(x)| < C$ for all $x\in \bbR$ one can see that for any fixed $w \in K$ the expression
$
|\nabla_w\nabla_\eta^2\hat L(w, \eta_1) - \nabla_w\nabla_\eta^2\hat L(w, \eta_2)| 
$
scales at most polynomially in $\eta_{1,2}$ and thus the assumptions of Theorem~\ref{th:main} are satisfied. 
\end{proof}

\begin{remark}[Combining Dropout with Minibatching]
	\label{rem:combining-Dropout-with-Minibatching}
It is easy to see that the combination of Dropout noise and minibatching results in the same dynamics as Dropout gradient descent without minibatching. Consider the loss
\[
\hat{L}(w, \tilde\eta, \eta) =\frac{1}{N} \sum_{i=1}^N(1+\tilde \eta_i) \hat \ell(w, \eta,  x_i, y_i).
\]
As $\hat{L}$ is linear in the minibatching noise variables $\tilde\eta$, we have $\Delta_{\tilde\eta}\hat{L}(w, \tilde\eta, \eta) = 0$ and thus the regularizer takes the same form as the  regularizer of the corresponding dropout gradient descent $\frac12 \Delta_{\tilde\eta, \eta}\hat{L}(w, \tilde\eta, \eta) = \frac12 \Delta_{\eta}\hat{L}(w, \tilde\eta, \eta) $. The same holds for the DropConnect case.
\end{remark}

\section{Discussion and outlook}\label{s:discussion}

In this paper we define a class of noisy gradient-descent systems and prove their convergence in the limit of small step size and in some cases also small noise. This class of systems unifies a broad collection of existing training algorithms in a common structure, and the convergence theorems thus give a more global understanding of the effect of noise in various overparametrized training situations. 

In this section we add some more remarks and discuss various generalizations. 

\subsection{Constant step sizes}

It is common to change the step size $\alpha$ during training, for instance to generate phases of larger and smaller noise in minibatch SGD. For some such non-constant step-size training algorithms the results of this paper should continue to hold. The essential properties that need to be verified are Assumptions~\ref{ass:Katzenberger:noise}, \ref{ass:Katzenberg:Zn}, and~\ref{assum:4.2}, which are all formulated in terms of the integrator sequences $A_n$ and $\hat A_n$ and therefore are also meaningful for non-constant step sizes. 

\subsection{Correlated noise}

Similarly, we have chosen to make the coordinates $\eta_{k,i}$ of each iterate $\eta_k$ independent from each other, but this again only for simplicity of formulation. 

To give an example of a generalization, consider again the example of Figures~\ref{fig:ex1-intro} and~\ref{fig:ex2-intro} in the introduction. Figure~\ref{fig:example-introduction-correlated-noise} compares the effect of uncorrelated (left) and correlated noise. Here we choose as uncorrelated noise
\[
\eta_{k,i} \sim \calN(0,\sigma^2 I_2)
\]
and as correlated noise
\[
	\eta_k \sim \calN(0,C), \qquad
	C = \frac{\sigma^2}2 \begin{pmatrix} 1&1 \\ 1&1  \end{pmatrix}.
\]
The corresponding regularizers are 
\begin{align*}
\text{uncorrelated: }& \frac12 \Delta_w L(w)\\
\text{correlated: }& \frac12 \nabla^2_w L(w) : C = \frac14 \bra*{\partial_{11} L + 2\partial_{12} L + \partial_{22}L},
\end{align*}
and the  minimizers of these two functions are indicated by green circles. 
\begin{figure}[ht]
	\centering
{\newcommand{\figheight}{6cm}
\includegraphics[height=\figheight]{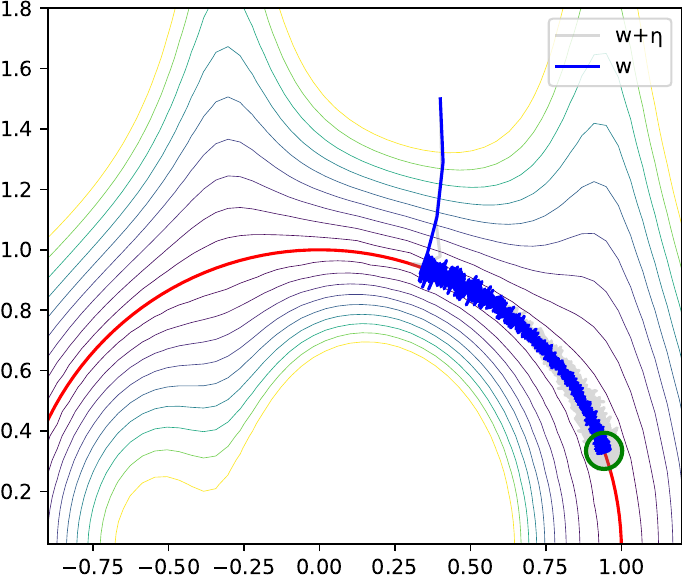}
\qquad
\includegraphics[height=\figheight]{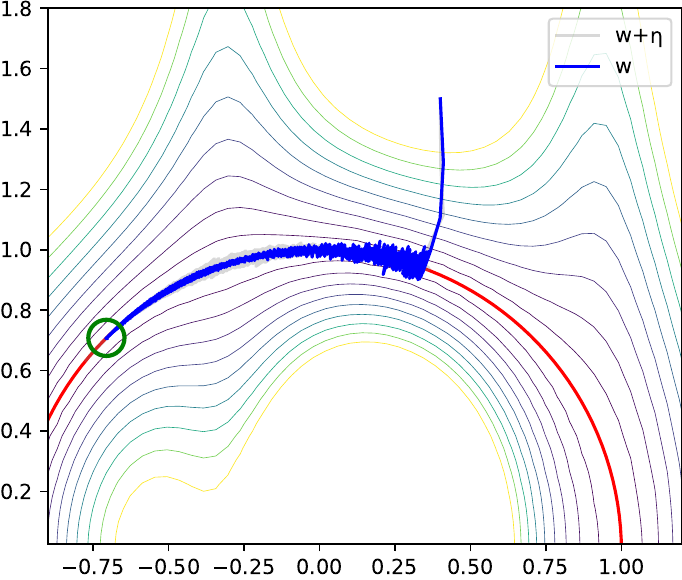}
}
\caption{When the coordinates of the vector $\eta_k$ are correlated, the resulting regularizer $\Reg$ is modified, and the evolution follows a different path along $\Gamma$. In both diagrams each $\eta_k$ is a centered normal two-dimensional random variable with covariance $C$, independent for each $k$; in the left-hand picture $C=\sigma^2 I_2$, implying independence of $\eta_{k,1}$ from $\eta_{k,2}$, while in the right-hand picture $C=\tfrac12{\sigma^2}\begin{psmallmatrix} 1&1\\1&1\end{psmallmatrix}$, implying that $\eta_{k,1}$ and $\eta_{k,2}$ are fully correlated.}
\label{fig:example-introduction-correlated-noise}
\end{figure}

\subsection{Constrained gradient flow and regularisation}

In this paper we have used the term `regularizer' and the notation `$\Reg$' for the function that drives the limiting constrained gradient flow~\eqref{eq:fthA:GF}. This terminology is inspired by its relationship with Tychonov  regularization of inverse problems. 

To explain this, note that solutions of the constrained gradient flow tend to converge to local minimizers of $\Reg$, or if one is `lucky', even to a global minimizer, i.e.\ a solution of the constrained problem
\[
\min_w \bigl\{ \Reg(w): w\in \Gamma\bigr\}.
\]
This situation is reminiscent of regularized inverse problems of the form
\[
\min_u \|Tu-f\|^2 + \lambda R(u),
\]
for which in the `weak-regularisation' limit $\lambda \to0$ the minimizers $u_\lambda$ converge to the solution of the constrained minimization problem
\[
 \min_u 	\bigl\{ R(u): Tu = f\bigr\}.
\]
This is why we call $\Reg$, the driving functional in the constrained gradient flow, the (implicit) regulariser of the noisy gradient descent. 

Theorem~\ref{fth:main1} opens the door to a form of reverse engineering. 
Given a loss $L$, the choice of~$\hat L$ is only limited by the consistency $\hat L(w,0) = L(w)$, and one therefore has a wide freedom to tailor $\hat L$ to have particular properties.
Assuming that one has an understanding of what `good' and `bad' points $w\in \Gamma$ look like, Theorem~\ref{fth:main1} suggests to look for functions $\hat L$ such that $\Delta_\eta \hat L$  gives high value to `bad' points and low value to `good' ones.

Some pointers to how `good' and `bad' points can be recognized or characterized are given by the `robustness' criterion of~\cite{petzka2020notes} (see Remark~\ref{rem:generalisation}) or the discussion in~\cite{OrvietoKerstingProskeBachLucchi22} of the connection with PAC-Bayes bounds. We leave this aspect to future work.

\subsection{Convergence results in Skorokhod spaces and Katzenberger's theorem}
The first (to our knowledge) application of Katzenberger's theorem to machine learning models was in~\cite{LiWangArora22}. In that paper the dynamics the noise is  assumed to be of minibatch-type, namely
\[
dW_n = -\nabla L(W_n)d\hat{A}_n + \sqrt{\varSigma}\sigma_{\eta_k}(W_n),
\] 
where $\eta_k$ is sampled uniformly from the set $B = \{1, 2, \dots, N\}$, and for every $i \in B$, $\sigma_{i}(W_n)$ is a deterministic function and $\bbE\sigma_{\eta_k}(W) =0$. Note that by definition the noise process in such a setting is a martingale. Our definition of noisy gradient-descent systems generalizes this, and allows us to study for instance the effect of dropout noise.

The result of~\cite{LiWangArora22} has also been generalized to SGD with momentum. In~\cite{cowsik2022flatter}, the authors study the interplay between the momentum parameter and the noise distribution. The structure of the noise is the same as in \cite{LiWangArora22}, and one direction of future work is the study of SGD with momentum with general noise.

Another type of convergence results in the sense of processes are so-called mean-field limit results. In contrast to this work, mean-field convergence describes the behaviour of the models with variable number of parameters. For example, for shallow neural networks it studies the limiting training dynamic of models
\[
f_n(W_n, x) = \frac1n\sum_{i=1}^na_i\sigma(b_i^\top x+ c_i),
\]
where $w_i = (a_i, b_i, c_i)$ and $W =(w_1, ... w_n) $. It turns out that under suitable assumptions for $\mu_n$ defined as
\[
\mu_n(s) = \frac1n\sum_{i=1}^n\delta_{w_i(s)},
\]
it holds that $\mu_n \Rightarrow \mu$ in Skorokhod topology on $[0, T]$, where $\mu$ is a solution of a measure-valued evolution equations characterized by the loss function \cite{sirignano2020mean1, rotskoff2018trainability}. Similar results have been derived for deep neural networks \cite{sirignano2022mean, nguyen2019mean}. Note that the mean-field setting does not involve rescaling of time, implying that only the main (the fast) time scale is considered. 
Another direction of future work is to study the slow time scale dynamics in the measure-valued setting. Methods such as those in~\cite{DuongLamaczPeletierSharma2017,DuongLamaczPeletierSchlichtingSharma2018} could be useful for this.


\appendix

\section{Details of numerical simulations}
\label{app:details-num-sim}

The function $L:\R^2\to\R$ depicted in Figures~\ref{fig:ex1-intro}, \ref{fig:ex2-intro}, and~\ref{fig:nondeg-deg} is
\[
L(w) = \frac{(|w|^2-1)^2}{(|w|^2+1)^2} (1+a\sin(bw_1)), 
\qquad a=0.7, \ b=5.
\]
The step size in those figures is $\alpha=0.3$, and $\eta$ has a centered normal distribution with covariance $\sigma^2 I_2$ with $\sigma=0.03$.

The constrained gradient flow in Figures~\ref{fig:ex2-intro} and~\ref{fig:nondeg-deg} is implemented by parametrizing by $\Gamma$ by polar angle $\theta$ and writing the regularizer as a function of $\theta$.

\section{Auxiliary results}

The following lemma shows that for i.i.d.\ Gaussian noise variables $\eta$ the two conditions~\eqref{ass:speed} and~\eqref{ass:speed:deg1} follow from the assumption $\alpha_n\to0$.

\begin{lemma}[Gaussian filters satisfy the noise-decay condition]
	\label{l:supgaus}
	Let $\alpha_n$ and $\sigma_n$ be positive sequences such that $\alpha_n\to0$ and $\sigma_n\to\sigma_0\geq0$. 
	Let for each $n$,  $Y_k^n$, $k\in \N$,  be i.i.d.\@ centered Gaussian random variables with variance $\sigma_n^2$. Then for any $T \in \bbR_+$, for any $p\geq 1$ the following convergence holds in probability and in distribution:
	\[
	\sup_{k\leq \frac{T}{\alpha_n^2\sigma_n^2}} \alpha_n|Y_k^n|^p \longrightarrow 0
	\qquad\text{as }n\to\infty
	\]
	\end{lemma}
	\begin{proof}
	\begin{align*}
	\bbP\biggl(\sup_{k \leq \frac{T}{\alpha_n^2\sigma_n^2} } \alpha_n|Y^n_k|^p> \e\biggr) 
	&\leq \sum_{k \leq \frac{T}{\alpha_n^2\sigma_n^2} }  \bbP\bra*{|Y^n_k|^p> \frac{\e}{\alpha_n}}
	\leq \frac{T}{\alpha_n^2\sigma_n^2}\bbP\bra*{|Y^n_1|> \frac{\e^{1/p}}{\alpha_n^{1/p}}}\\
	&=\frac{T}{2\alpha_n^2\sigma_n}\erfc\left(\frac{\e^{1/p}}{\sigma_n\alpha_n^{1/p}\sqrt 2}\right)\\
	&\leq \frac{\alpha_n^{-2+1/p}\,T}{ \e^{1/p}\sqrt{2\pi}}\exp\bra*{-\frac{\e^{2/p}}{2\sigma_n^2\alpha_n^{2/p}}}\\
	&=
	\frac{\beta_n^{-p+1/2}\,T}{ \e^{1/p}\sqrt{2\pi}}\exp\bra*{-\frac{\e^{2/p}}{2\sigma_n^2\beta_n}},
	\qquad \text{with }\beta_n = \alpha_n^{2/p},
	\end{align*}
	and for every fixed $\e> 0$ this vanishes as $\beta_n\to0$.
	\end{proof}

\ifjmlr
\else
\bibliographystyle{alphainitials}
\fi
\bibliography{bib-Mark-draft4, refs-Anna, refs-Andre}

\end{document}

%% file: header.tex
\usepackage{amssymb}
\usepackage{amsmath,amsthm}
\usepackage[english]{babel}
\usepackage[T1]{fontenc}
\usepackage[a4paper,top=2cm,bottom=2cm,left=2cm,right=2cm,marginparwidth=1.75cm]{geometry}

\usepackage{microtype}

\usepackage{enumitem} 

\usepackage{mathtools}
\usepackage{graphicx,adjustbox}
\usepackage{subcaption}
\usepackage[dvipsnames]{xcolor}
\usepackage{bbm}
\usepackage[colorlinks=true, allcolors=black]{hyperref}
\usepackage[normalem]{ulem}
\usepackage{cancel}
\usepackage{varwidth}  
\usepackage{booktabs}

\usepackage{xparse}
\NewDocumentCommand{\grad}{e{_^}}{%
  \mathop{}\!
  \nabla
  \IfValueT{#1}{_{\!#1}}
  \IfValueT{#2}{^{#2}}
}

\newcommand{\R}{\bbR}
\newcommand{\e}{\varepsilon}
\DeclareMathOperator{\Expectation}{\mathbb{E}}
\newcommand{\Expect}{\mathbb E}

\DeclareMathOperator{\tr}{tr}

\DeclareMathOperator{\erfc}{erfc}

\DeclareMathOperator{\ReLU}{ReLU}
\DeclareMathOperator\Prob{\mathbb{P}}
\DeclareMathOperator\Uni{Unif}
\DeclareMathOperator\TVar{V}
\DeclareMathOperator\Var{Var}

\DeclareMathOperator{\ProbMeas}{\calP}

\newcommand{\Range}{\mathrm{Range}}
\newcommand{\Rank}{\mathrm{rank}}

\newcommand{\oti}{\mathop{\otimes}}
\newcommand{\din}{d_{\mathrm{in}}}

\usepackage{xspace}
\newcommand\cadlag{c\`adl\`ag\xspace}

\usepackage{dsfont}

\newcommand{\invunit}{\rotatebox[origin=c]{180}{$\mathrm e$}}

\newtheorem{theorem}{Theorem}[section]
\newtheorem{corollary}[theorem]{Corollary}
\newtheorem{lemma}[theorem]{Lemma}
\newtheorem{proposition}[theorem]{Proposition}
\newtheorem{assumption}[theorem]{Assumption}

\newtheorem{formaltheorem}{Formal Theorem}

\theoremstyle{definition}
\newtheorem{definition}[theorem]{Definition}

\usepackage{textcomp}
\newtheorem{remarkx}[theorem]{Remark}
\newtheorem{examplex}[theorem]{Example}
\newenvironment{remark}
{\pushQED{\qed}\remarkx} 
{\popQED\endremarkx}
\newenvironment{example}
{\pushQED{\qed}\examplex} 
{\popQED\endexamplex}

\DeclarePairedDelimiter{\abs}{\lvert}{\rvert}

\DeclarePairedDelimiter{\bra}{(}{)}
\DeclarePairedDelimiter{\pra}{[}{]}
\DeclarePairedDelimiter{\set}{\{}{\}}

\DeclareMathAlphabet{\mathup}{OT1}{\familydefault}{m}{n}
\newcommand{\dx}[1]{\mathop{}\!\mathup{d} #1}
\newcommand{\dd}{\dx} 

\usepackage{ifthen}
\newlength{\leftstackrelawd}
\newlength{\leftstackrelbwd}
\def\leftstackrel#1#2{\settowidth{\leftstackrelawd}%
{${{}^{#1}}$}\settowidth{\leftstackrelbwd}{$#2$}%
\addtolength{\leftstackrelawd}{-\leftstackrelbwd}%
\leavevmode\ifthenelse{\lengthtest{\leftstackrelawd>0pt}}%
{\kern-.5\leftstackrelawd}{}\mathrel{\mathop{#2}\limits^{#1}}}

  \def\calC{{\mathcal C}}
  \def\calF{{\mathcal F}}
  
  \def\calL{{\mathcal L}}
 \def\calN{{\mathcal N}} 
\def\calP{{\mathcal P}}

\def\rmM{{\mathrm M}}

  \def\sfF{{\mathsf F}}

 \def\sfZ{{\mathsf Z}} 

 \def\bbE{{\mathbb E}} \def\bbF{{\mathbb F}}

 \def\bbN{{\mathbb N}} 
\def\bbP{{\mathbb P}}  \def\bbR{{\mathbb R}}

 \def\bbZ{{\mathbb Z}} 